\documentclass{article}

\usepackage[margin = 1.35in]{geometry}  

\usepackage{amsmath, amssymb, amsthm}
\usepackage{bbm}  
\usepackage{accents}  

\usepackage{graphicx}
\usepackage{tikz}
\usepackage{epstopdf}
\usepackage{epsfig}
\usepackage[pagebackref=true,breaklinks=true,letterpaper=true,colorlinks,bookmarks=false]{hyperref}

\usepackage[numbers,sort&compress]{natbib} 

\usepackage{url}

\usepackage{authblk}
\usepackage{xspace}

\usepackage{hhline}  

\usepackage{comment}
\usepackage{siunitx}
\usepackage{relsize}
\usepackage{ifthen}
\usepackage[colorinlistoftodos]{todonotes}

\usepackage[vlined,ruled,linesnumbered]{algorithm2e}
\usepackage{graphics} 
\usepackage{rotating}
\usepackage{color}
\usepackage{enumerate}
\usepackage[T1]{fontenc}
\usepackage{psfrag}
\usepackage{epsfig} 
\usepackage{booktabs}
\usepackage{graphicx,url}
\usepackage{multirow}
\usepackage{array}
\usepackage{latexsym}
\usepackage{amsfonts}
\usepackage{amsmath}
\usepackage{amssymb}
\usepackage{mathtools}
\usepackage{xstring}
\usepackage[noend]{algorithmic}
\usepackage{multirow}
\usepackage{xcolor}
\usepackage{prettyref}
\usepackage{flexisym}
\usepackage{bigdelim}
\usepackage{breqn} 
\usepackage{listings}

\usepackage{enumitem}
\usepackage{xspace}
\usepackage{bm}
\graphicspath{{./figures/}}
\usepackage{tikz}
\usetikzlibrary{matrix,calc}

\usepackage{mdwlist}

\makecompactlist{itemize}{stditemize}

\newrefformat{prob}{Problem\,\ref{#1}}
\newrefformat{def}{Definition\,\ref{#1}}
\newrefformat{sec}{Section\,\ref{#1}}
\newrefformat{sub}{Section\,\ref{#1}}
\newrefformat{prop}{Proposition\,\ref{#1}}
\newrefformat{app}{Appendix\,\ref{#1}}
\newrefformat{alg}{Algorithm\,\ref{#1}}
\newrefformat{cor}{Corollary\,\ref{#1}}
\newrefformat{thm}{Theorem\,\ref{#1}}
\newrefformat{lem}{Lemma\,\ref{#1}}
\newrefformat{fig}{Fig.\,\ref{#1}}
\newrefformat{tab}{Table\,\ref{#1}}

\newtheorem{theorem}{Theorem}

\newtheorem{proposition}[theorem]{Proposition}

\newcommand{\cf}{\emph{cf.}\xspace}

\newcommand{\bdmath}{\begin{dmath}}
\newcommand{\edmath}{\end{dmath}}
\newcommand{\beq}{\begin{equation}}
\newcommand{\eeq}{\end{equation}}
\newcommand{\bdm}{\begin{displaymath}}
\newcommand{\edm}{\end{displaymath}}
\newcommand{\bea}{\begin{eqnarray}}
\newcommand{\eea}{\end{eqnarray}}
\newcommand{\beal}{\beq \begin{array}{ll}}
\newcommand{\eeal}{\end{array} \eeq}
\newcommand{\beas}{\begin{eqnarray*}}
\newcommand{\eeas}{\end{eqnarray*}}
\newcommand{\ba}{\begin{array}}
\newcommand{\ea}{\end{array}}
\newcommand{\bit}{\begin{itemize}}
\newcommand{\eit}{\end{itemize}}
\newcommand{\ben}{\begin{enumerate}}
\newcommand{\een}{\end{enumerate}}

\newcommand{\calC}{{\cal C}}

\newcommand{\calE}{{\cal E}}

\newcommand{\calG}{{\cal G}}

\newcommand{\calI}{{\cal I}}

\newcommand{\calK}{{\cal K}}

\newcommand{\calN}{{\cal N}}

\newcommand{\calV}{{\cal V}}

\newcommand{\eg}{\emph{e.g.,}\xspace}
\newcommand{\ie}{\emph{i.e.,}\xspace}

\newcommand{\hide}[1]{}
\newcommand{\wrt}{w.r.t.\xspace}

\newcommand{\hiddenText}{{\color{gray} hidden text.}}
\newcommand{\hideWithText}[1]{\hiddenText}

\newcommand{\kron}{\otimes}

\newcommand{\subject}{\text{ subject to }}

\DeclareMathOperator*{\argmin}{arg\,min}

\newcommand{\norm}[1]{\left\| #1 \right\|}

\newcommand{\tran}{^{\mathsf{T}}}

\newcommand{\diag}[1]{\mathrm{diag}\left(#1\right)}
\newcommand{\trace}[1]{\mathrm{tr}\left(#1\right)}

\newcommand{\rank}[1]{\mathrm{rank}\left(#1\right)}

\newcommand{\inv}{^{-1}}

\newcommand{\zero}{{\mathbf 0}}
\newcommand{\eye}{{\mathbf I}}

\newcommand{\Real}[1]{ { {\mathbb R}^{#1} } }

\newcommand{\SEthree}{\ensuremath{\mathrm{SE}(3)}\xspace}

\newcommand{\SOthree}{\ensuremath{\mathrm{SO}(3)}\xspace}
\newcommand{\Othree}{\ensuremath{\mathrm{O}(3)}\xspace}

\newcommand{\scenario}[1]{{\smaller \sf#1}\xspace}

\newcommand{\gtwoo}{{\smaller\sf g2o}\xspace}

\newcommand{\gtsam}{{\smaller\sf gtsam}\xspace}

\newcommand{\blue}[1]{{\color{blue}#1}}

\newcommand{\linkToPdf}[1]{\href{#1}{\blue{(pdf)}}}
\newcommand{\linkToPpt}[1]{\href{#1}{\blue{(ppt)}}}
\newcommand{\linkToCode}[1]{\href{#1}{\blue{(code)}}}
\newcommand{\linkToWeb}[1]{\href{#1}{\blue{(web)}}}
\newcommand{\linkToVideo}[1]{\href{#1}{\blue{(video)}}}
\newcommand{\linkToMedia}[1]{\href{#1}{\blue{(media)}}}
\newcommand{\award}[1]{\xspace} 

\renewcommand{\norm}[1]{\left\lVert #1 \right\rVert}

\newcommand{\vectorize}[1]{\mathrm{vec}\parentheses{#1}}

\newcommand{\cbrace}[1]{\left\{#1\right\}}
\newcommand{\sym}[1]{\mathbb{S}^{#1}}

\newcommand{\bmat}{\left[ \begin{array}}
\newcommand{\emat}{\end{array}\right]}

\newcommand{\parentheses}[1]{\left(#1\right)}

\newcommand{\bracket}[1]{\left[#1\right]}

\newcommand{\tldp}{\widetilde{p}}
\newcommand{\hatp}{\widehat{p}}

\newcommand{\sesync}{\scenario{SE-Sync}}
\newcommand{\ransac}{\scenario{RANSAC}}
\newcommand{\colmap}{\scenario{COLMAP}}
\newcommand{\ceres}{\scenario{Ceres}}
\renewcommand{\gtsam}{\scenario{GTSAM}}
\newcommand{\SIMthree}{\mathrm{SIM}(3)}
\newcommand{\simsync}{\scenario{SIM-Sync}}
\newcommand{\sR}{\widetilde{R}}
\renewcommand{\gtwoo}{\scenario{g2o}}
\newcommand{\teaser}{\scenario{TEASER}}
\newcommand{\gnc}{\scenario{GNC}}
\newcommand{\tum}{\scenario{TUM}}
\newcommand{\sOthree}{\mathfrak{s}\mathrm{O}(3)}
\newcommand{\qcqp}{\scenario{QCQP}}
\newcommand{\sesyncgtwoo}{\scenario{SE-Sync+g2o}}
\newcommand{\bal}{\scenario{BAL}}

\newcommand{\simsyncgnc}{\scenario{SIM-Sync-GNC}}
\newcommand{\gncsimsync}{\scenario{GNC+SIM-Sync}}
\newcommand{\teasersimsync}{\scenario{TEASER+SIM-Sync}}
\newcommand{\teasersesync}{\scenario{TEASER+SE-Sync}}
\newcommand{\teasersimsyncgt}{\scenario{TEASER+SIM-Sync+GT}}
\newcommand{\sift}{\scenario{SIFT}}
\newcommand{\caps}{\scenario{CAPS}}
\newcommand{\midas}{\scenario{MiDaS-v3}}
\newcommand{\orbslam}{\scenario{ORB-SLAM3}}

\newcommand{\teasersimsyncgtdepth}{\scenario{TEASER+SIM-Sync+GTDepth}}

\title{\vspace{-20mm} SIM-Sync: From Certifiably Optimal Synchronization over the 3D Similarity Group to Scene Reconstruction with Learned Depth}

\date{}

\author{Xihang Yu\thanks{Work done during visit at the Harvard Computational Robotics Lab. Email: \texttt{xihangyu@umich.edu}} }
\affil{College of Literature, Science, and the Arts, University of Michigan. }

\author{Heng Yang\thanks{Email: \texttt{hankyang@seas.harvard.edu}}}
\affil{School of Engineering and Applied Sciences, Harvard University}

\begin{document}

\maketitle


\vspace{-8mm}

\begin{abstract}

    We present \simsync, a \emph{certifiably optimal} algorithm that estimates camera trajectory and 3D scene structure \emph{directly from multiview image keypoints}. \simsync fills the gap between pose graph optimization and bundle adjustment; the former admits efficient global optimization but requires relative pose measurements and the latter directly consumes image keypoints but is difficult to optimize globally (due to camera projective geometry). 
    
    The bridge to this gap is a \emph{pretrained} depth prediction network. Given a graph with nodes representing monocular images taken at unknown camera poses and edges containing pairwise image keypoint correspondences, \simsync first uses a pretrained depth prediction network to \emph{lift} the 2D keypoints into 3D \emph{scaled} point clouds, where the scaling of the per-image point cloud is unknown due to the scale ambiguity in monocular depth prediction. \simsync then seeks to \emph{synchronize} jointly the unknown camera poses and scaling factors (\ie over the 3D similarity group) by minimizing the sum of the Euclidean distances between edge-wise scaled point clouds. The \simsync formulation, despite nonconvex, allows designing an efficient certifiably optimal solver that is almost identical to the \sesync algorithm. Particularly, after solving the translations in closed-form, the remaining optimization over the rotations and scales can be written as a \emph{quadratically constrained quadratic program}, for which we apply Shor's semidefinite relaxation. We show how to add scale regularization in the semidefinite program to prevent contraction of the estimated scales.

    We demonstrate the tightness, robustness, and practical usefulness of \simsync in both simulated and real experiments. In simulation, we show (i) \simsync compares favorably with \sesync in scale-free synchronization, and (ii) \simsync can be used together with robust estimators to tolerate a high amount of outliers. In real experiments, we show (a) \simsync achieves similar performance as \ceres on bundle adjustment datasets, and (b) \simsync performs on par with \orbslam on the \tum dataset with zero-shot depth prediction.\footnote{Code available: \url{https://github.com/ComputationalRobotics/SIM-Sync}.}
\end{abstract}

\section{Introduction}
\label{sec:introduction}

3D scene reconstruction and camera trajectory estimation from image sequences remains one of the most fundamental and extensively studied problems in robotics and computer vision. At the heart of this problem lies (i) finding reliable \emph{feature correspondences}, \eg keypoints, across images (sometimes referred to as data association in robotics), and (ii) \emph{estimating camera poses} given the associated features. In this paper, we focus on the camera trajectory estimation problem and denote $x_i = (R_i,t_i) \in \SEthree, i=1,\dots,N$ as the set of camera poses to be estimated.

A long list of formulations and solutions have been developed for camera trajectory estimation, for which we refer to \cite[Chapter 9]{barfoot17book-state} and \cite[Chapter 11]{szeliski22book-computer}.
We motivate our work by describing two of the most popular formulations.

The first formulation, exemplified by \emph{pose graph optimization} (PGO)~\cite{rosen2019se} for simultaneous localization and mapping (SLAM) in robotics, first estimates \emph{relative} camera poses from image features, \ie estimating $\tilde{x}_{ij} \approx x_i \inv x_j$ for $(i,j) \in \calE$\footnote{In PGO, a pose graph is formulated, where the node set $\calV = \{1,\dots,N\}$ includes the unknown absolute camera poses, and the edge set $\calE$ contains all pairs of nodes such that relative poses can be measured.} with overlapping image features,\footnote{For example, relative camera poses can be estimated using \ransac~\cite{fischler81acm-ransac} plus the five-point algorithm~\cite{nister04pami-fivept}. More generally, such relative poses can be estimated not only from cameras images, but also from other sensor modalities such as IMU, GPS, and LiDAR with suitable algorithms.} and then solves an optimization problem that \emph{synchronizes} $\{ x_i \}_{i=1}^N$ from the relative measurements $\{\tilde{x}_{ij} \}_{(i,j) \in \calE}$. The synchronization problem, despite nonconvex, has an objective function that is \emph{polynomial} in the unknowns. Consequently, the seminal work \sesync~\cite{rosen2019se,carlone15iros-lagrangian,carlone16tro-planar} demonstrated efficiently solving the problem to \emph{certifiable global optimality} using semidefinite programming (SDP) relaxations and customized low-rank SDP solvers. Holmes and Barfoot~\cite{holmes23ral-efficient} derived similar SDP-based global optimality certificates for landmark-based SLAM, as long as the landmark measurements are in 3D, which preserves the polynomiality of the objective function.  

The second formulation, exemplified by \emph{bundle adjustment} (BA)~\cite{agarwal10eccv-ba,schoenberger16cvpr-sfm} for structure from motion (SfM) in computer vision, solves an optimization problem that jointly estimates camera poses and 3D keypoints by minimizing \emph{geometric reprojection errors}. The key challenge to solve this formulation is that the objective function is no longer polynomial in the camera poses and 3D keypoints, but rather a sum of \emph{rational} functions. Therefore, the most popular solution methods, \eg~\colmap~\cite{agarwal10eccv-ba}, \ceres~\cite{Agarwal22-ceres}, and \gtsam~\cite{gtsam}, rely on gradient-based local optimization techniques and can be sensitive to the quality of initialization. It is, however, not impossible to solve this formulation to global optimality. For example, by replacing the geometric reprojection error with the \emph{object space
error},\footnote{The object space error is essentially the point-to-line distance between the 3D keypoint and the bearing vector emanating from the camera center to the 2D image keypoint. This error function has been used by multiple authors as an approximation for the geometric reprojection error~\cite{kneip14eccv-upnp}.} \cite{schweighofer06bmvc-fast} recovers polynomiality and designed a globally convergent algorithm, albeit not based on SDP relaxations. It may also be possible to apply the SDP relaxation hierarchy designed for rational function optimization~\cite{bugarin16mpc-rational} to the BA formulation. However, the SDP relaxation hierarchy in~\cite{bugarin16mpc-rational} scales poorly (and worse than its polynomial counterpart) and such an attempt has never been made in the literature. 

In summary, given image feature correspondences, the BA-type formulation estimates camera poses in a single step by minimizing an objective function that is directly constructed from image keypoints.\footnote{We remark that BA algorithms often estimate relative camera poses to initialize the joint optimization in camera poses and 3D keypoints. Therefore, one can also consider them as two-step approaches.} The PGO-type formulation, however, takes a two-step approach, where the first step estimates relative camera poses and the second step performs pose synchronization. The BA-type formulation is more straightforward than the PGO-type formulation,\footnote{The BA-type formulation only require camera images, while the PGO-type formulation typically requires additional sensors such as IMU.} but more difficult to optimize globally. Therefore, we ask the question: \emph{can we design a camera trajectory estimation formulation that (i) directly consumes image features and (ii) admits efficient global optimization?}

{\bf Contributions}. Inspired by the recent trend in computer vision~\cite{luo2020consistent,kopf21robust,zhang2022structure, merrillfast} that leverages \emph{a learned depth prediction network}~\cite{Ranftl2022} for bundle adjustment, we propose a \emph{certifiably optimal} camera trajectory estimation algorithm that directly consumes image feature correspondences. The key insight is that, with the help of a pretrained depth prediction network, 2D keypoint correspondences can be effectively \emph{lifted} to 3D keypoint correspondences, leading to a formulation whose objective function is again polynomial in the unknown camera poses. Yet different from the formulation in {multiple point cloud registration} (MPCR)~\cite{chaudhury15siopt-mpcr,iglesias20cvpr-global}, the lifted 3D keypoints have an \emph{unknown scaling factor} (per camera frame) due to the scale ambiguity of depth prediction. As a result, our formulation seeks to estimate, for each camera frame $i \in [N]$, both the camera pose $(R_i,t_i) \in \SEthree$ and the scaling factor $s_i > 0$, \ie an element $x_i = (s_i,R_i,t_i)$ in the \emph{3D similarity group} $\SIMthree$. For this reason, we call our formulation \simsync, which performs synchronization over $\SIMthree$ using pairwise image keypoint correspondences.\footnote{An earlier work~\cite{strasdat10rss-scale} formulates a synchronization problem over $\SIMthree$ with relative pose measurements instead of direct image keypoints.} A graphical illustration of \simsync using the \tum dataset~\cite{sturm2012benchmark} as an example is provided in Fig.~\ref{SIM-Sync_vis}.

The nice property of our \simsync formulation is that it allows designing a certifiably optimal solver in a way that is almost identical to \sesync. Specifically, we first solve the unknown translations as a function of the rotations and scales, arriving at an optimization problem whose objective is a \emph{quartic} (\ie degree-four) polynomial in the rotations and scales. Then, by creating \emph{scaled rotations} $ \sR_i = s_i R_i, i=1,\dots,N$, the quartic polynomial becomes \emph{quadratic} and the problem becomes a quadratically constrained quadratic program (QCQP), for which we apply the standard Shor's semidefinite relaxation~\cite{bao11mp-qcqp}. Given the same number of camera frames $N$, our \simsync relaxation leads to an SDP having the same matrix size as that of \sesync ($3N \times 3N$), but with \emph{fewer} linear equality constraints (due to fewer quadratic equality constraints on $\sR_i$). Moreover, we show that it is possible to \emph{regularize} the scale estimation to be close to $1$ by adding $\sum_{i=1}^N(s_i^2 - 1)^2$ into the objective, which can be conveniently handled by the SDP relaxation with small positive semidefinite variables. This regularization effectively prevents \emph{contraction} (\ie $s_i$ tends to zero) of the estimated camera trajectory in certain test cases (\eg the 3D grid graph).

We then conduct a suite of simulated and real experiments to investigate the empirical performance of \simsync. Particularly, with simulations we show that (i) the \simsync relaxation is almost always \emph{exact} (tight) under small to medium noise corruption in the measurements, \ie globally optimal estimates can be computed and certified; (ii) \simsync achieves similar estimation accuracy as \sesync (and \sesync refined by \gtwoo); (iii) \simsync can be robustified against $40-80\%$ outlier correspondences by running \gnc~\cite{yang2020graduated} and \teaser~\cite{yang2020teaser} to preprocess pairwise correspondences. With real experiments we show that (a) \simsync compares favorably with \ceres on the \bal bundle adjustment dataset~\cite{agarwal10eccv-ba}, and (b) \simsync achieves similar performance as \orbslam~\cite{campos2021orb} on the \tum dataset~\cite{sturm2012benchmark} with zero-shot depth prediction.

{\bf Paper organization}. 
We present the \simsync formulation in Section~\ref{sec:problem-formulation}, where we also derive the simplified QCQP formulation. We introduce the semidefinite relaxation and scale regularization in Section~\ref{sec:semidefinite-relaxation}. We present experimental results in Section~\ref{sec:experiments} and conclude in Section~\ref{sec:conclusions}.

\begin{figure}[t]
    \centering
    \includegraphics[width= 1.0\columnwidth]{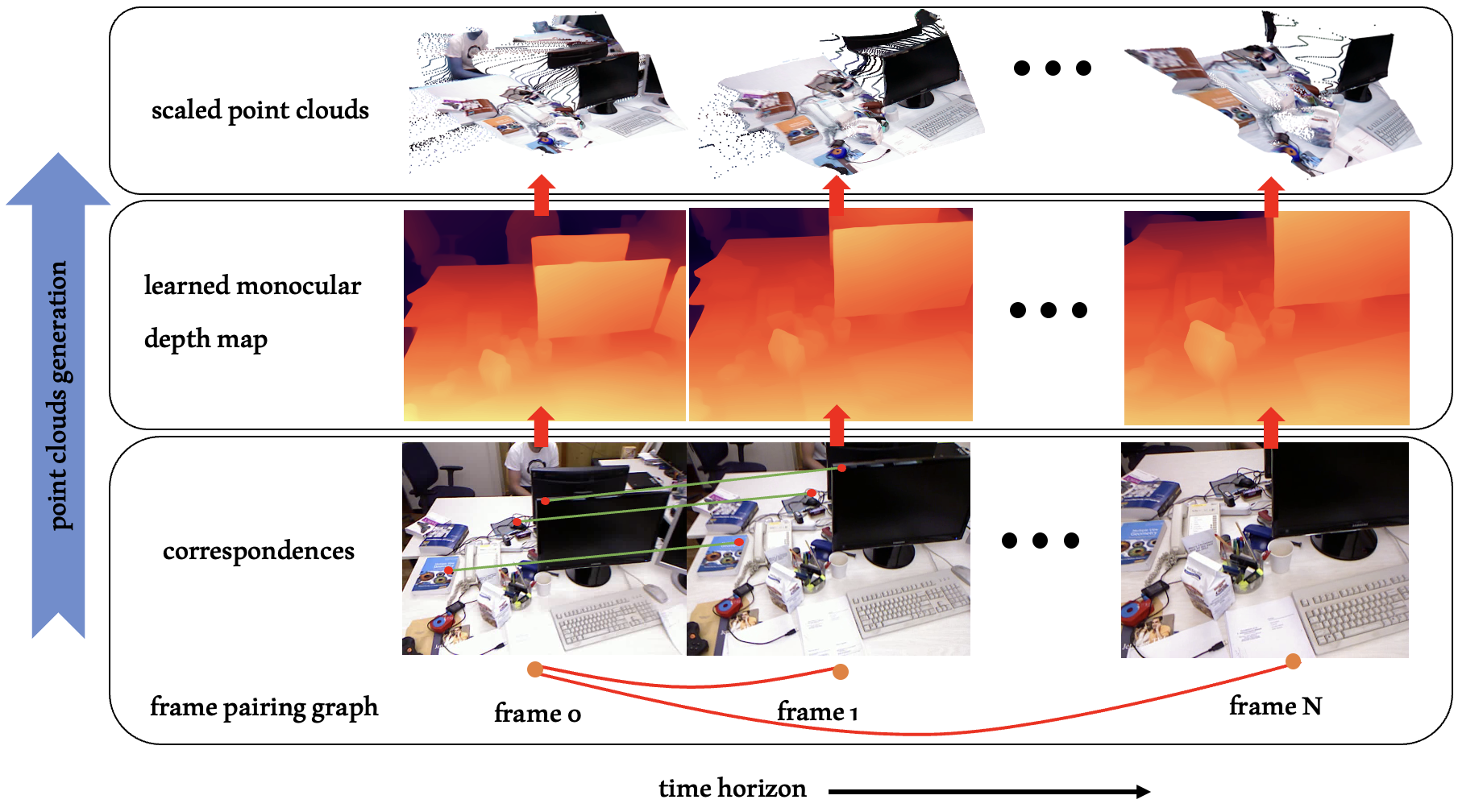}
    \caption{Illustration of \simsync on the \tum dataset~\cite{sturm2012benchmark}.}
    \label{SIM-Sync_vis}
\end{figure}

\section{Problem Formulation}
\label{sec:problem-formulation}

Consider a graph $\calG = (\calV,\calE)$, where each node $i \in \calV = [N]$ is associated with an RGB image $I_i \in \Real{H \times W \times 3}$ and an unknown camera pose $(R_i,t_i) \in \SEthree$, and each edge $(i,j) \in \calE$ contains a set of $n_{ij}$ dense pixel-to-pixel correspondences $\calC_{ij} = \{ p_{i,k} \leftrightarrow p_{j,k} \}_{k=1}^{n_{ij}}$ with $p_{i,k} \in \Real{2}$ the $k$-th pixel location in image $I_i$ and $p_{j,k} \in \Real{2}$ the $k$-th pixel location in image $I_j$. Assuming all the camera intrinsics $\{K_i\}_{i=1}^N$ are known, we can compute 
\bea 
\tldp_{i,k} = K_i\inv \bmat{c} p_{i,k}^x \\ p_{i,k}^y \\ 1 \emat
\eea 
as the \emph{bearing vector} normalized by the camera intrinsics. The third entry of $\tldp_{i,k}$ is equal to $1$.

{\bf Pretrained depth prediction}. Suppose we are given a pretrained depth estimation network that, for each image $I_i$, produces a depth map. Let $d_{i,k} > 0$ be the predicted depth of $p_{i,k}$ and $s_i > 0$ be the unknown scale coefficient for image $I_i$.\footnote{In practice, we use interpolation to obtain $d_{i,k}$ because the depth map is discretized. We also discard depth values that are too far away, which tend to be erroneous.} Consequently, 
\bea
\hatp_{i,k} = s_i d_{i,k} \tldp_{i,k}
\eea  
corresponds to the 3D location of $p_{i,k}$ in the $i$-th camera frame. Effectively, with the pretrained depth predictor, for every $(i,j) \in \calE$, we have a pair of \emph{scaled} point cloud measurements $\{ d_{i,k} \tldp_{i,k} \}_{k=1}^{n_{ij}}$ and $\{ d_{j,k} \tldp_{j,k} \}_{k=1}^{n_{ij}}$, as shown in Fig.~\ref{SIM-Sync_vis} top panel.

{\bf The \simsync formulation}. We are interested in estimating the unknown camera poses and the per-image scale coefficients $\{x_i=(s_i, R_i, t_i) \}_{i=1}^N$. 
We formulate the following optimization
\begin{equation}
\label{eq:sba}
\min_{\substack{s_i > 0, R_i \in \SOthree, t_i \in \Real{3} \\ i=1,\dots,N } }  \sum_{(i,j) \in \calE} \sum_{k=1}^{n_{ij}} w_{ij,k}\norm{ \parentheses{R_i \underbrace{(s_i d_{i,k} \tldp_{i,k} )}_{\hatp_{i,k}} + t_i} - \parentheses{R_j\underbrace{(s_j d_{j,k} \tldp_{j,k} )}_{\hatp_{j,k}} + t_j} }^2 \tag{\simsync}
\end{equation} 
where the objective function seeks to minimize the 3D point-to-point distances because $(R_i,t_i)$ transforms $\hatp_{i,k}$, and $(R_j,t_j)$ transforms $\hatp_{j,k}$ into the same global coordinate frame. In \eqref{eq:sba}, we include $w_{ij,k} > 0$ for generality: these known weights capture the potential uncertainty of the correspondences. Usually these weights are unknown and in our experiments we use \gnc and \teaser to estimate them so that $w_{ij,k} = 1$ indicates inliers and $w_{ij,k} = 0$ indicates outliers. 

{\bf Anchoring}.
Problem~\eqref{eq:sba} is ill-defined. One can choose $s_i \rightarrow 0,\forall i=1,\dots,N$, $t_1 = t_2 = \dots = t_N = \text{constant}$, and the objective of~\eqref{eq:sba} can be set arbitrarily close to zero. To resolve this issue, we anchor the first frame and set $R_1 = \eye_3, t_1 = \zero, s_1 = 1$, which is common practice in many related pose graph estimation formulations \cite{rosen2019se}.

\subsection{A QCQP Formulation}

The \eqref{eq:sba} formulation is readily in the form of a polynomial optimization problem (POP). The objective function is a quartic polynomial, the constraint $s_i > 0$ is an affine polynomial inequality, and the constraint $R_i \in \SOthree$ is equivalent to a set of quadratic polynomial constraints~\cite{yang2022certifiably}. Therefore, one can directly apply Lasserre's hierarchy of moment relaxations~\cite{lasserre01siopt-global} to design convex SDP relaxations for \eqref{eq:sba}. However, as noted in \cite{yang2022certifiably}, a direct application of Lasserre's hierarchy often leads to SDPs beyond the scalability of current solvers. Therefore, in the following we will simplify \eqref{eq:sba} as a quadratically constrained quadratic problem (QCQP), in a way that is inspired by \sesync~\cite{rosen2019se}.

Our first step is to simplify the objective function in \eqref{eq:sba}.

\begin{proposition}[Simple Objective Function]
    \label{prop:simple-objective}
    Let $t = [t_1\tran,\dots,t_N\tran]\tran \in \Real{3N}$ be the concatenation of translations, and $r = [\vectorize{s_1 R_1}\tran,\dots,\vectorize{s_N R_N}\tran]\tran \in \Real{9N}$ be the concatenation of (vectorized) scaled rotations, then the objective function of \eqref{eq:sba} can be written as
    \begin{equation}\label{eq:simpleobjective}
        L(t,r) = t\tran (Q_1 \kron \eye_3)t + 2 r\tran (V \kron \eye_3) t + r\tran (Q_2 \kron \eye_3) r,
    \end{equation}
    where $Q_1,Q_2,V$ can be computed as follows
    \bea 
    Q_1 = & \displaystyle \sum_{(i,j) \in \calE} (W_{ij} e_i \tran - W_{ij} e_j \tran)\tran (W_{ij} e_i \tran - W_{ij} e_j \tran) \in \Real{N \times N}, \label{eq:Q1}\\
Q_2 = & \displaystyle \sum_{(i,j) \in \calE} (e_i \tran \kron P_i \tran  - e_j \tran \kron P_j \tran )\tran (e_i \tran \kron P_i \tran  - e_j \tran \kron P_j \tran ) \in \Real{3N \times 3N},\\
V = & \displaystyle \sum_{(i,j) \in \calE} (e_i \tran \kron P_i \tran  - e_j \tran \kron P_j \tran )\tran (W_{ij} e_i \tran - W_{ij} e_j \tran) \in \Real{3N \times N},
    \eea
    with 
    \bea 
    P_{i} =& \bmat{ccc} \sqrt{w_{ij,1}} d_{i,1} \tldp_{i,1} & \cdots & \sqrt{w_{ij,n_{ij}}} d_{i,n_{ij}}  \tldp_{i,n_{ij}} \emat \in \Real{3 \times n_{ij}}, \quad i =1,\dots,N, \\
    W_{ij} =& [ \sqrt{w_{ij,1}} ,\dots, \sqrt{w_{ij,n_{ij}}} ]\tran \in \Real{n_{ij}}, \quad (i,j) \in \calE,
    \eea
    and $e_i \in \Real{N}$ the all-zero vector except that the $i$-th entry is equal to $1$. 
\end{proposition}

The proof of Proposition \ref{prop:simple-objective} is in Appendix \ref{proof_prop1}.

Now that the objective $L(t,r)$ in \eqref{eq:simpleobjective} is quadratic in $t$, an unconstrained variable. Therefore, we can set the gradient of $L(t,r)$ \wrt $t$ to zero and solve the optimal $t$ in closed form.

\begin{proposition}[Scaled-Rotation-Only Formulation]
    \label{prop:scale-rotation-only}
    Let $R = [s_1 R_1,\dots,s_N R_N] \in \Real{3 \times 3N}$ be the concatenation of (unvectorized) scaled rotations, then problem \eqref{eq:sba} is equivalent to the following optimization
    \bea \label{eq:scaledRonly}
    \rho^\star = \min_{R} \trace{Q R\tran R},
    \eea 
    where $Q$ can be computed as 
    \bea 
    Q = A\tran Q_1 A + VA + A\tran V\tran + Q_2 \in \sym{3N},
    \eea 
    with 
    \bea 
    A = \bmat{c} \zero_{1\times 3N} \\ - (\bar{Q}_1\tran \bar{Q}_1)\inv \bar{Q}_1\tran V\tran \emat \in \Real{N \times 3N},
    \eea 
    and $\bar{Q}_1 \in \Real{N \times (N-1)}$ includes the last $N-1$ columns of $Q_1$. Moreover, denote the optimal solution of \eqref{eq:scaledRonly} as $R^\star$, then the optimal translation to \eqref{eq:sba} can be recovered as 
    \bea 
    t^\star = (A \kron \eye_3) \vectorize{R^\star}.
    \eea 
\end{proposition}

The proof of Proposition \ref{prop:scale-rotation-only} is in Appendix \ref{proof_prop2}.

Problem \eqref{eq:scaledRonly} is still a quartic POP because the variable $R$ contains the product between scales and rotations. Our last step is to create new variables $\bar{R}_i = s_i R_i, i=1,\dots,N$ so that problem \eqref{eq:scaledRonly} becomes a QCQP.

\begin{proposition}[QCQP Formulation]\label{prop:qcqp}
    Let $\sOthree \subset \Real{3\times 3}$ be the set of matrices that can be written as the product between a nonnegative scalar and a $3\times 3$ orthogonal matrix, \ie 
    \bea \label{eq:sOthreedef}
    \sOthree = \{ \bar{R} \in \Real{3\times 3} \mid \exists s \geq 0, R \in \Othree \text{ such that } \bar{R} = sR \}.
    \eea 
    Then $\sOthree$ can be described by the following quadratic constraints
    \bea \label{eq:sOthreequadratic}
    \bar{R} = \bmat{ccc} c_1 & c_2 & c_3 \emat \in \sOthree \Longleftrightarrow \begin{cases}
        c_1\tran c_1 = c_2\tran c_2 = c_3\tran c_3 \\
        c_1\tran c_2 = c_2\tran c_3 = c_3\tran c_1 = 0
    \end{cases}.
    \eea 
    Consider the following quadratically constrained quadratic program (QCQP)
    \begin{equation} \label{eq:simsync-qcqp}
        \rho_{\qcqp}^\star = \min_{R} \trace{Q R\tran R} \subject R = \bmat{ccc} \bar{R}_1 & \cdots & \bar{R}_N \emat \in \sOthree^N, \tag{\qcqp}
    \end{equation}  
    and let $R^\star = [\bar{R}_1^\star,\dots,\bar{R}_N^\star]\tran$ be a global optimizer. If 
    \bea \label{eq:determinant}
    \det{\bar{R}_1^\star } > 0, i=1,\dots,N,
    \eea 
    then $R^\star$ is a global minimizer to problem \eqref{eq:scaledRonly} and hence also \eqref{eq:sba}.
\end{proposition}

The proof of Proposition \ref{prop:qcqp} is in Appendix \ref{proof_prop3}.

With Proposition \ref{prop:qcqp}, we know that if we can solve \eqref{eq:simsync-qcqp} to global optimality, then by checking the determinants of the optimal solution as in \eqref{eq:determinant}, we can certify its global optimality to the original \eqref{eq:sba} problem. In fact, as shown in \sesync~\cite{rosen2019se}, typically the relaxation from $\SOthree$ to $\Othree$ is tight because the set of rotations and the set of reflections in $\Othree$ are disjoint from each other. Therefore, we can almost expect $\rho^\star_{\qcqp} = \rho^\star$ (as we also observe in experiments).

\section{Semidefinite Relaxation}
\label{sec:semidefinite-relaxation}

The previous section has reformulated (more precisely, relaxed) the \eqref{eq:sba} formulation as the compact \eqref{eq:simsync-qcqp}. We can now design the following semidefinite relaxation. 

\begin{proposition}[SDP Relaxation]\label{prop:sdprelaxation}
    The following semidefinite program (SDP)
    \bea\label{eq:sdp}
    f^\star = \min_{X \in \sym{3N}} \trace{QX} \subject X = \bmat{ccc} \alpha_1 \eye_3 & \cdots & * \\
    \vdots & \ddots & \vdots \\
    * & \cdots & \alpha_N \eye_3
    \emat \succeq 0
    \eea
    is a convex relaxation to \eqref{eq:simsync-qcqp} and $f^\star \leq \rho^\star_{\qcqp}$. Let $X^\star$ be a global minimizer of \eqref{eq:sdp}. If $\rank{X^\star} = 3$, then $X^\star$ can be factorized as $X^\star = (R^\star) \tran R^\star$, where $R^\star \in \sOthree^N$ is a global optimizer to \eqref{eq:simsync-qcqp}.  
\end{proposition}

Note that $\alpha_1 = 1$ in \eqref{eq:sdp} because we set $s_1 = 1$. To enforce the diagonal blocks of $X$ in \eqref{eq:sdp} to be scaled identity matrices, one just need to (i) set their off-diagonal entries as zero, and (ii) set their diagonal entries to be equal to each other. As a result, there are $5(N-1) + 6$ linear equality constraints in \eqref{eq:sdp}, which is fewer than the $6N$ linear equality constraints in \sesync.

{\bf Suboptimality}. In practice, checking the rank condition of the optimal solution of \eqref{eq:sdp} can be sensitive to numerical thresholds. Therefore, we always generate a solution $\widehat{R}$ from $X^\star$ that is also feasible for problem \eqref{eq:scaledRonly} and evaluate the objective of \eqref{eq:scaledRonly} at $\widehat{R}$, denoted as $\hat{\rho}$ and satisfies
\bea \label{eq:suboptimality}
f^\star \leq \rho^\star_{\qcqp} \leq \rho^\star \leq \hat{\rho}.
\eea
We then compute the \emph{relative suboptimality}
\bea 
\label{eq:subopt}
\eta = \frac{\hat{\rho} - f^\star}{1 + |f^\star| + |\hat{\rho}|}.
\eea 
Clearly, $\eta = 0$ certifies global optimality of the solution $\widehat{R}$ and tightness of the SDP relaxation.

{\bf Rounding}. We perform the following procedure to round a feasible $\widehat{R}$ from $X^\star$. First we compute the spectral decomposition of $X^\star = \sum_{i=1}^{3N} \lambda_i u_i u_i\tran$. Then we assemble 
\bea 
U = \bmat{ccc} \sqrt{\lambda_1} u_1 & \sqrt{\lambda_2} u_2 & \sqrt{\lambda_3} u_3 \emat = \bmat{ccc} U_1 & \cdots & U_N \emat\tran \in \Real{3N \times 3},
\eea 
where $\lambda_1,\lambda_2,\lambda_3$ are the three largest eigenvalues. Finally, we compute the scales and rotations
\bea 
\hat{s}_1 = 1, \hat{s}_i = \Vert U_1\tran U_i \Vert_F / \sqrt{3}, i=2,\dots,N, \\
\widehat{R}_1 = \eye_3, \widehat{R}_i = \Pi_{\SOthree} (U_1\tran U_i/s_i), i=2,\dots,N,
\eea 
and assemble $\widehat{R} = [\hat{s}_1 \widehat{R}_1,\dots,\hat{s}_N \widehat{R}_N]$, where $\Pi_{\SOthree}$ denotes the projection onto $\SOthree$.
 
\subsection{Scale Regularization}

Empirically, we find that for certain graph structures (shown in Section \ref{sec:experiments}), the optimal scale estimation of \eqref{eq:scaledRonly} tends to become much smaller than $1$ for $i=2,\dots,N$, a phenomenon that we call \emph{contraction}. This is undesired because the true scales are often close to $1$. Therefore, we propose to regularize the \eqref{eq:simsync-qcqp} and the SDP \eqref{eq:sdp}. Observe that, if the relaxation is tight, then $\alpha_i = s_i^2$ in \eqref{eq:sdp} for $i=1,\dots,N$. Therefore, by adding $(\trace{X_{ii}}/3 - 1)^2 = (\alpha_i - 1)^2$ ($X_{ii}$ denotes the $i$-th diagonal block of $X$) into the objective of \eqref{eq:sdp}, we encourage the SDP to penalize $(s_i^2 - 1)^2$ and hence prevent the scale estimation from contracting. The scale-regularized problem, fortunately, is still an SDP.

\begin{proposition}[Scale Regularization]\label{prop:scale-regularization}
    The following scale-regularied problem
    \bea 
    \min_{X \in \sym{3N}} \trace{QX} + \lambda \parentheses{\sum_{i=1}^N (\trace{X_{ii}} / 3 - 1)^2 } \subject X \text{ as in } \eqref{eq:sdp}
    \eea
    for a given $\lambda > 0$, is equivalent to
    \bea 
    \min_{X \in \sym{3N}} & \trace{Q X} + \lambda \sum_{i=1}^N t_i \label{eq:scale-regularize-simplify}\\
    \subject & X \text{ as in } \eqref{eq:sdp} \\
    & \bmat{cc} 1 & \trace{X_{ii}}/3 - 1 \\ \trace{X_{ii}}/3 - 1 & t_i \emat \succeq 0, i=1,\dots,N. \label{eq:sr-t-upper-bound}
    \eea     
\end{proposition}

Proposition \ref{prop:scale-regularization} is easy to verify because \eqref{eq:sr-t-upper-bound} implies $t_i \geq (\trace{X_{ii}} / 3 - 1)^2$ by the Schur complement lemma, and the minimization in \eqref{eq:scale-regularize-simplify} will push $t_i = (\trace{X_{ii}} / 3 - 1)^2$.

Both SDPs \eqref{eq:sdp} and \eqref{eq:scale-regularize-simplify} are implemented in Python and solved with MOSEK by directly passing the problem data to the MOSEK Python interface.
\section{Experiments}
\label{sec:experiments}

We test the performance of \simsync in both simulated and real datasets. All experiments are conducted on a laptop
equipped with an Intel 14-Core i7-12700H CPU and 32 GB memory.

In Section \ref{scale-known}, we test \simsync in simulated \emph{scale-free synchronization} problems (\ie $s_i=1,i=1,\dots,N$) and compare its performance to \sesync and \sesyncgtwoo. 

In Section \ref{regularization}, we test the scale regularization \eqref{eq:scale-regularize-simplify} and show that it effectively prevents contraction of the estimated pose graph.

In Section \ref{outlier-rejection}, we simulate outliers in the feature correspondences and demonstrate that \simsync can be used together with robust estimators such as \gnc and \teaser.

Finally, we test \simsync in real datasets. Section \ref{bal-experiment} provides results of \simsync on the \bal dataset~\cite{agarwal10eccv-ba} that is popular in computer vision for bundle adjustment. Section \ref{tum-experiment} provides results of \simsync on the \tum dataset~\cite{sturm2012benchmark} that is popular in robotics for SLAM.

\subsection{Scale-free Synchronization}
\label{scale-known}
\textbf{Setup.}  We assume that the scaling factor is known, \ie $s_i = 1,i=1,\dots,N$, which is a realistic assumption when the images are taken by RGB-D cameras or registered with LiDAR scanners. Consequently, we are only interested in estimating node-wise poses $(R_i,t_i),i \in \calV$ given pairs of point cloud measurements over the edges $\calE$. To simulate the pose graph $\calG = (\calV,\calE)$, we first simulate a random point cloud $P\in \Real{3 \times n},n=1000$ in the world frame. Each point in $P$ follows a Gaussian distribution $\calN(\zero,\eye_3)$. We then simulate a trajectory of camera poses $(R_i,t_i),i=1,\dots,N$ with $N=50$ by following certain graph topologies, specifically, a circle, a grid, and a line, as commonly used in related works \cite{rosen2019se,iglesias20cvpr-global}. Details for simulating the camera trajectories are as follows.
\begin{itemize}
    \item Circle. The camera moves in a circle with a radius of 10 meters for one round in 50 steps, while facing the circle's center. 
    \item Grid. In a cube with a edge-length of $2$ meters, the camera moves on the surface for 50 steps. The camera can only move one meter to an adjacent node at each step. The starting point is randomly chosen among all nodes. A typical example is as shown in Fig.~\ref{Grid_vis}.
    \item Line. The camera moves linearly for 3 meters in 50 discrete steps while facing the point cloud $P$ at a distance of 10 meters from the line.
\end{itemize}

Given each camera pose $(R_i,t_i),i =1,\dots,N$, we generate a noisy point cloud observation
\bea
P_i = R_i P + t_i + \epsilon_i
\eea
where $\epsilon_i \in \Real{3 \times n}$ are i.i.d. Gaussian noise vectors following $\calN(0, \sigma^2 \eye_3)$. We then simulate correspondences over each edge $(i,j) \in \calE$ by subsampling $P_i$ and $P_j$. To make the correspondences more realistic, we associate $P_i$ and $P_j$ as follows:
\begin{enumerate}
    \item We first find a subset of indices $\calI_{ij} \subseteq [n]$ such that points in $\calI_{ij}$ lie in both the field of view (FOV) of camera $i$ and the FOV of camera $j$. FOV is set as 60 degrees for all experiments.
    \item We then randomly select a subset $\calK_{ij} \subseteq \calI_{ij}$ with cardinality $q$, a random number between $10$ and $|\calI_{ij}|$, and let $\widehat{P}_i = \{p_{i,k} \in P_i \mid k \in \calK_{ij} \}$ and $\widehat{P}_j=\{p_{j,k} \in P_j \mid k \in \calK_{ij} \}$ be the final point cloud pairs on edge $(i,j)$.
\end{enumerate}
We pass $(\widehat{P}_i,\widehat{P}_j)_{(i,j) \in \calE}$ to \eqref{eq:sba} to estimate node-wise absolute poses.

{\bf Baselines}.
We compare \simsync with \sesync. In order to use \sesync, we need to estimate relative poses among all the edges $\calE$. This is done by running Arun's method~\cite{arun1987least} on the point cloud pairs $(\widehat{P}_i,\widehat{P}_j)_{(i,j) \in \calE}$. \sesync also requires a covariance estimation of the relative pose. To do so, we compute the Cramer-Rao lower bound at Arun's optimal solution and feed the covariance estimates to \sesync. We provide a detailed derivation of the covariance matrix in Appendix~\ref{scale-free noise}. Note that \sesync assumes the rotational noise follows an isotropic Langevin distribution and it internally computes a Langevin approximation of the covariance matrix fed to it. Therefore, we also compare with \sesyncgtwoo, where the \sesync solution is used to initialize a local search using \gtwoo with the Cramer-Rao lower bound.

\begin{figure}[t]
    \centering
    \includegraphics[width=0.8\columnwidth]{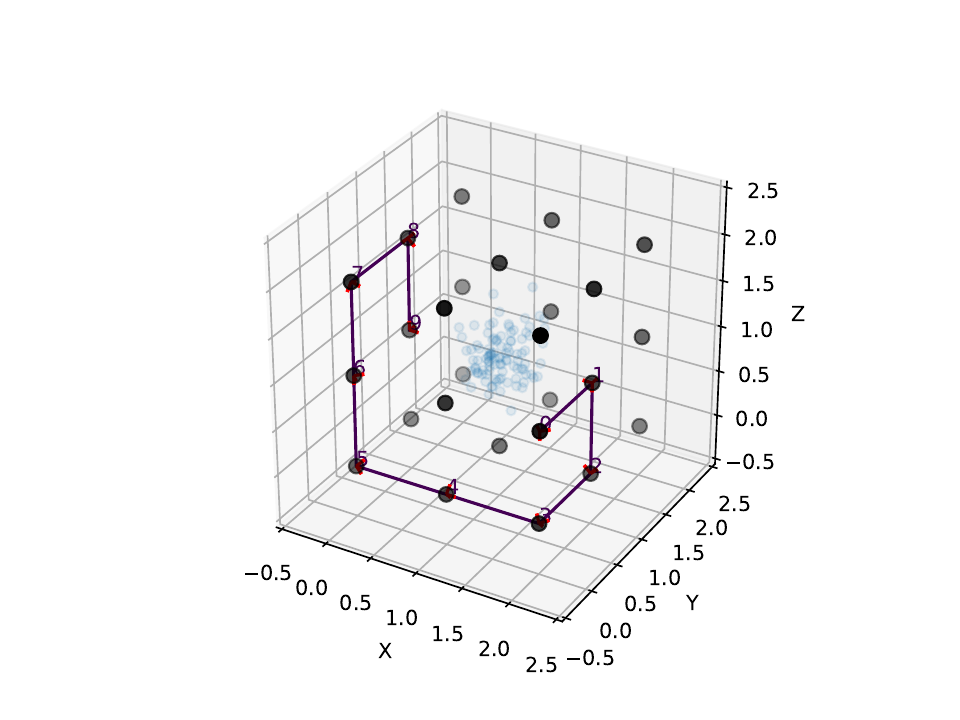}
    \caption{A sample camera trajectory in the Grid dataset.}
    \label{Grid_vis}
\end{figure}

\textbf{Results.} 
We choose $\sigma \in \{0.001, 0.01, 0.1, 1, 2, 3\}$ and at each noise level we run $20$ Monte Carlo random tests. Fig.~\ref{fig:scale_free_results} shows the rotation errors and translation errors of \simsync compared with \sesync and \sesyncgtwoo.
In both the circle dataset and the grid dataset, \simsync surpasses \sesync and \sesyncgtwoo by a (very) small margin, while in the line dataset, \simsync and \sesync perform almost the same.
Fig.~\ref{fig:scale_free_results} bottom row plots the relative suboptimality $\eta$ (\cf \eqref{eq:suboptimality}) of \simsync and \sesync. We consider the relaxation is not tight if $\eta$ exceeds $0.05$ (the red horizontal dashed line). In the circle dataset and the line dataset, when $\sigma=4$, \sesync's relaxation becomes completely inexact, while \simsync can still achieve tightness, although not always.


\begin{figure*}[h]
	\begin{center}
	\begin{minipage}{\textwidth}
	\begin{tabular}{ccc}%
		\begin{minipage}{0.32\textwidth}%
			\centering%
			\includegraphics[width=\textwidth]{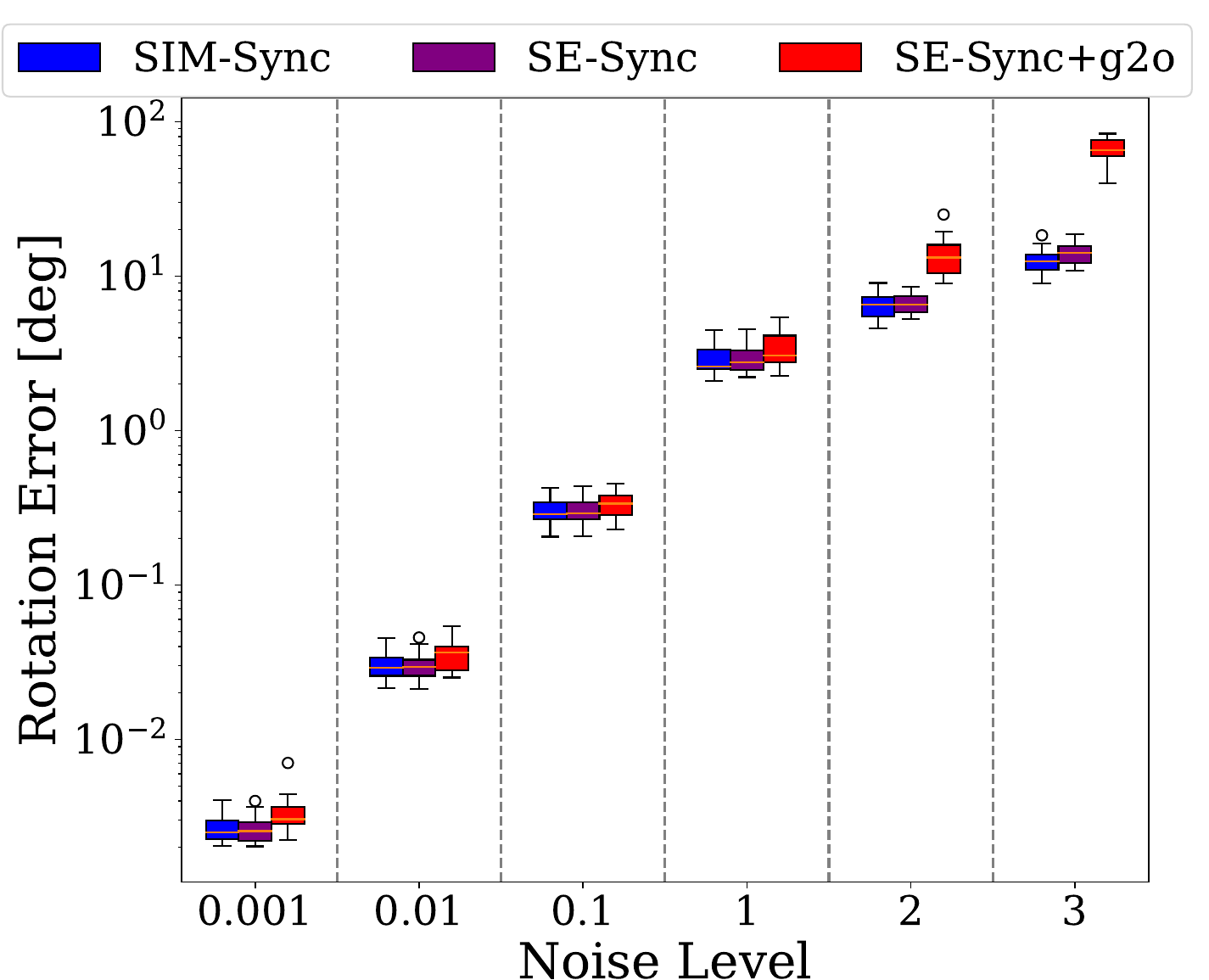}
		\end{minipage}
		&
		\begin{minipage}{0.32\textwidth}%
			\centering%
			\includegraphics[width=\textwidth]{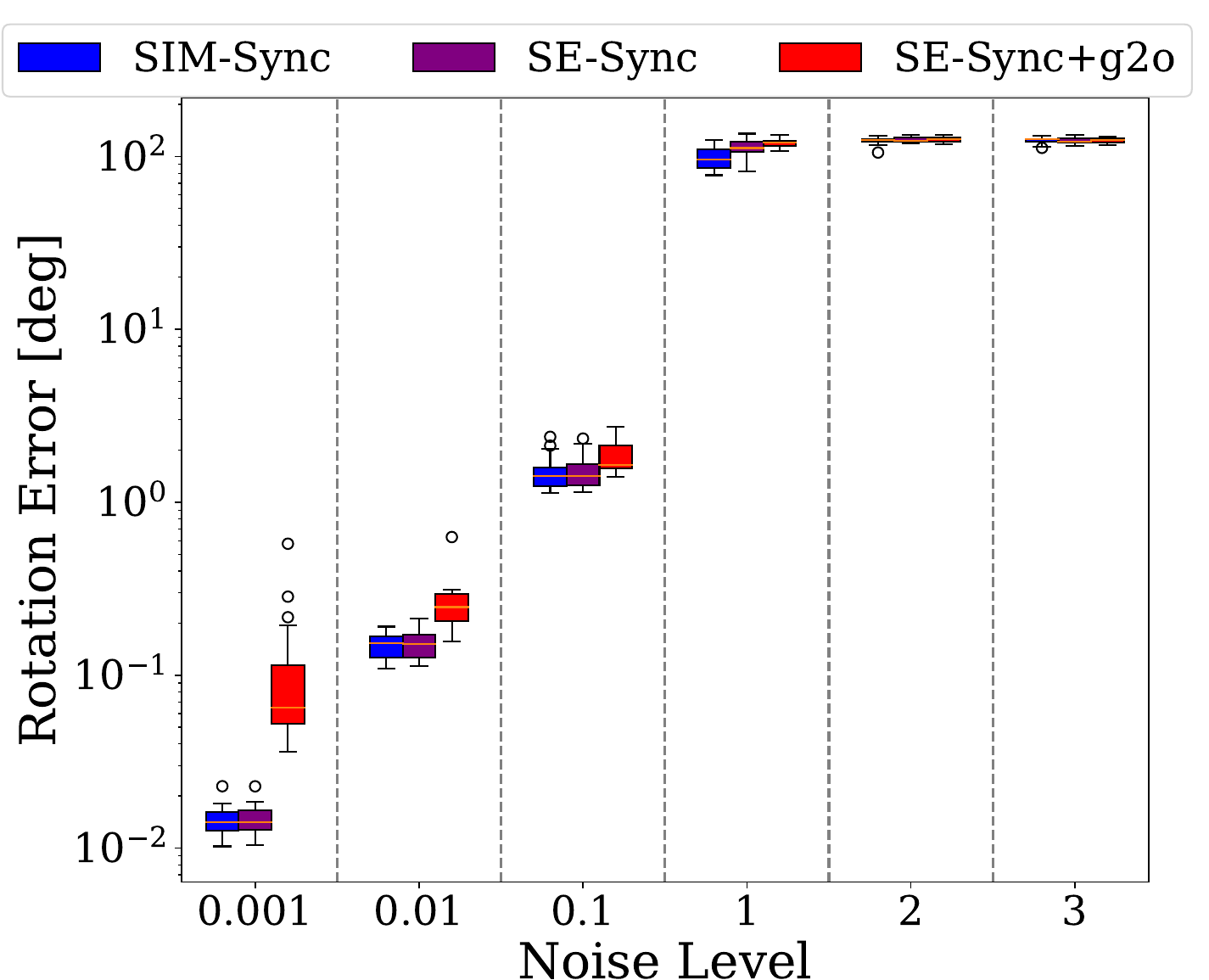}
		\end{minipage}
		&
		\begin{minipage}{0.32\textwidth}%
			\centering%
			\includegraphics[width=\textwidth]{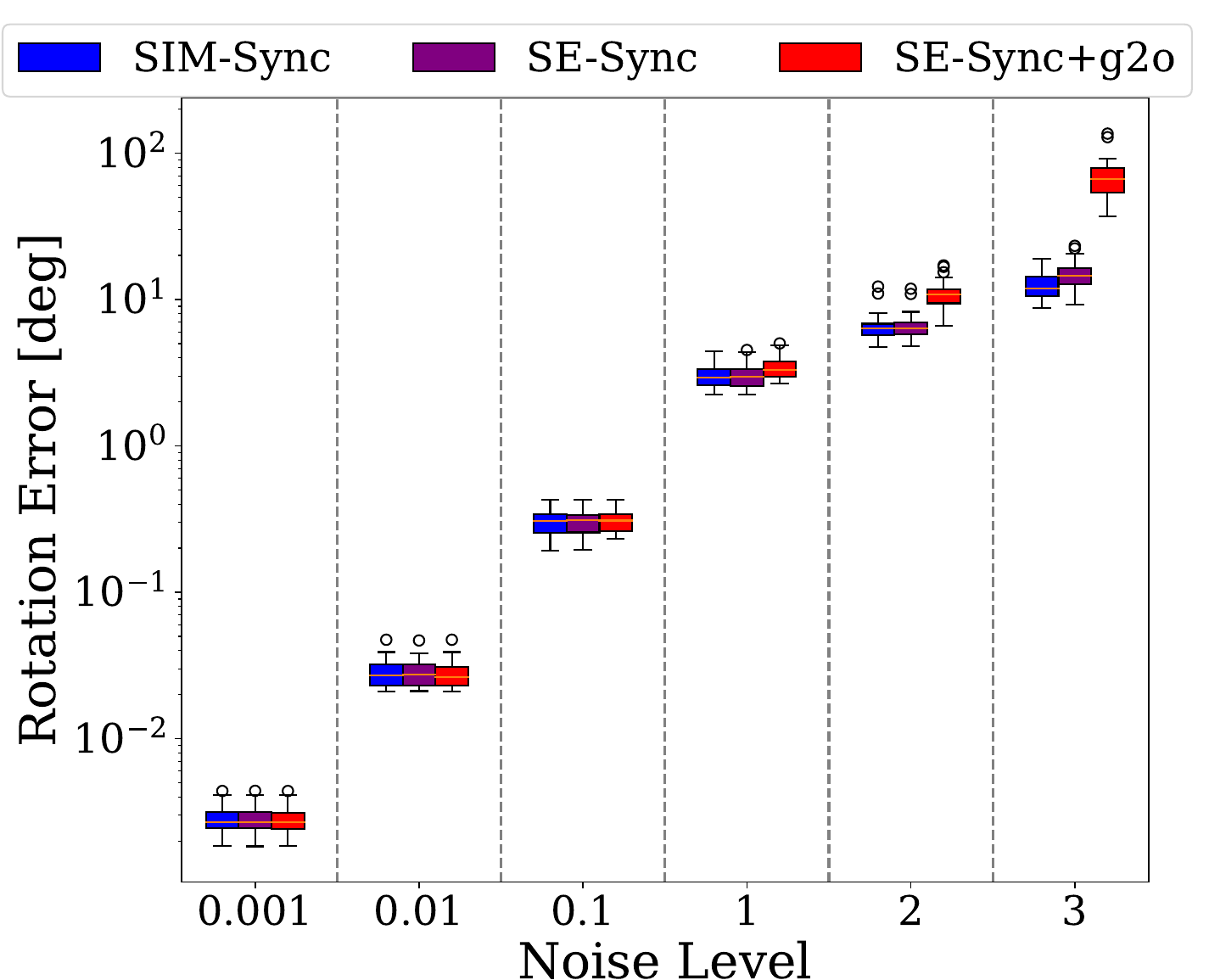}
		\end{minipage} \\
		\begin{minipage}{0.32\textwidth}%
			\centering%
			\includegraphics[width=\textwidth]{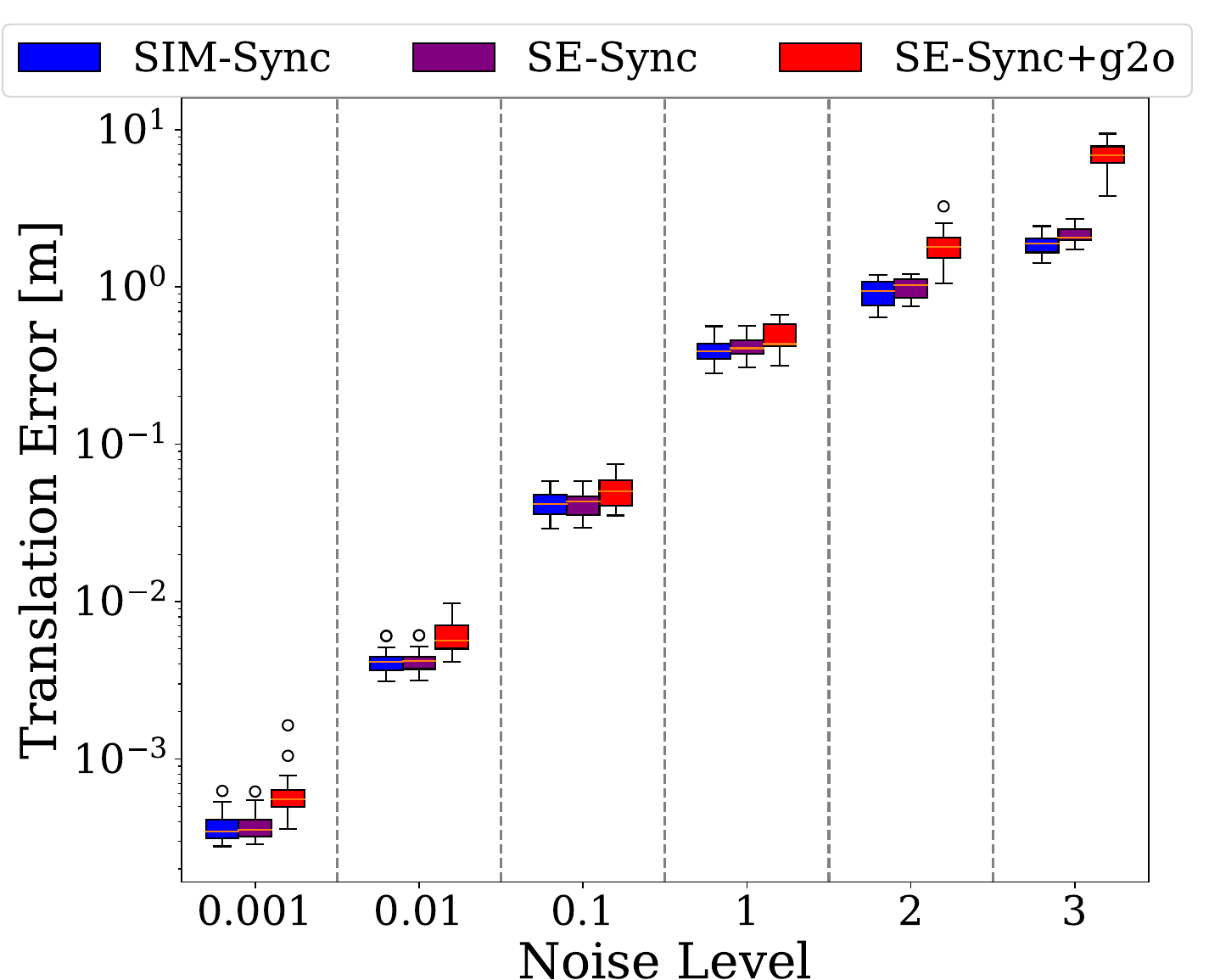}
		\end{minipage}
		&
		\begin{minipage}{0.32\textwidth}%
			\centering%
			\includegraphics[width=\textwidth]{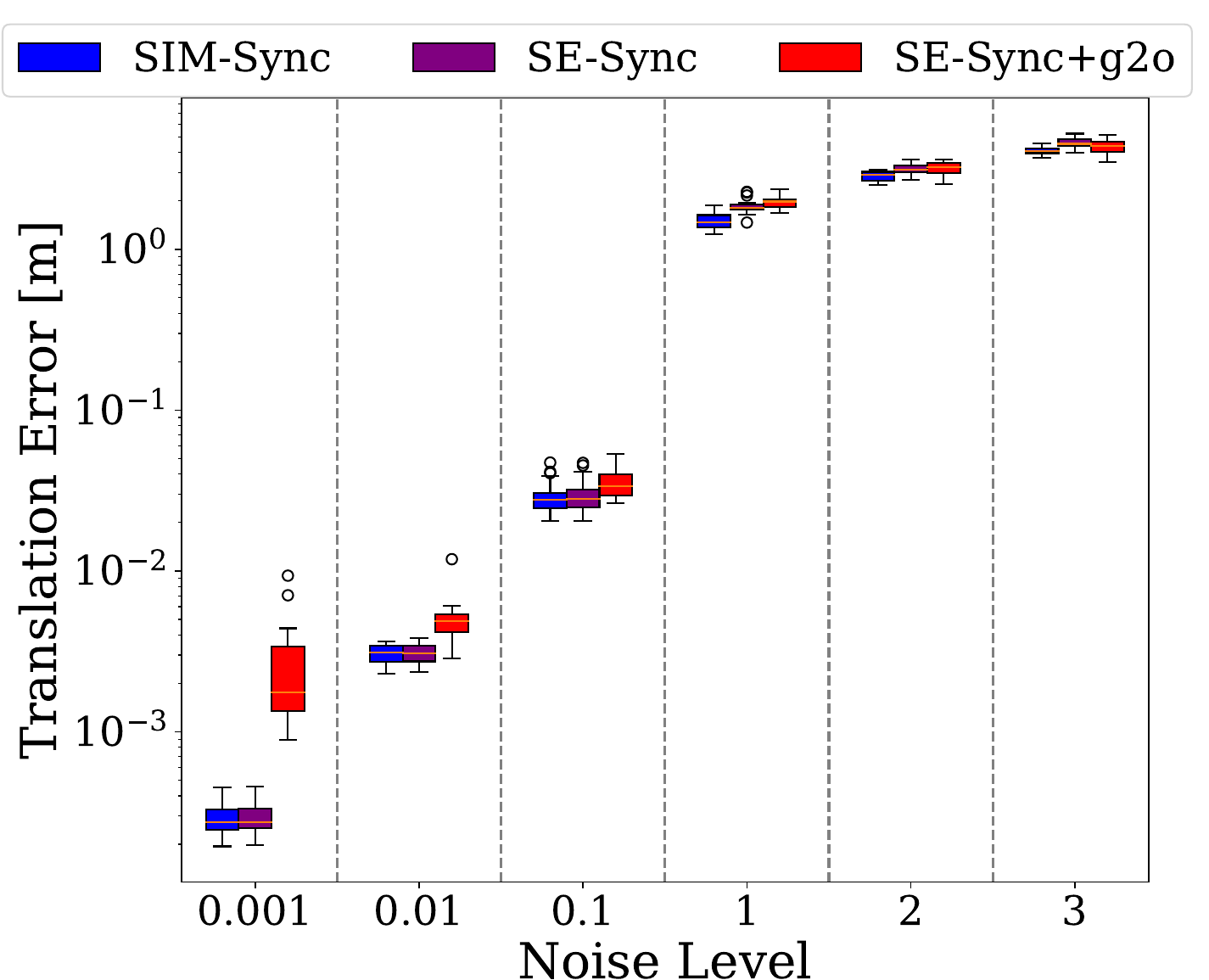}
		\end{minipage}
		&
		\begin{minipage}{0.32\textwidth}%
			\centering%
			\includegraphics[width=\textwidth]{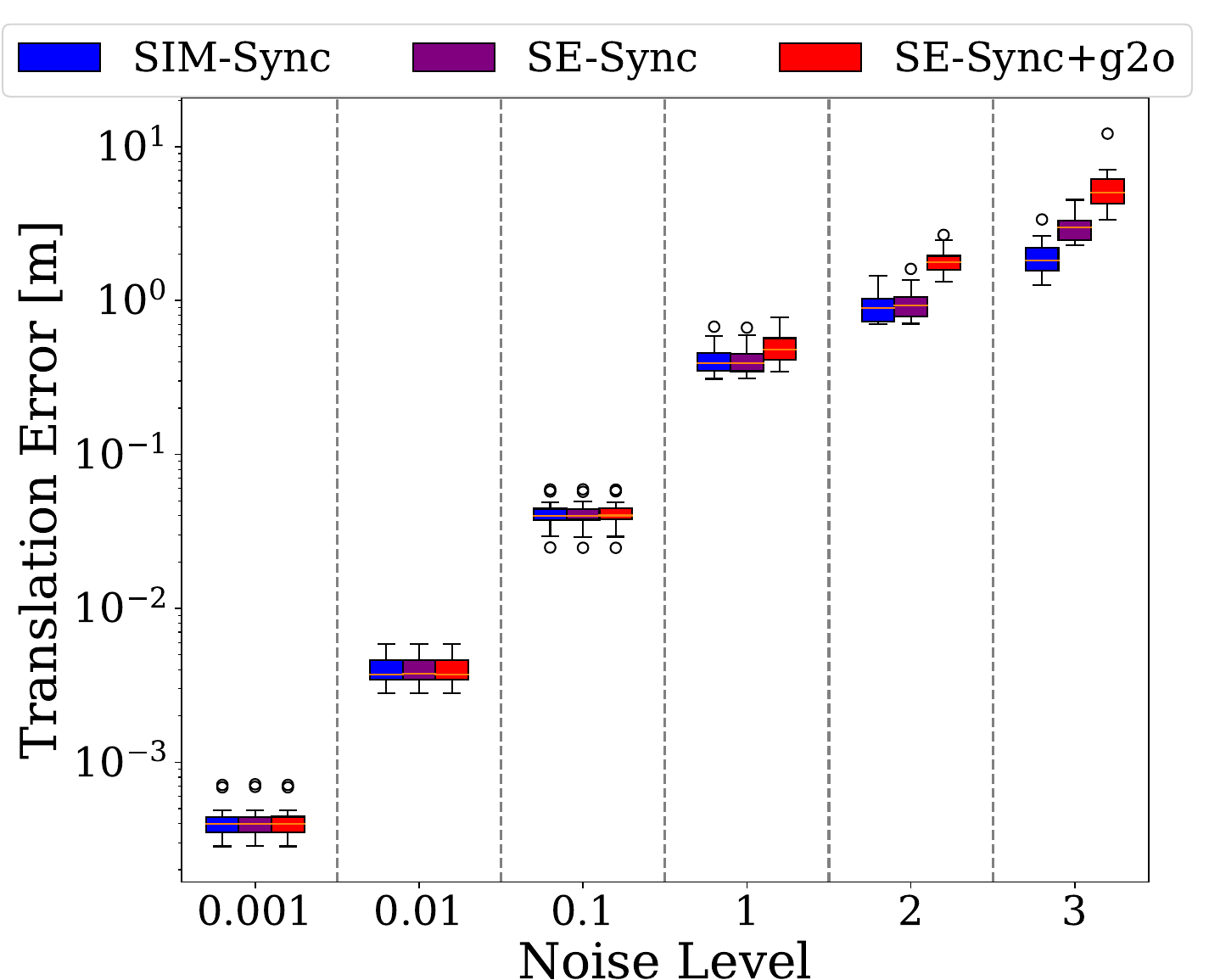}
		\end{minipage} \\
		\begin{minipage}{0.32\textwidth}%
			\centering%
			\includegraphics[width=\textwidth]{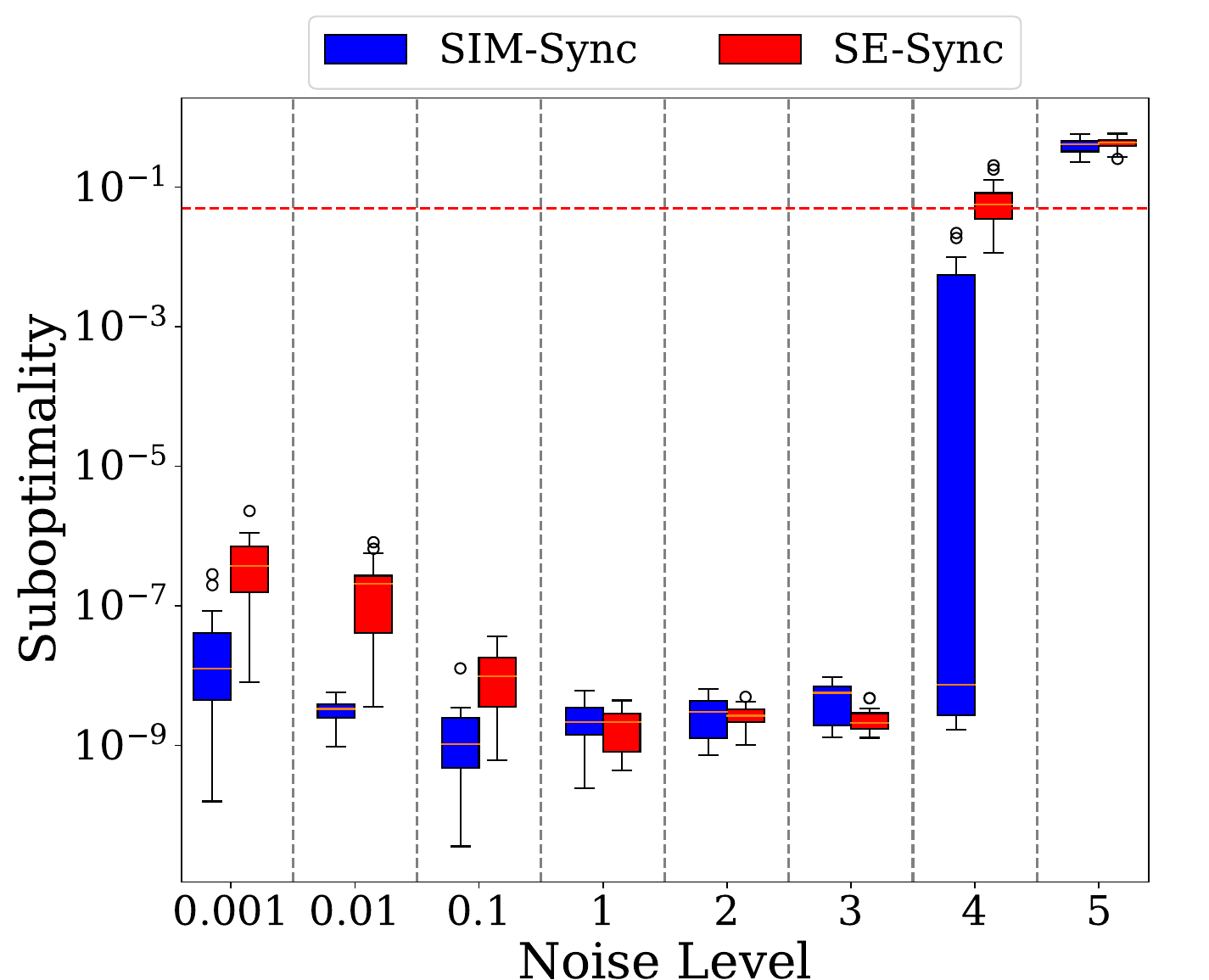}\\
			{\smaller (a) Circle}
		\end{minipage}
		&
		\begin{minipage}{0.32\textwidth}%
			\centering%
			\includegraphics[width=\textwidth]{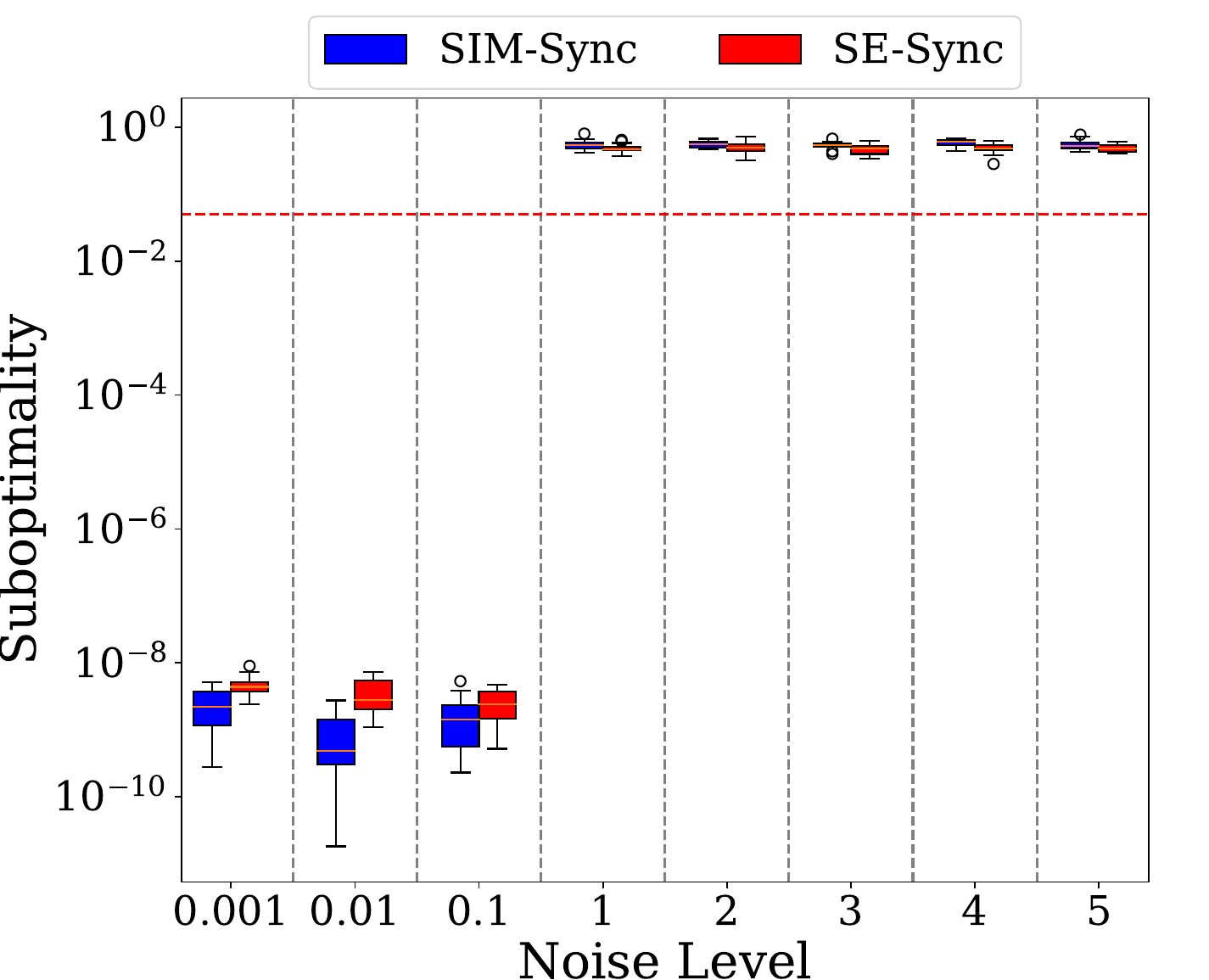}\\
			{\smaller (b) Grid}
		\end{minipage}
		&
		\begin{minipage}{0.32\textwidth}%
			\centering%
			\includegraphics[width=\textwidth]{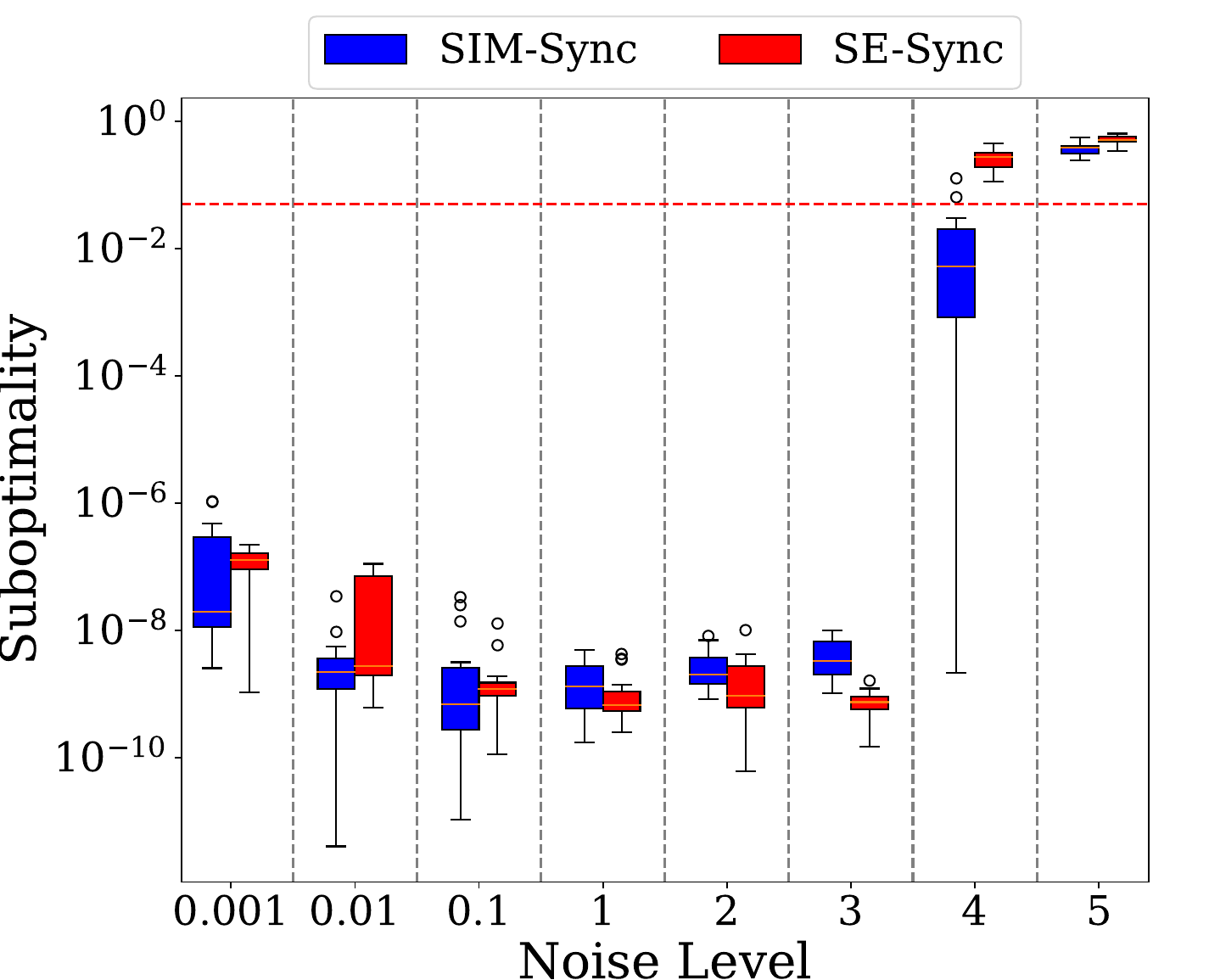}\\
			{\smaller (c) Line}
		\end{minipage}
	\end{tabular}
	\end{minipage}
	\vspace{-4mm} 
	\caption{
		Rotation errors (top row), translation errors (middle row), and suboptimality (bottom row) of \simsync, \sesync and \sesyncgtwoo in scale-free synchronization.} 
        \label{fig:scale_free_results}
	\vspace{-8mm} 
	\end{center}
\end{figure*}

\subsection{Scale Regularization}
\label{regularization}

\textbf{Setup.} 
We study the effect of the number of poses $N$ and the regularization factor $\lambda$ on the performance of \simsync. We use the same circle, grid, and line datasets in Section \ref{scale-known}, where the scaling factor is unknown and randomly generated in $[0.9, 1.1]$. The noise level $\sigma$ is fixed to $0.01$. For each dataset, we generate $n=100$ points, and vary the number of poses $N \in \{10, 50, 100, 200, 400\}$. The regularization factor $\lambda$ is tested with $\{0, 1, 10, 100\}$ for the circle and line datasets, while an additional $\lambda=200$ for the grid dataset. To ensure statistical significance, we perform 20 Monte Carlo simulations for each combination of $N$ and $\lambda$.

\textbf{Results.} 
Fig.~\ref{fig:ba-scale-reg}(a)(c) plot the averaged scale estimation, rotation error, and translation error \wrt number of poses $N$ on the circle dataset and the line dataset, with different colors representing different regularization factors $\lambda$. We observe that (i) as $N$ increases, translation estimation gets slightly worse, but rotation estimation remains unaffected; (ii) the scale estimation does not contract a lot as $N$ increases, even without scale regularization (\ie $\lambda=0$). Fig.~\ref{fig:ba-scale-reg}(b) shows the same results on the grid dataset, where we clearly observe contraction. Without regularization, the average scale decreases to $0.5$ when $N=400$, which also leads to poor translation estimation. With regularization $\lambda=200$, however, we see that contraction is effectively prevented and the translation error also gets improved. This suggests that regularization improves the performance of \simsync when $N$ is large. It is interesting to see that rotation estimation is not affected by $N$ and $\lambda$. This makes sense because scale and translation are coupled, while rotation is independent. We suspect that the circle graph and the line graph have a certain type of ``\emph{rigidity}'' that makes them more robust to contraction, while the grid graph has weaker ``\emph{rigidity}'' (\ie it is easier to bend and twist the trajectory in Fig.~\ref{Grid_vis}).


\begin{figure*}[h]
	\begin{center}
		\begin{tabular}{ccc}
			\begin{minipage}{0.32\textwidth}
				\centering
				\includegraphics[width=\textwidth]{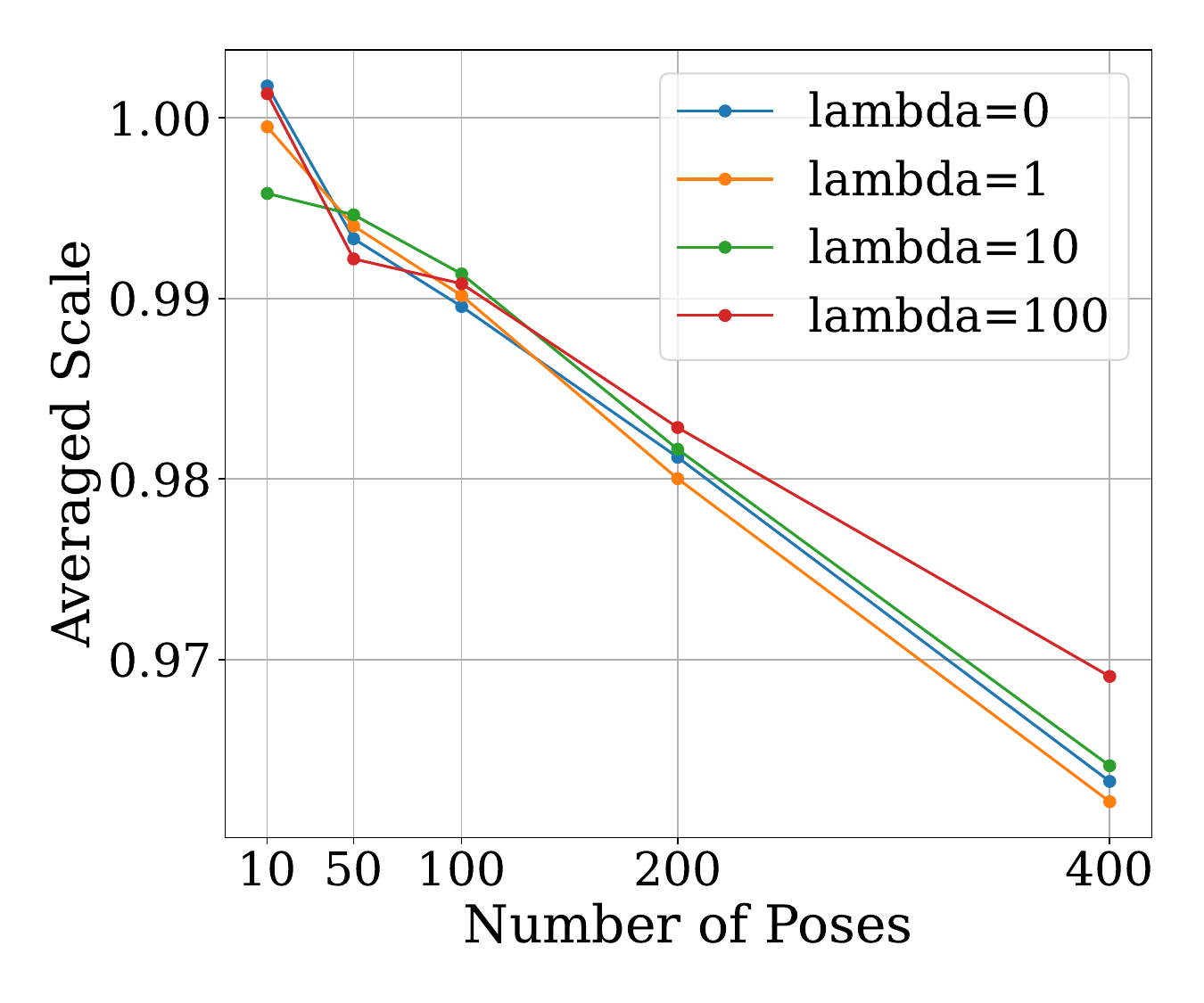}
			\end{minipage}
			&
			\begin{minipage}{0.32\textwidth}
				\centering
				\includegraphics[width=\textwidth]{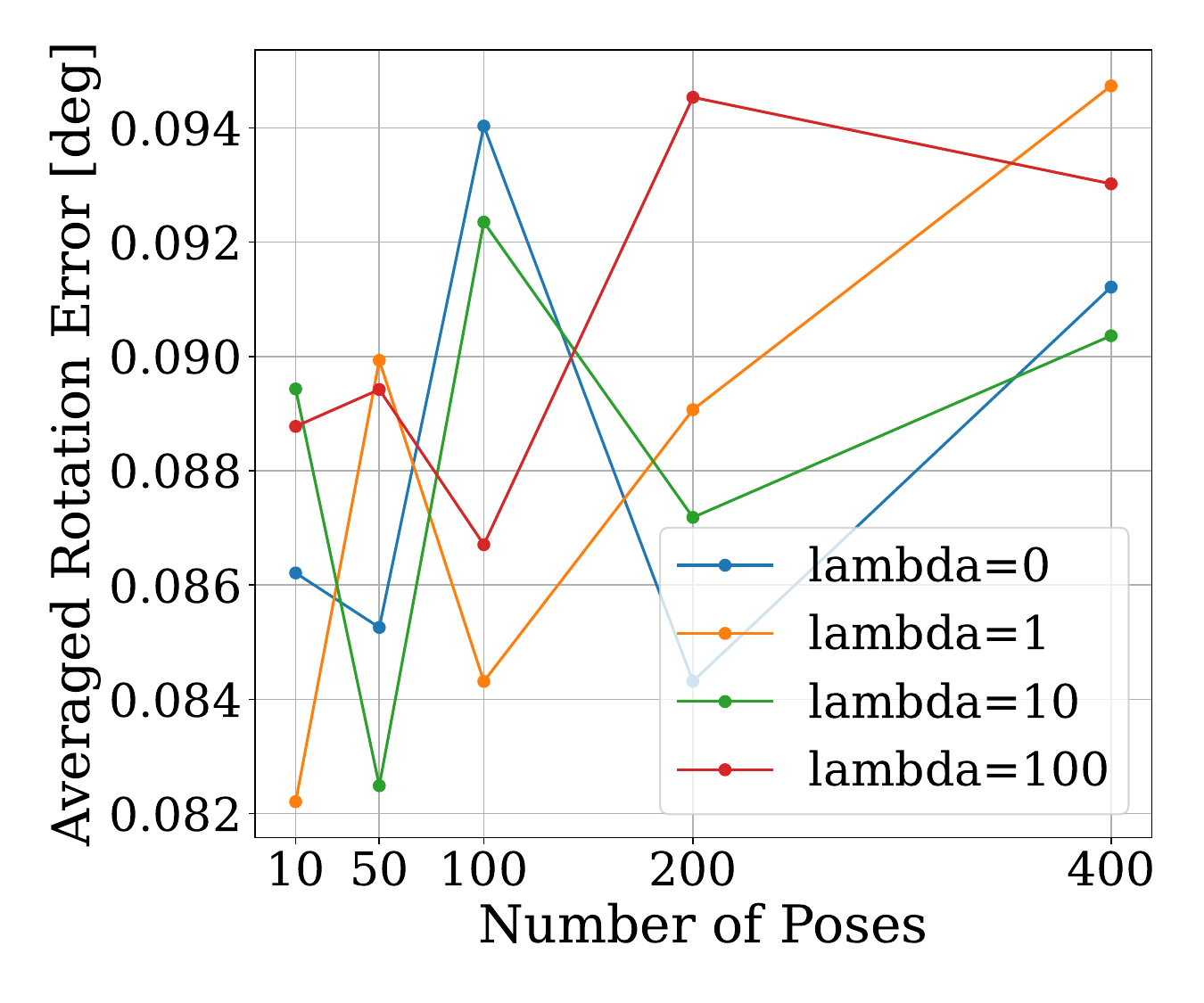}
			\end{minipage}
			&
			\begin{minipage}{0.32\textwidth}
				\centering
				\includegraphics[width=\textwidth]{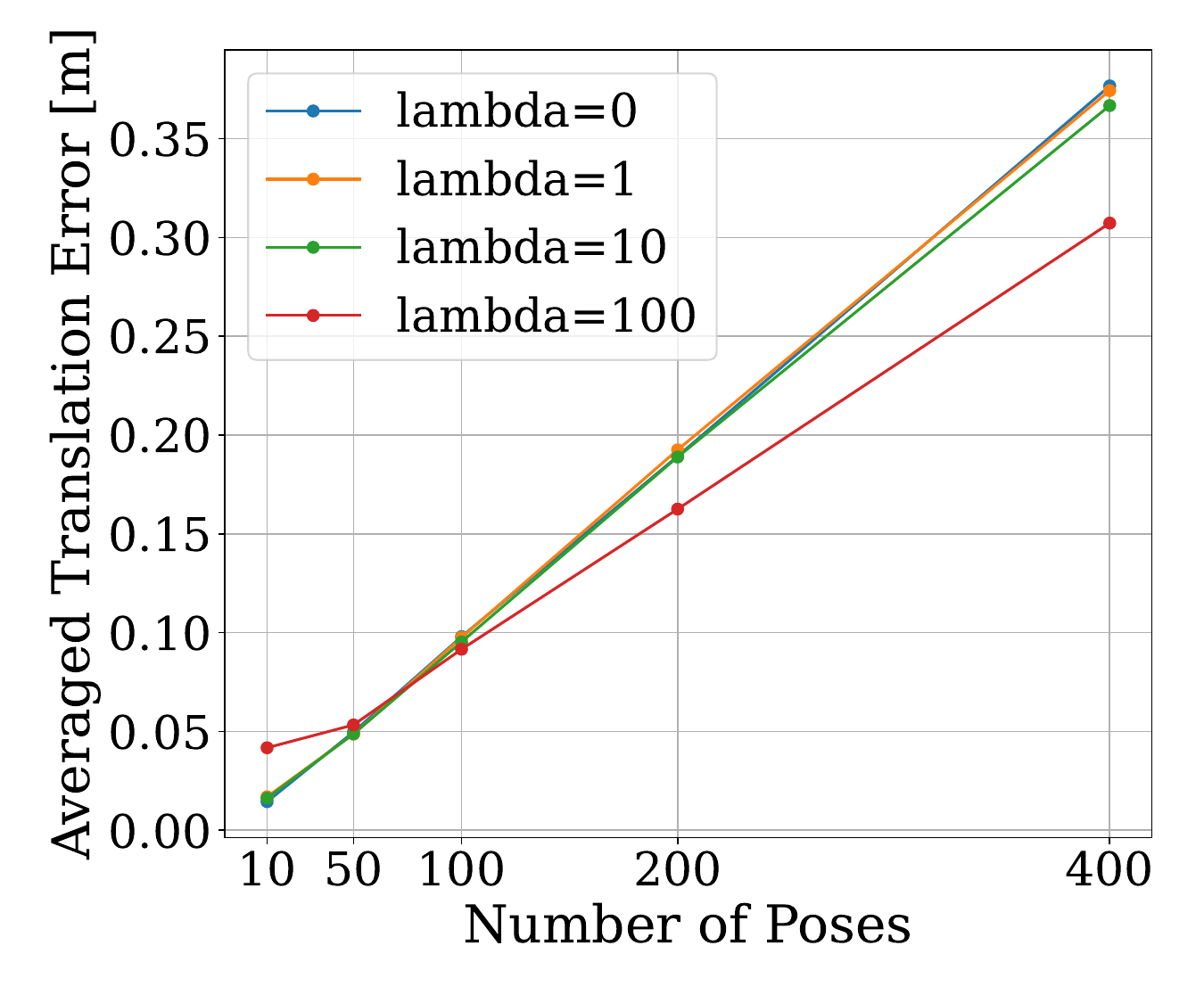}
			\end{minipage} \\
			\multicolumn{3}{c}{{\smaller (a) Circle}}\\
			\begin{minipage}{0.32\textwidth}
				\centering
				\includegraphics[width=\textwidth]{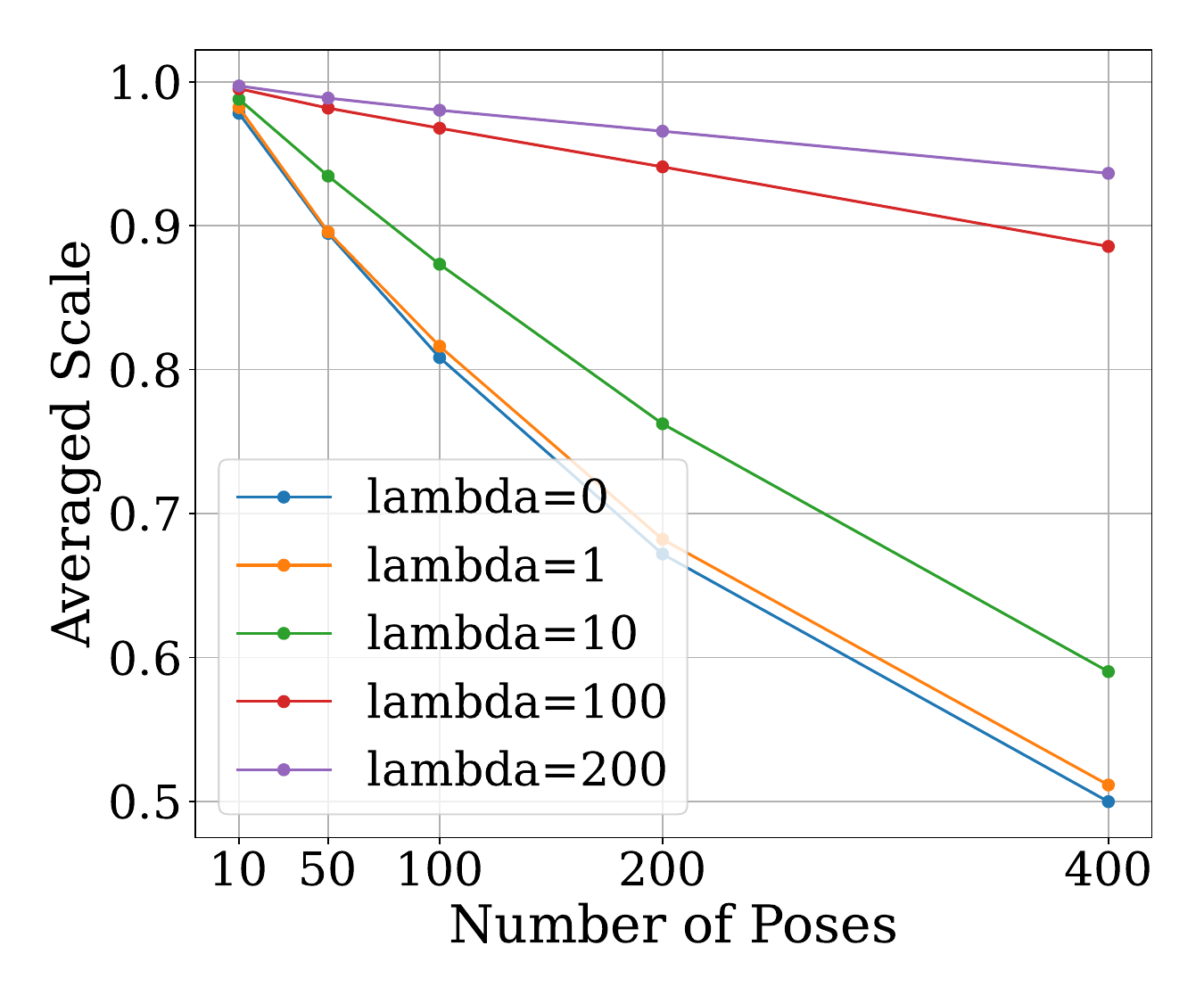}
			\end{minipage}
			&
			\begin{minipage}{0.32\textwidth}
				\centering
				\includegraphics[width=\textwidth]{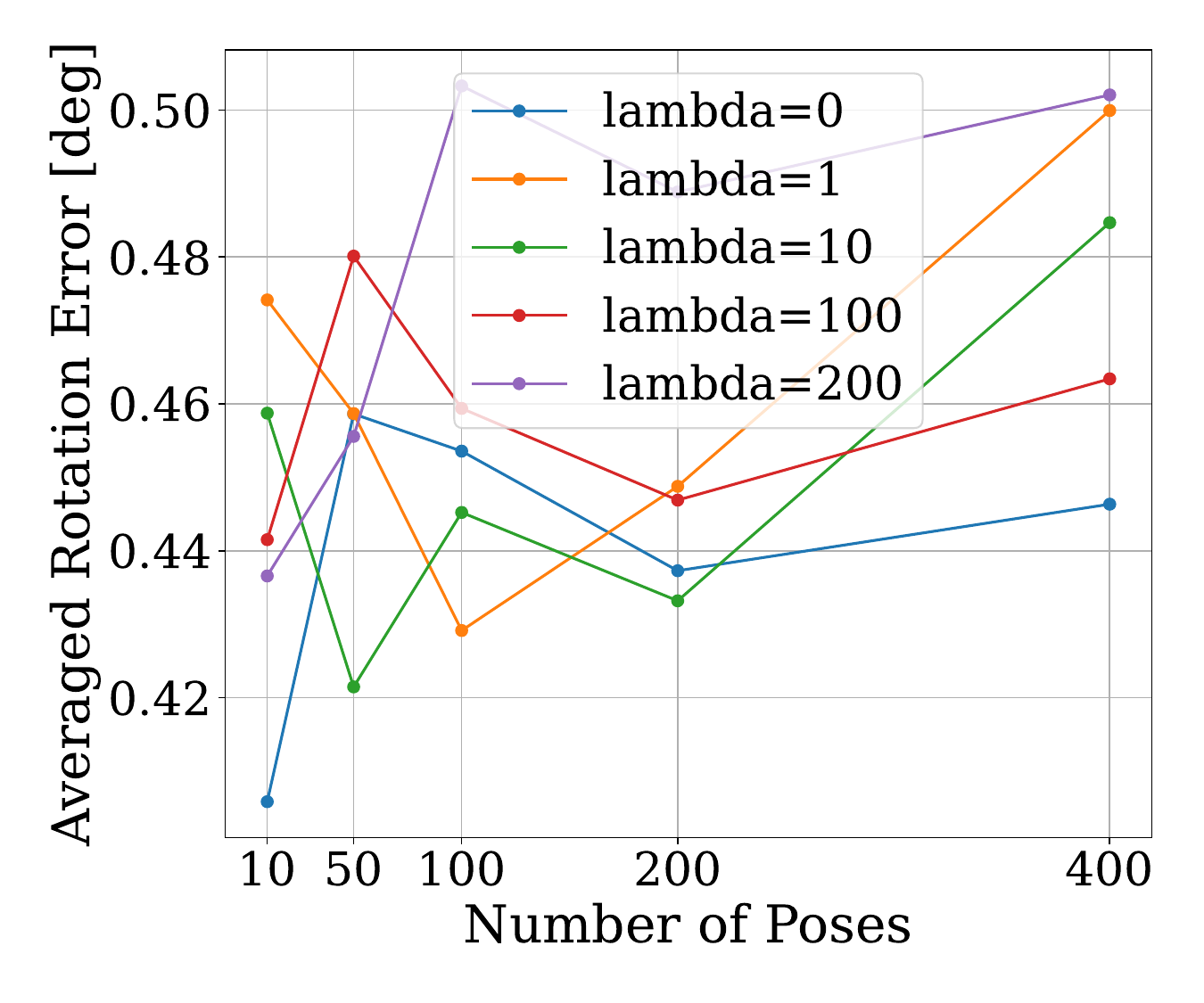}
			\end{minipage}
			&
			\begin{minipage}{0.32\textwidth}
				\centering
				\includegraphics[width=\textwidth]{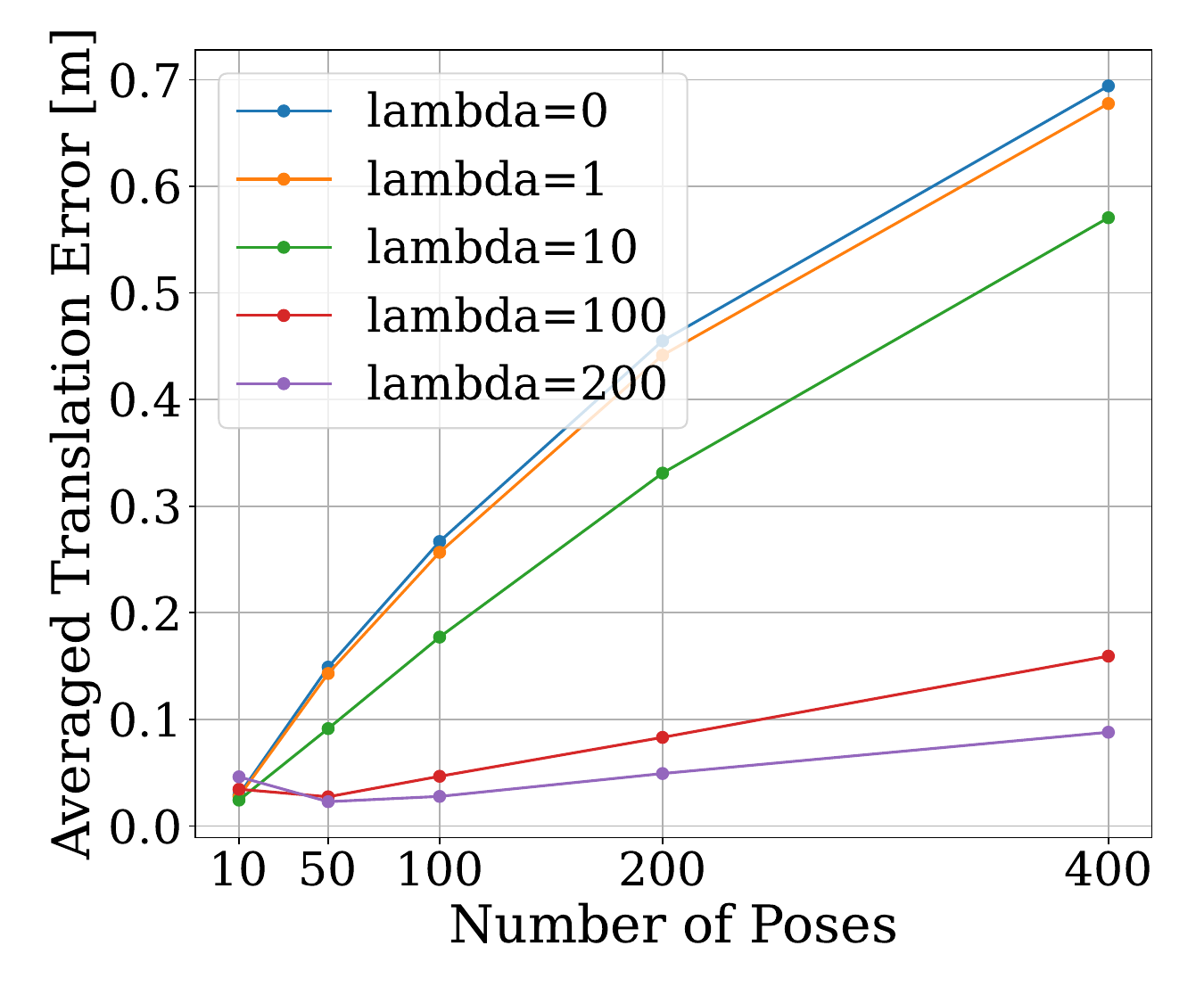}
			\end{minipage}\\
			\multicolumn{3}{c}{{\smaller (b) Grid}}\\
			\begin{minipage}{0.32\textwidth}
				\centering
				\includegraphics[width=\textwidth]{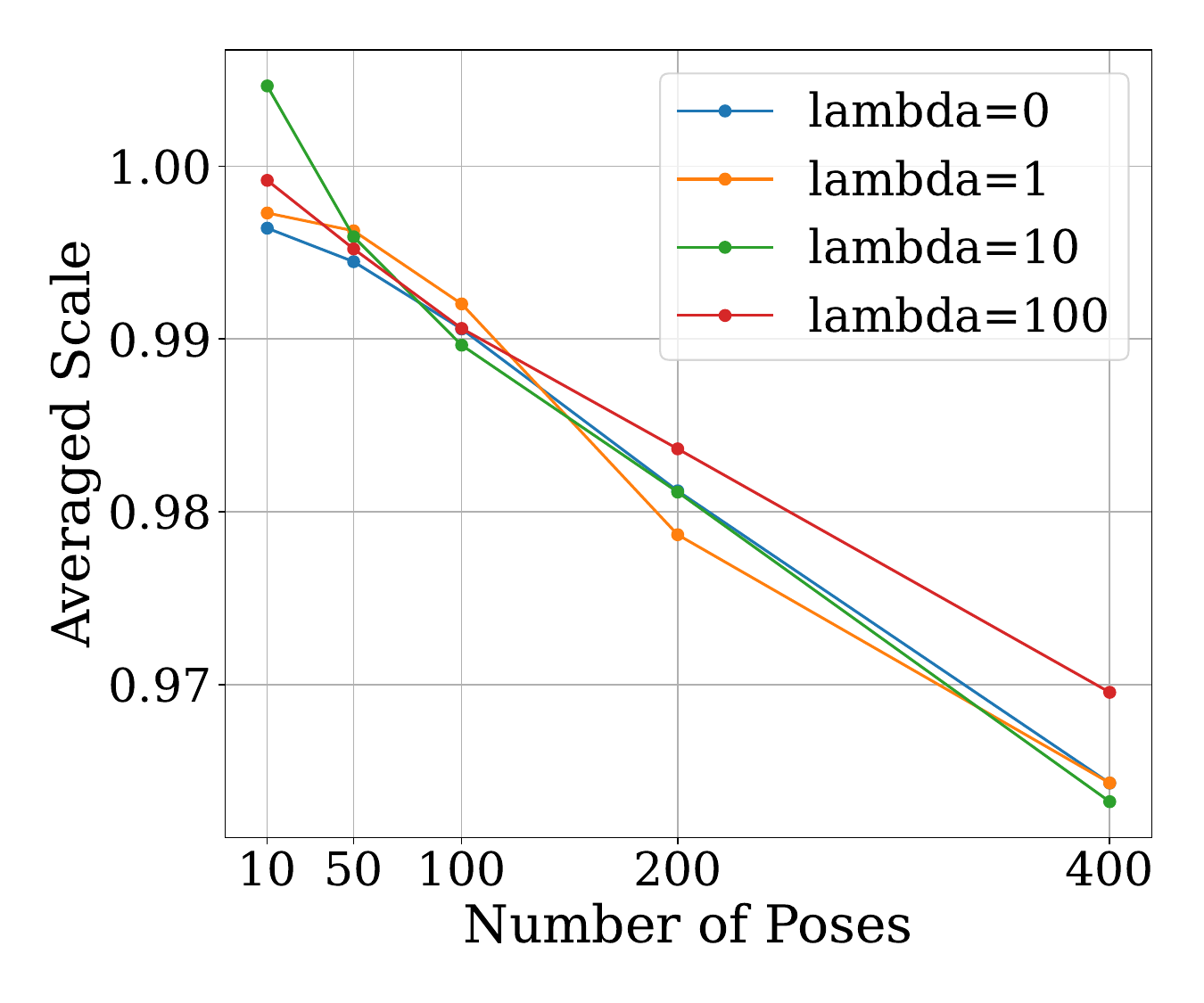}
			\end{minipage}
			&
			\begin{minipage}{0.32\textwidth}
				\centering
				\includegraphics[width=\textwidth]{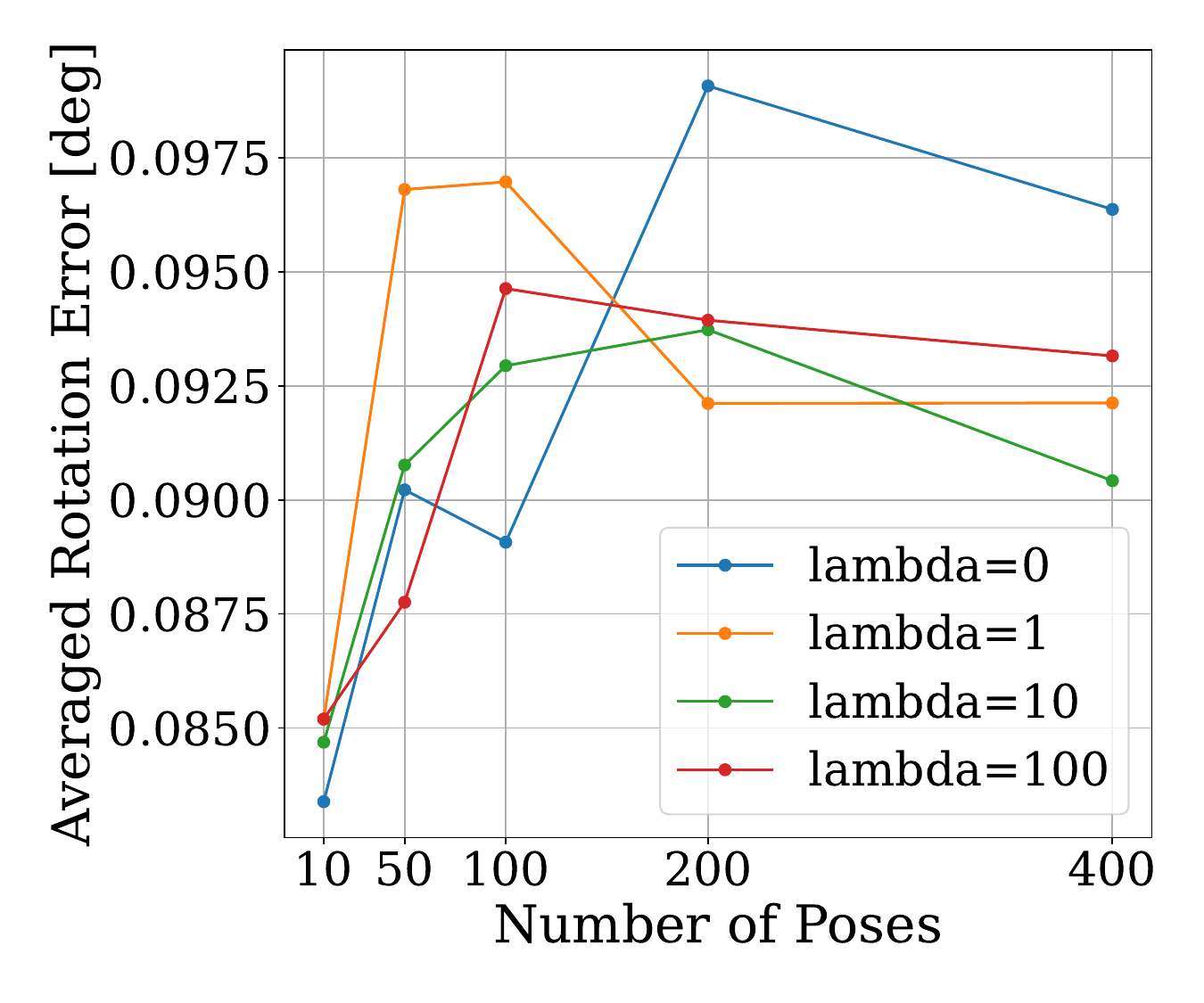}
			\end{minipage}
			&
			\begin{minipage}{0.32\textwidth}
				\centering
				\includegraphics[width=\textwidth]{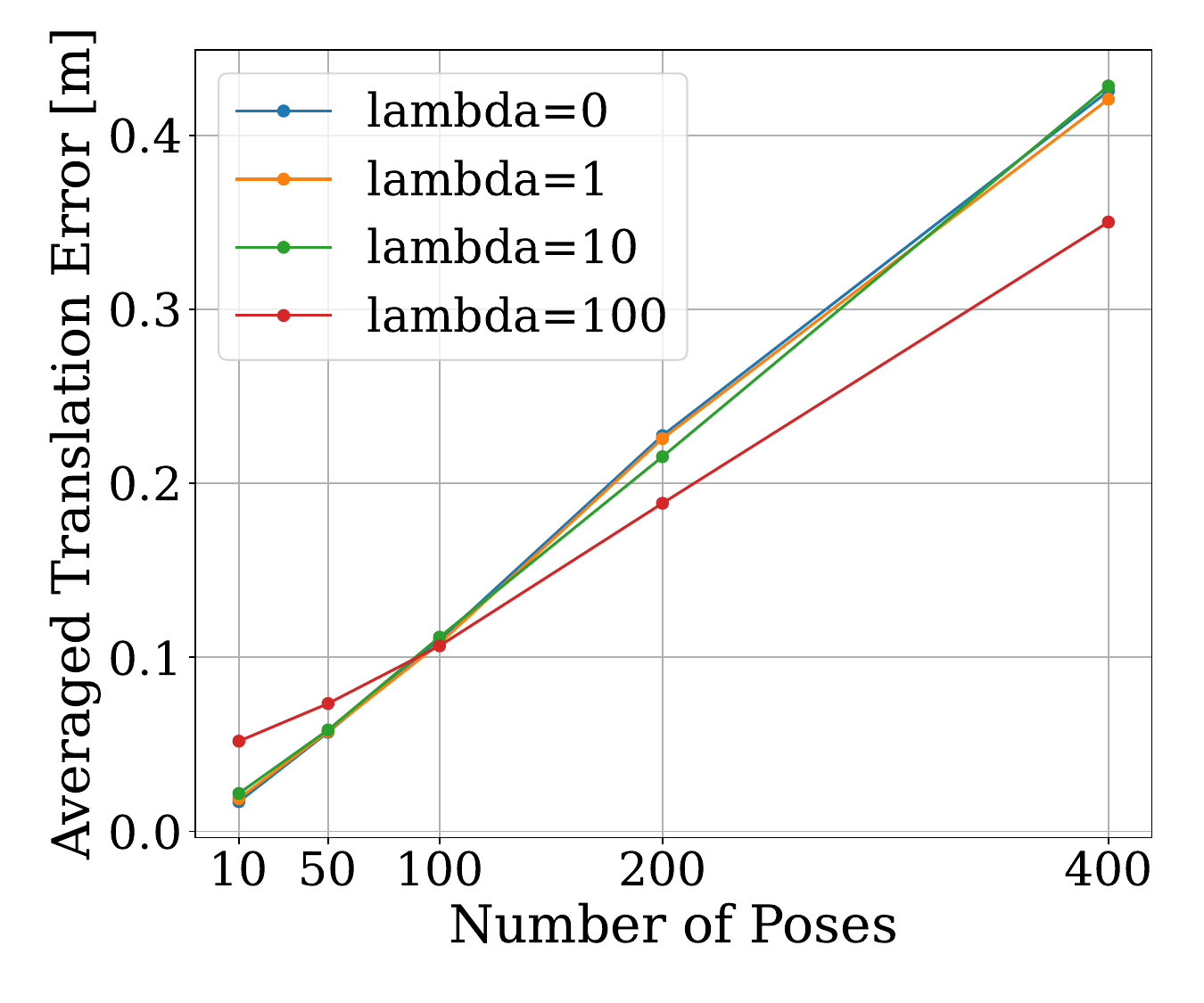}
			\end{minipage}\\
			\multicolumn{3}{c}{{\smaller (c) Line}}
		\end{tabular}
	\end{center}
	\vspace{-7mm}
	\caption{Scale regularization of \simsync in (a) circle, (b) grid, and (c) line datasets.}
	\label{fig:ba-scale-reg}
	\vspace{-4mm}
\end{figure*}

\subsection{Outlier Rejection with Robust Estimators}
\label{outlier-rejection}
\textbf{Setup.} 
We follow the same setup as in previous sections to generate the circle, grid, and line datasets with $N=50$ poses. We sample the scale $s_i$ uniformly from $[0.9,1.1]$ and choose the noise $\sigma = 0.01$. To generate outliers, we randomly replace a fraction of the edge-wise point correspondences $(\widehat{P}_i, \widehat{P}_j)$ with outliers generated from $\calN(\zero,\eye_3)$. We sweep the outlier rate from $0$ to $90\%$ and perform 20 Monte Carlo simulations at each outlier rate.

\textbf{Robustify \simsync}. We propose three ways to robustify \simsync.
\begin{enumerate}
    \item \simsyncgnc. This approach directly wraps the \simsync solver in the \gnc framework~\cite{yang2020graduated} with a truncated least squares (TLS) robust cost function, as shown in the following optimization. 
    \bea \label{eq:simsyncgnc}
    \min_{\substack{s_i > 0, R_i \in \SOthree, t_i \in \Real{3} \\ i=1,\dots,N } } \sum_{(i,j) \in \calE} \sum_{k=1}^{n_{ij}} \min \cbrace{ \frac{\norm{ \parentheses{R_i {(s_i d_{i,k} \tldp_{i,k} )} + t_i} - \parentheses{R_j(s_j d_{j,k} \tldp_{j,k} ) + t_j} }^2}{\beta^2}, 1},
    \eea
    where $\beta$ is set according to Appendix \ref{noise-scale-simsync}.
    To solve \eqref{eq:simsyncgnc}, \simsyncgnc starts with all weights $w_{ij,k} = 1$ in \eqref{eq:sba} and  gradually sets some of the weights to zero based on the residuals.

    \item \gncsimsync. This approach is a two-step algorithm. In step one, we solve the following TLS \emph{scaled point cloud registration} problem over each edge $(i,j) \in \calE$
    \bea \label{eq:scaledpairregistration}
    \min_{s_{ij}>0, R_{ij} \in \SOthree, t_{ij} \in \Real{3}} \sum_{k=1}^{n_{ij}} \min \cbrace{ \frac{\Vert s_{ij} R_{ij} p_{j,k} + t_{ij} - p_{i,k} \Vert^2}{\beta^2}, 1    }
    \eea
    using \gnc with a nonminimal solver that we develop, presented in Appendix \ref{weighted_umeyama}, based on Umeyama's method \cite{umeyama1991least}. $\beta$ is set according to Appendix \ref{noise with scale}.

    In step two, we remove all outliers deemed by \gnc and solve \eqref{eq:sba}.
    \item \teasersimsync. Since problem \eqref{eq:scaledpairregistration} is exactly the problem solved by \teaser~\cite{yang2020teaser}, we directly apply \teaser to solve \eqref{eq:scaledpairregistration} and then pass the inliers to \eqref{eq:sba}. 
    
\end{enumerate}

\textbf{Results.} 
Fig.~\ref{fig:robust_simsync} shows the rotation, translation, and scale estimation errors \wrt outlier rates in the circle, grid, and line datasets. We first observe that (i) \simsyncgnc fails at outlier rate $10\%$. This shows that \gnc plus a nonminimal solver does not always work, especially when the model to be estimated is high-dimensional and the robust estimation problem is more combinatorial. We then observe that (ii) \gncsimsync is robust against $50\%$ outliers. This shows that it is a better strategy to first apply \gnc to low-dimensional robust fitting.\footnote{Note that for high outlier rates, there are no data points for \gncsimsync because it fails and produces infinite values that are discarded from the plots.} Lastly, we observe that (iii) \teasersimsync successfully handles outlier rates of $70\%$ and $80\%$.


\begin{figure*}[t]
	\begin{center}
	\begin{minipage}{\textwidth}
	\begin{tabular}{ccc}%
		\begin{minipage}{0.32\textwidth}%
			\centering%
			\includegraphics[width=\textwidth]{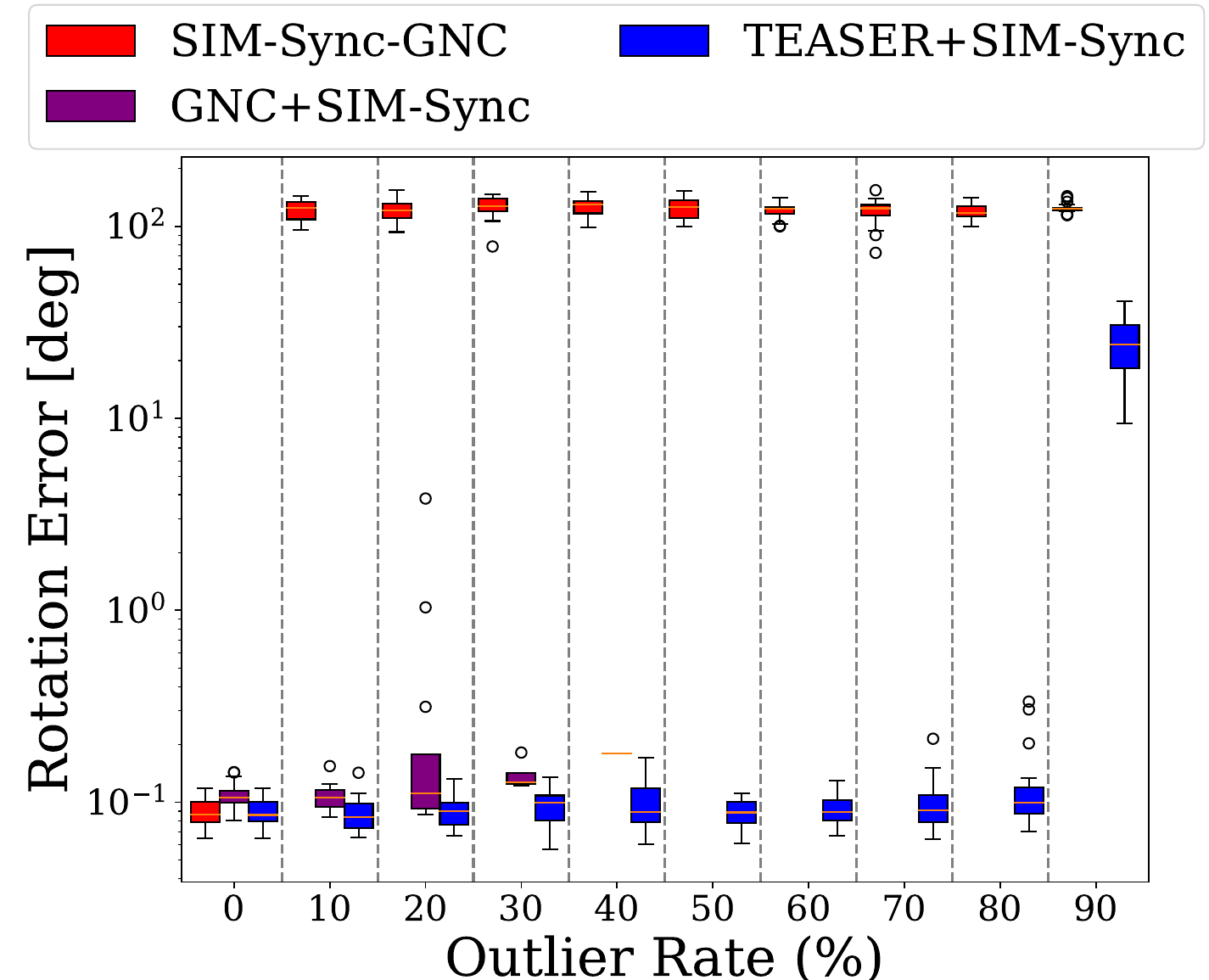}
		\end{minipage}
		&
		\begin{minipage}{0.32\textwidth}%
			\centering%
			\includegraphics[width=\textwidth]{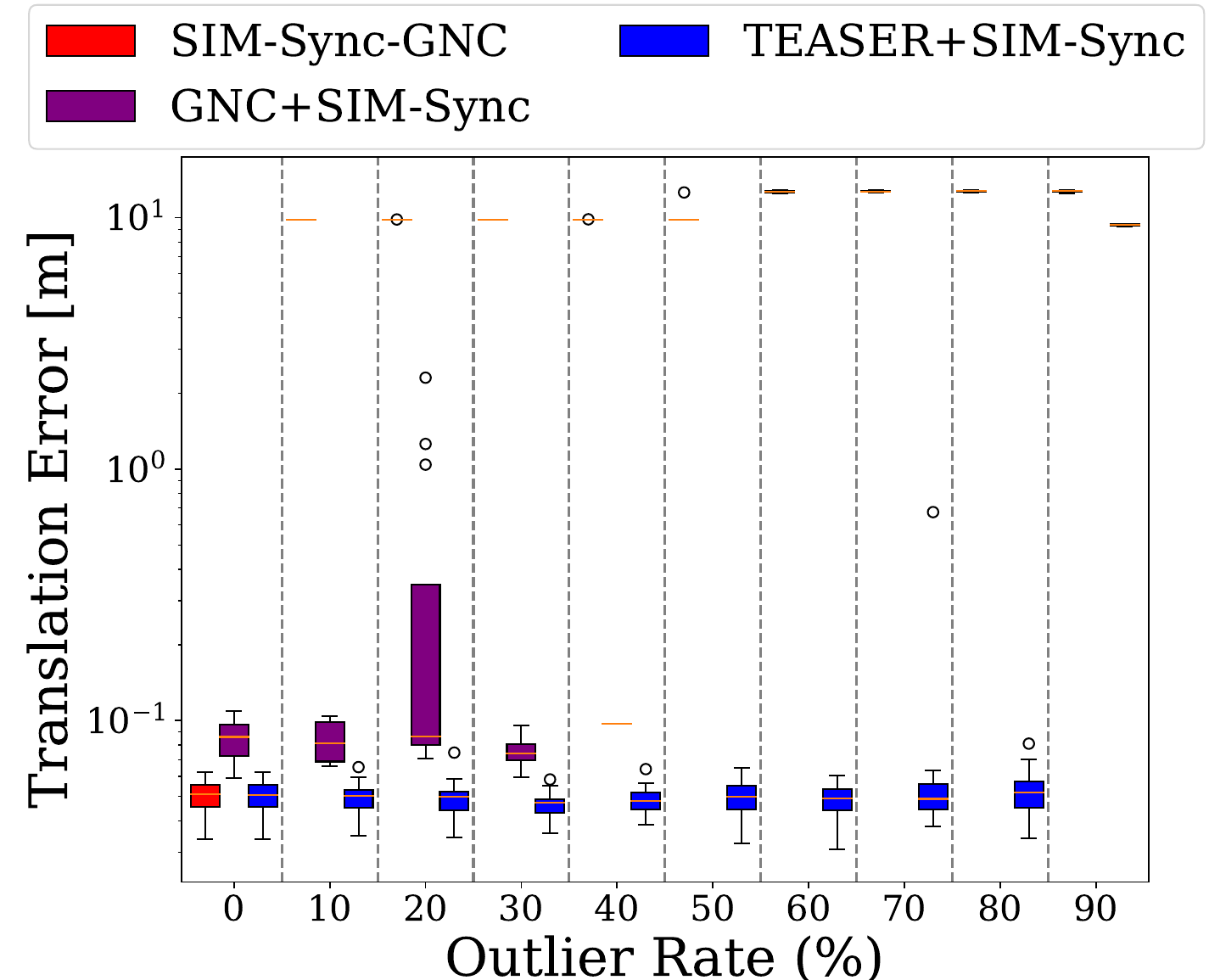}
		\end{minipage}
		&
		\begin{minipage}{0.32\textwidth}%
			\centering%
			\includegraphics[width=\textwidth]{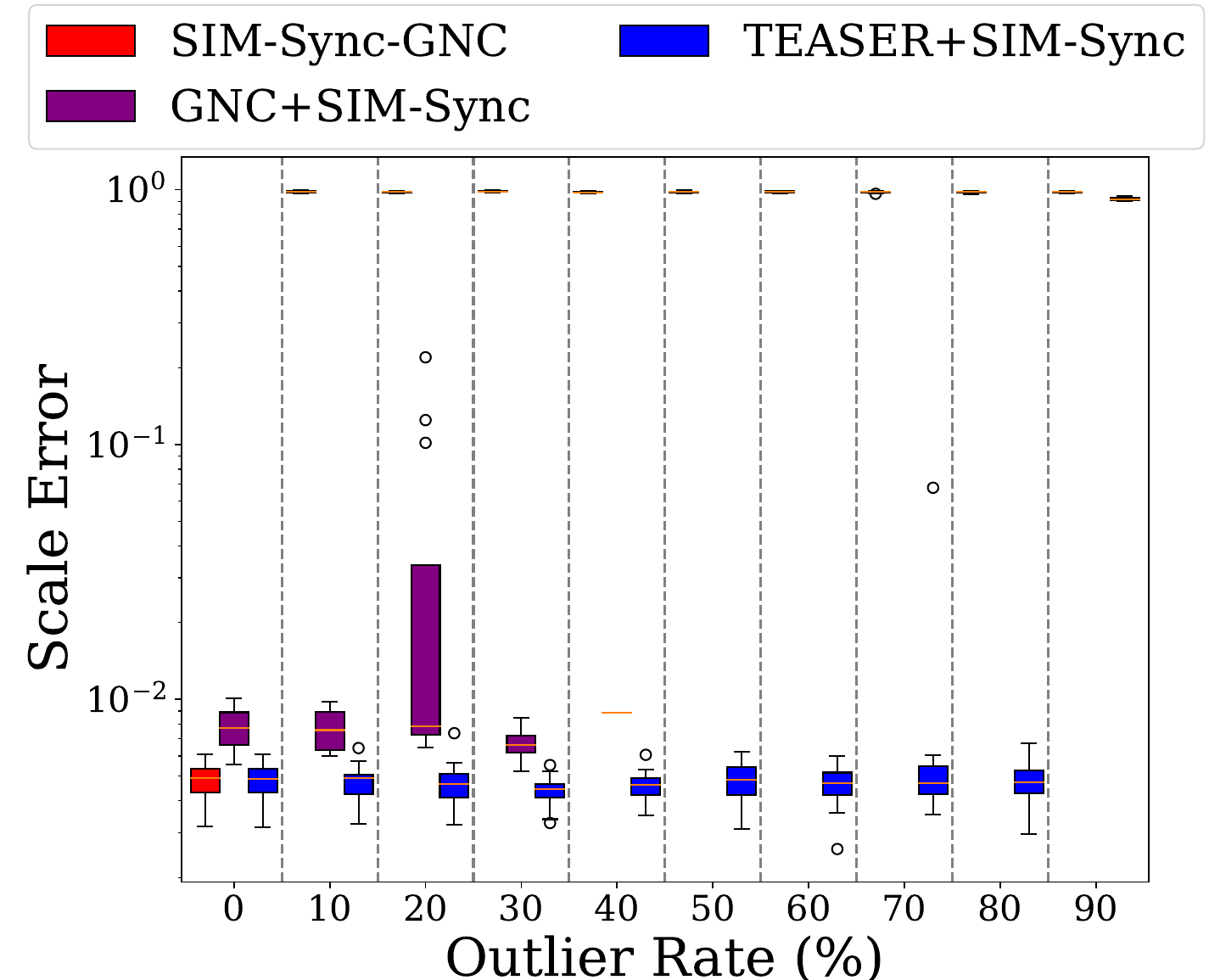}
		\end{minipage} \\
		\multicolumn{3}{c}{{\smaller (a) Circle}}\\
		\begin{minipage}{0.32\textwidth}%
			\centering%
			\includegraphics[width=\textwidth]{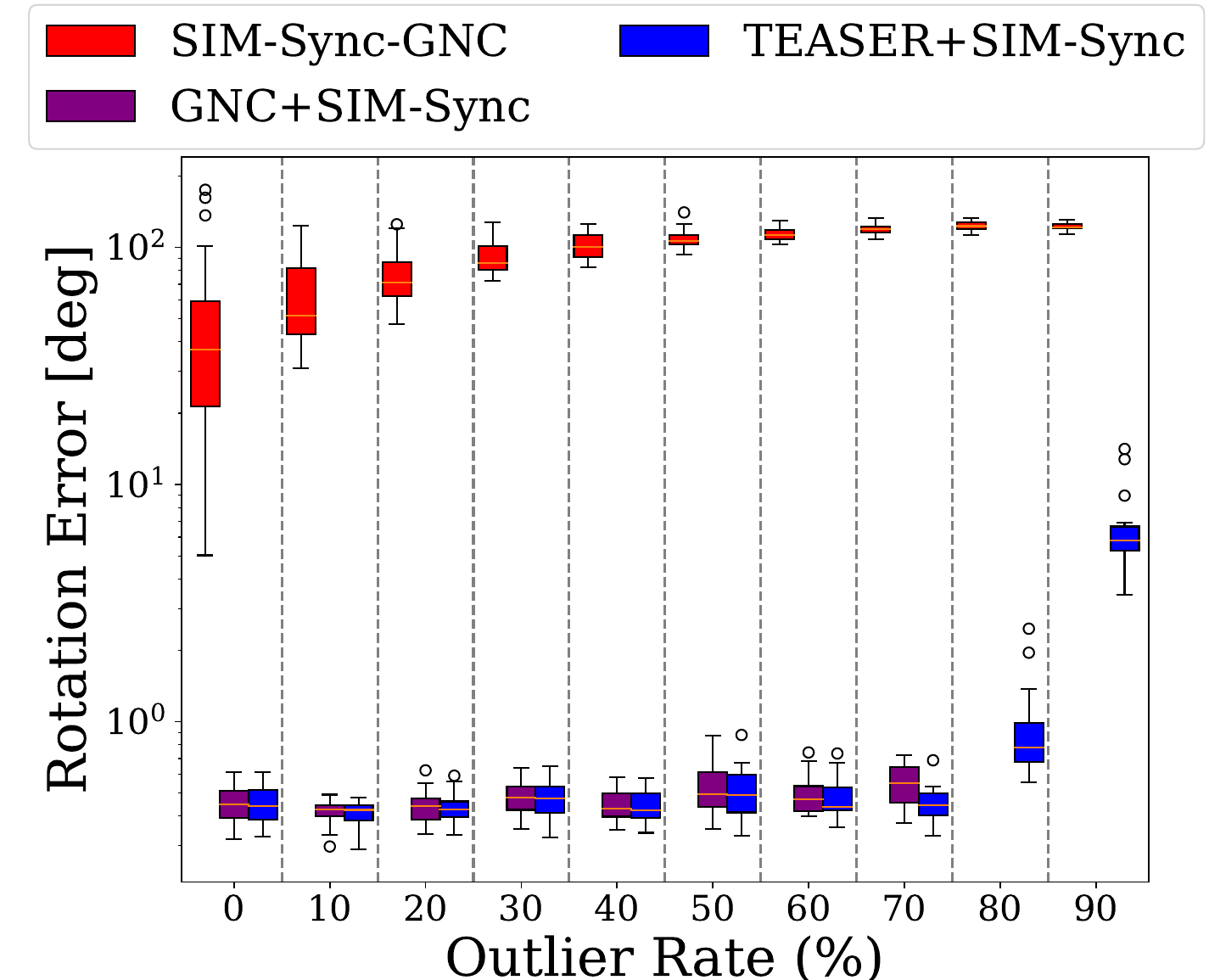}
		\end{minipage}
		&
		\begin{minipage}{0.32\textwidth}%
			\centering%
			\includegraphics[width=\textwidth]{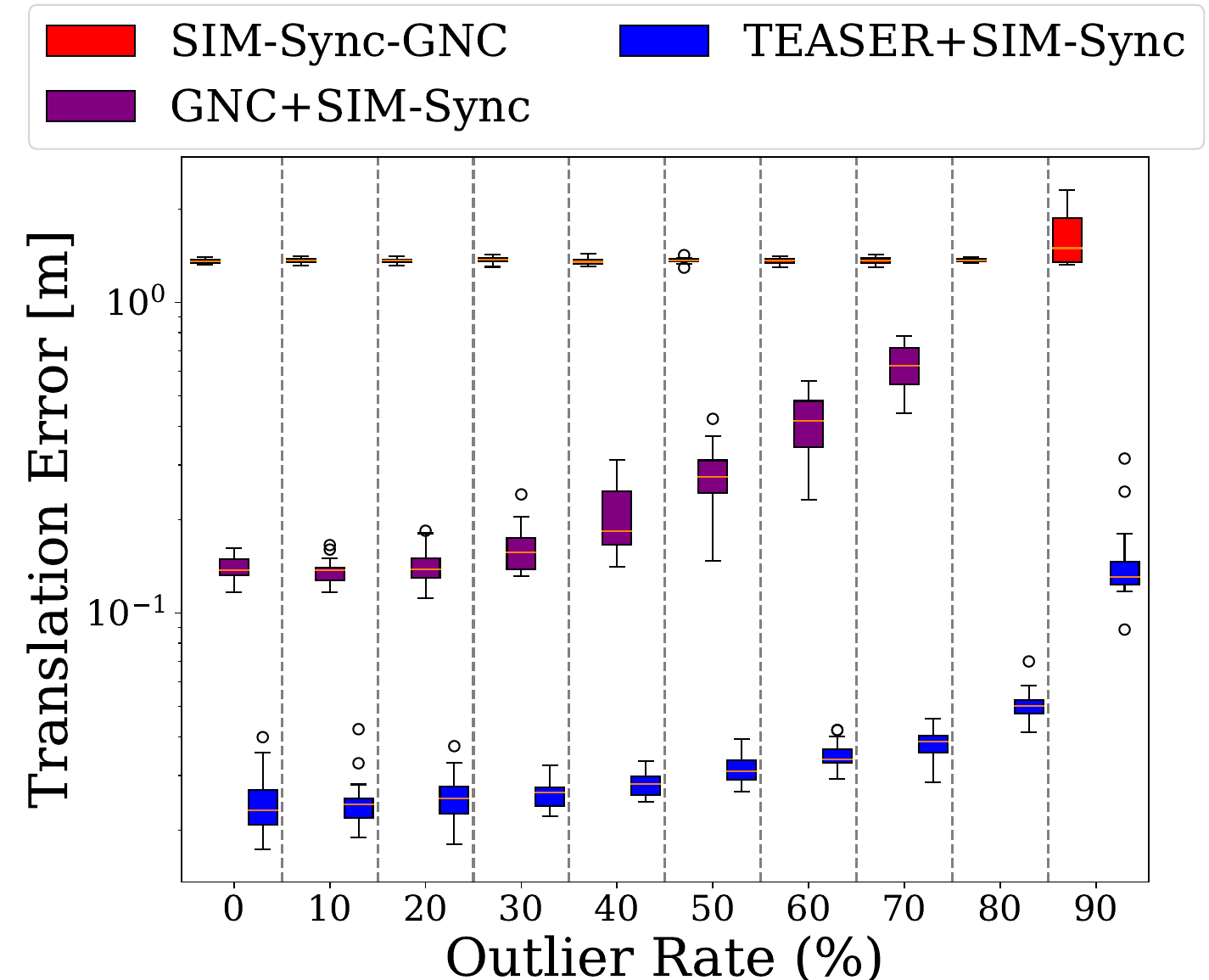}
		\end{minipage}
		&
		\begin{minipage}{0.32\textwidth}%
			\centering%
			\includegraphics[width=\textwidth]{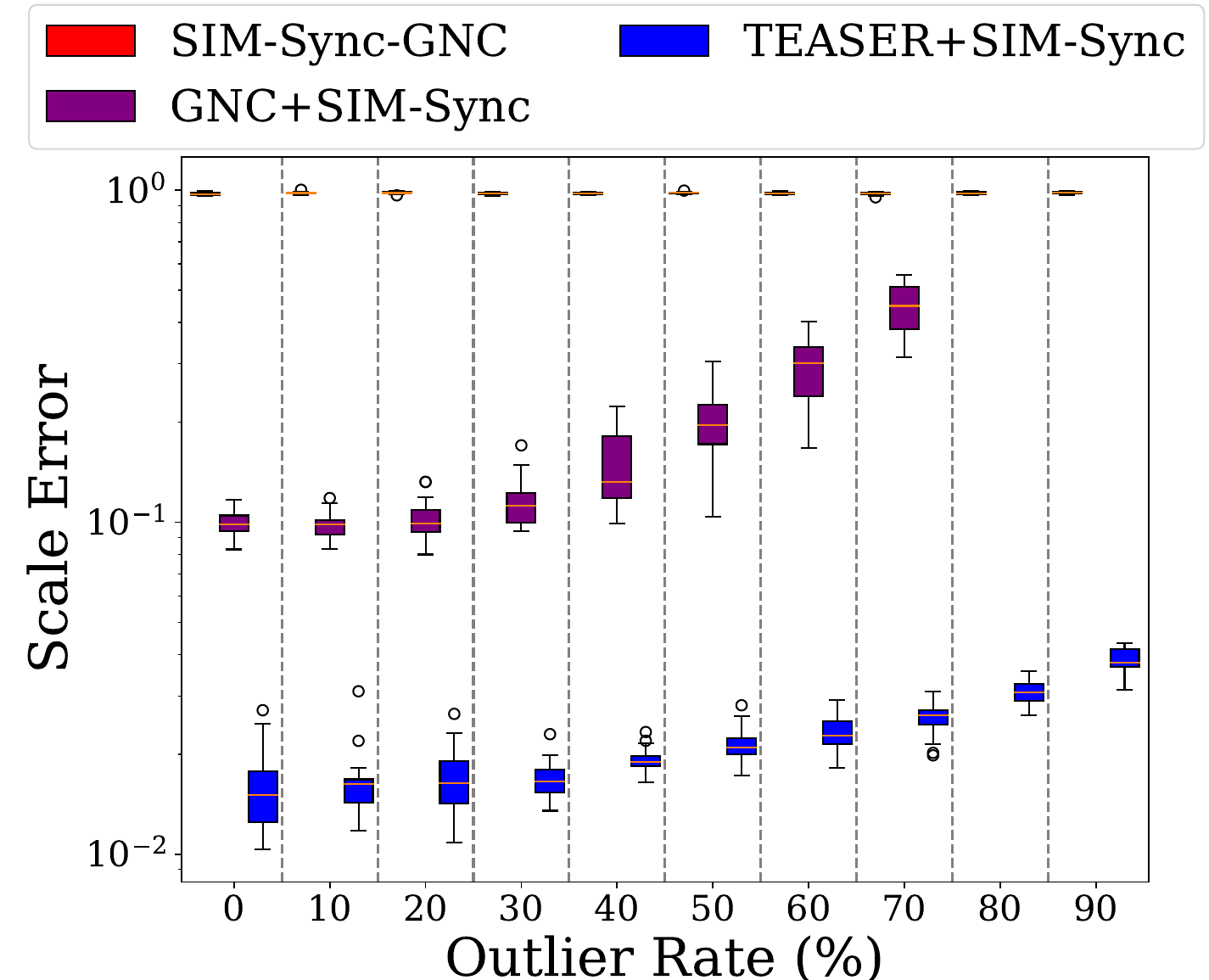}
		\end{minipage} \\
		\multicolumn{3}{c}{{\smaller (b) Grid}}\\
		\begin{minipage}{0.32\textwidth}%
			\centering%
			\includegraphics[width=\textwidth]{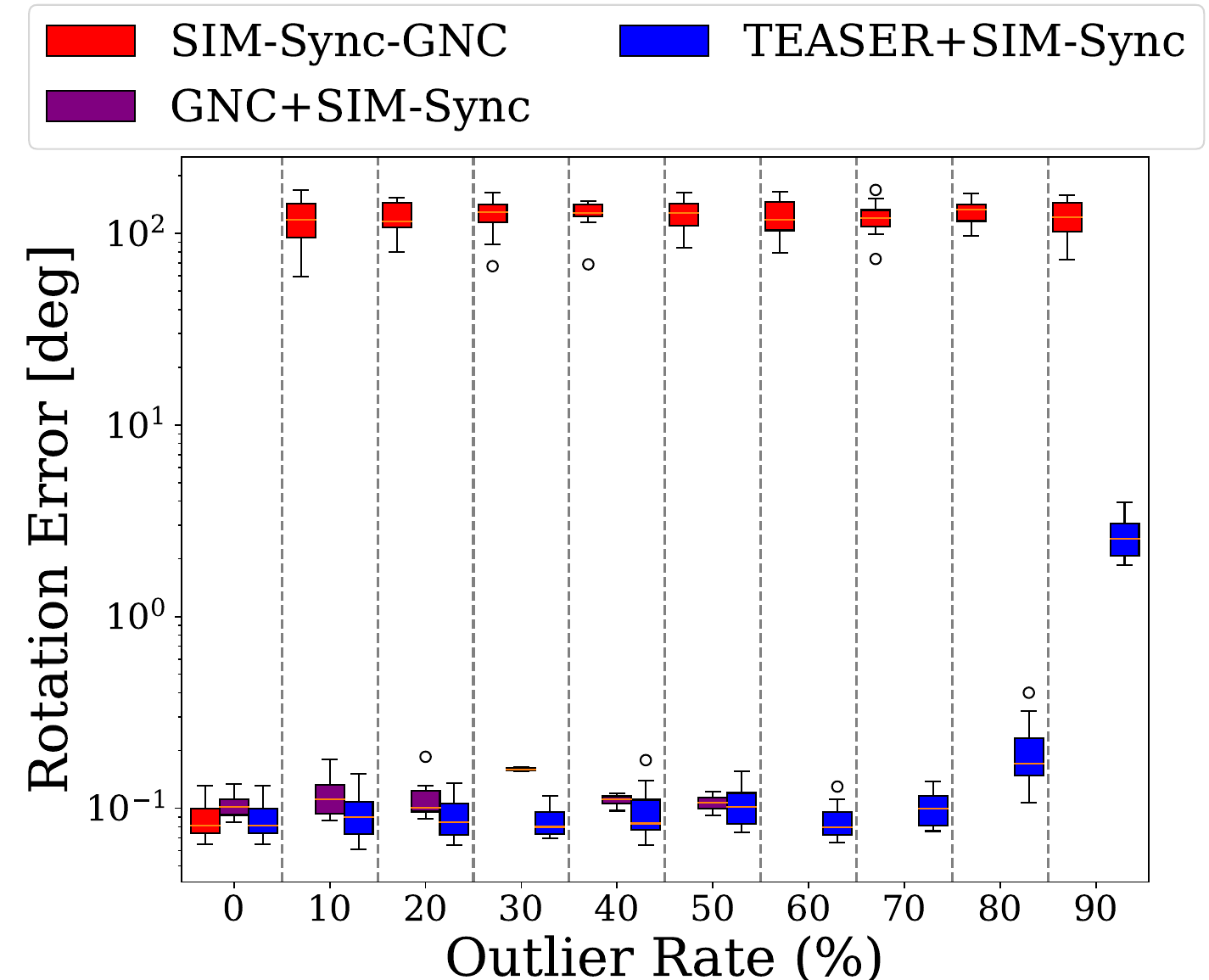}
		\end{minipage}
		&
		\begin{minipage}{0.32\textwidth}%
			\centering%
			\includegraphics[width=\textwidth]{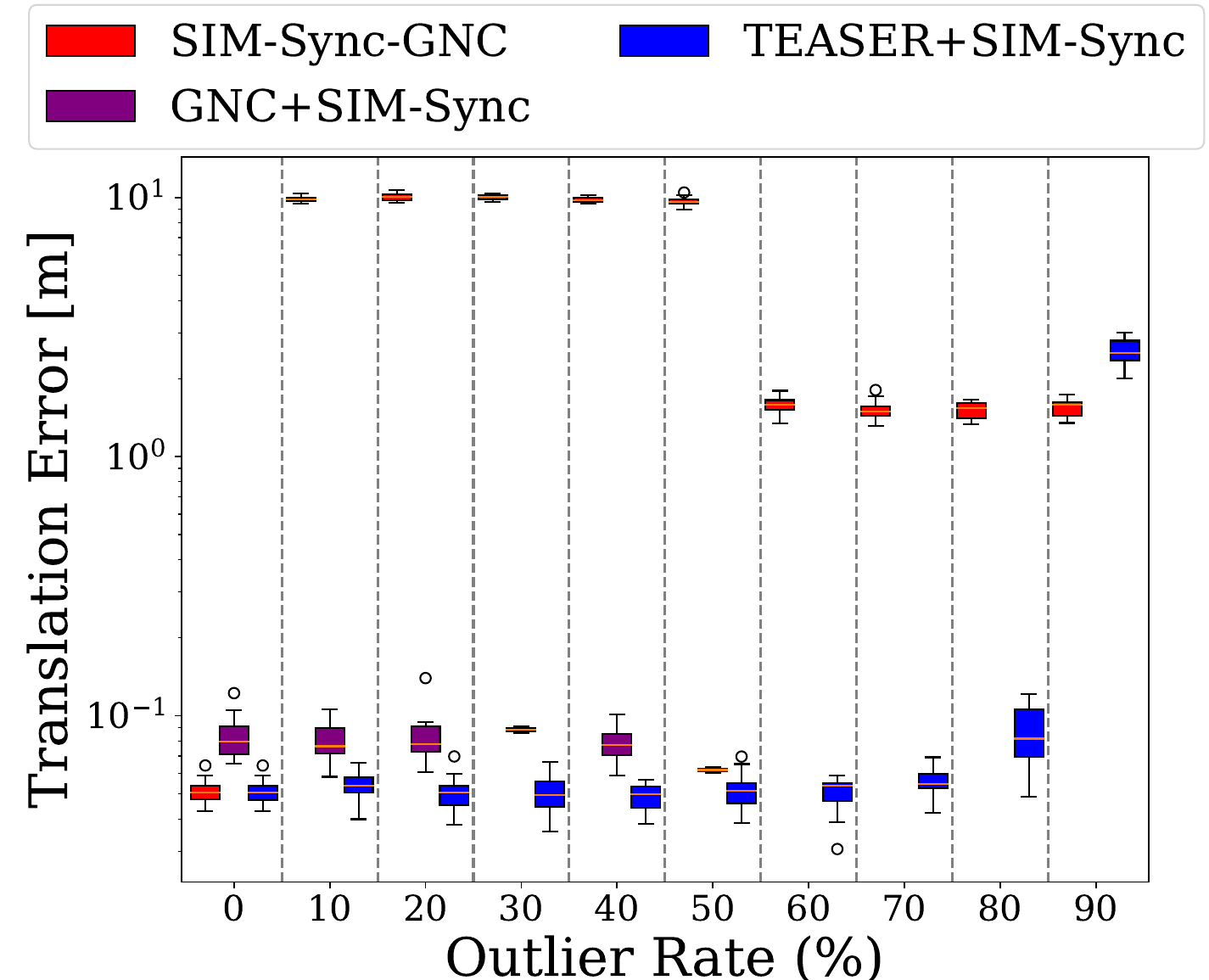}
		\end{minipage}
		&
		\begin{minipage}{0.32\textwidth}%
			\centering%
			\includegraphics[width=\textwidth]{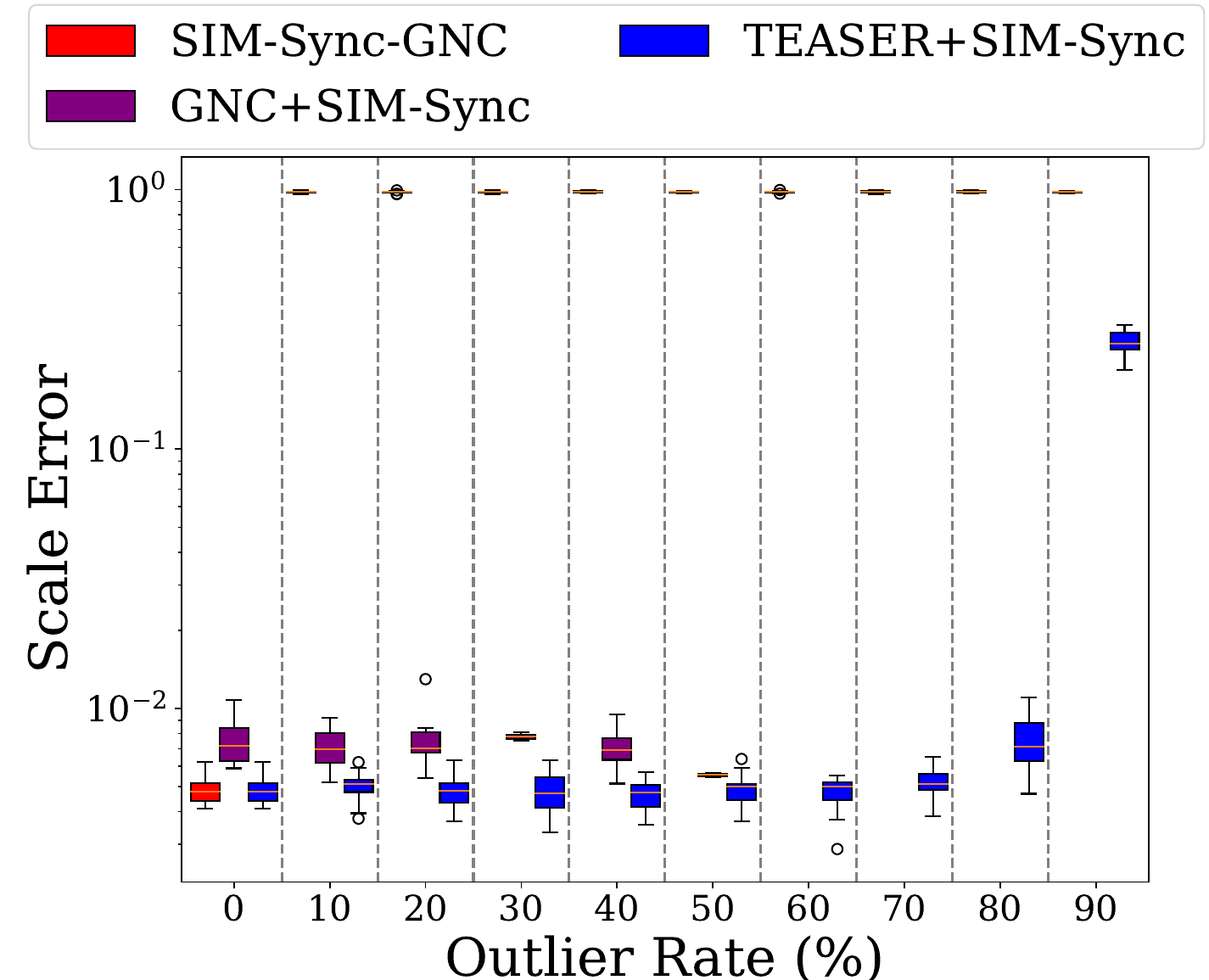}
		\end{minipage} \\
		\multicolumn{3}{c}{{\smaller (c) Line}}
	\end{tabular}
	\end{minipage}
	\vspace{-3mm} 
	\caption{Robustness of \simsyncgnc, \gncsimsync and \teasersimsync in the (a) circle, (b) grid, and (c) line datasets.} 
        \label{fig:robust_simsync}
	\vspace{-7mm} 
	\end{center}
\end{figure*}

\subsection{BAL Experiments}
\label{bal-experiment}

\textbf{Setup.} We test two sequences in the bundle adjustment dataset \bal~\cite{agarwal10eccv-ba}: the \texttt{dubrovnik-16-22106} sequence and the \texttt{ladyburg-318-41628} sequence. The former sequence consists of 16 poses with 22,106 points, and the latter sequence consists of 318 poses with 41,628 points. Both sequences provide pixel-wise correspondences for frame pairs, and these correspondences are contaminated with outliers. Since no images are provided, we cannot use a learned module to predict depth. Therefore, we use the $z$-component of the ground truth point position in camera frame $i$, which is computed by using the groundtruth camera pose to transform the point from the global frame to the $i$-th camera frame. 
Consequently, the scaling effect is not applicable, and we disable the scale prediction. To remove outliers, we use \teaser. We also test the performance of \teasersimsyncgt, which uses ground truth poses to filter out outlier correspondences.

\textbf{Baseline.}
We compare with two baselines \teasersesync and \ceres. \teasersesync first uses \teaser to estimate pair-wise relative poses and then feed them into \sesync, while \ceres directly optimizes reprojection errors of 3D keypoints.\footnote{We use the official implementation \url{http://ceres-solver.org/nnls_tutorial.html}.} 
We initialize \ceres as follows: (i) camera intrinsics initialized as groundtruth; (ii) 3D keypoints initialized as groundtruth; (iii) the $z$-component of the camera poses are initialized using groundtruth; (iv) the other components of the camera poses are initialized to be zeros. We remark that this initialization strategy using groundtruth values is optimistic.\footnote{Without these groundtruth values as initialization it is difficult to get \ceres to work well.} We also combine \teasersimsync and \teasersesync with \ceres, \ie we use the estimation from \teasersimsync and \teasersesync to initialize \ceres.

\textbf{Results.}
Table~\ref{table: BAL_table_small} shows the quantitative results for the \texttt{dubrovnik-16-22106} sequence (the rotation error and the translation error are averaged over all the nodes). We can see that (i) in the ideal case where outliers are filtered, \teasersimsyncgt achieves very accurate reconstruction; (ii) \teasersimsync and \teasersesync perform worse than \teasersimsyncgt, but with the refinement of \ceres, the final results are accurate as well. In fact, they are better than using \ceres alone. 
Fig.~\ref{BAL_3dRecon} shows qualitative results of the reconstruction, where ICP is used to refine the reconstruction. Both \teasersimsync and \teasersesync did not use \ceres refinement. We see that \teasersimsync already achieves good reconstruction without ICP and \ceres.

Table~\ref{table: BAL_table} shows the results for the \texttt{ladyburg-318-41628} sequence. We observe that \teasersimsyncgt still performs quite well. However, both \teasersimsync and \teasersesync fail to produce accurate pose estimation, though their results are better than \ceres. In fact we see that both \teasersimsync and \teasersesync lose tightness (suboptimality is around $10\%$ and $20\%$). This suggests that the correspondences provided in this sequence is contaminated by a higher amount of outliers (compared with the \texttt{dubrovnik-16-22106} sequence) and the point clouds are also noisier. Fig. \ref{BAL_3dRecon2} shows the qualitative reconstruction results.


\begin{figure}[t]
    \centering
    \includegraphics[width= 1.0\columnwidth]{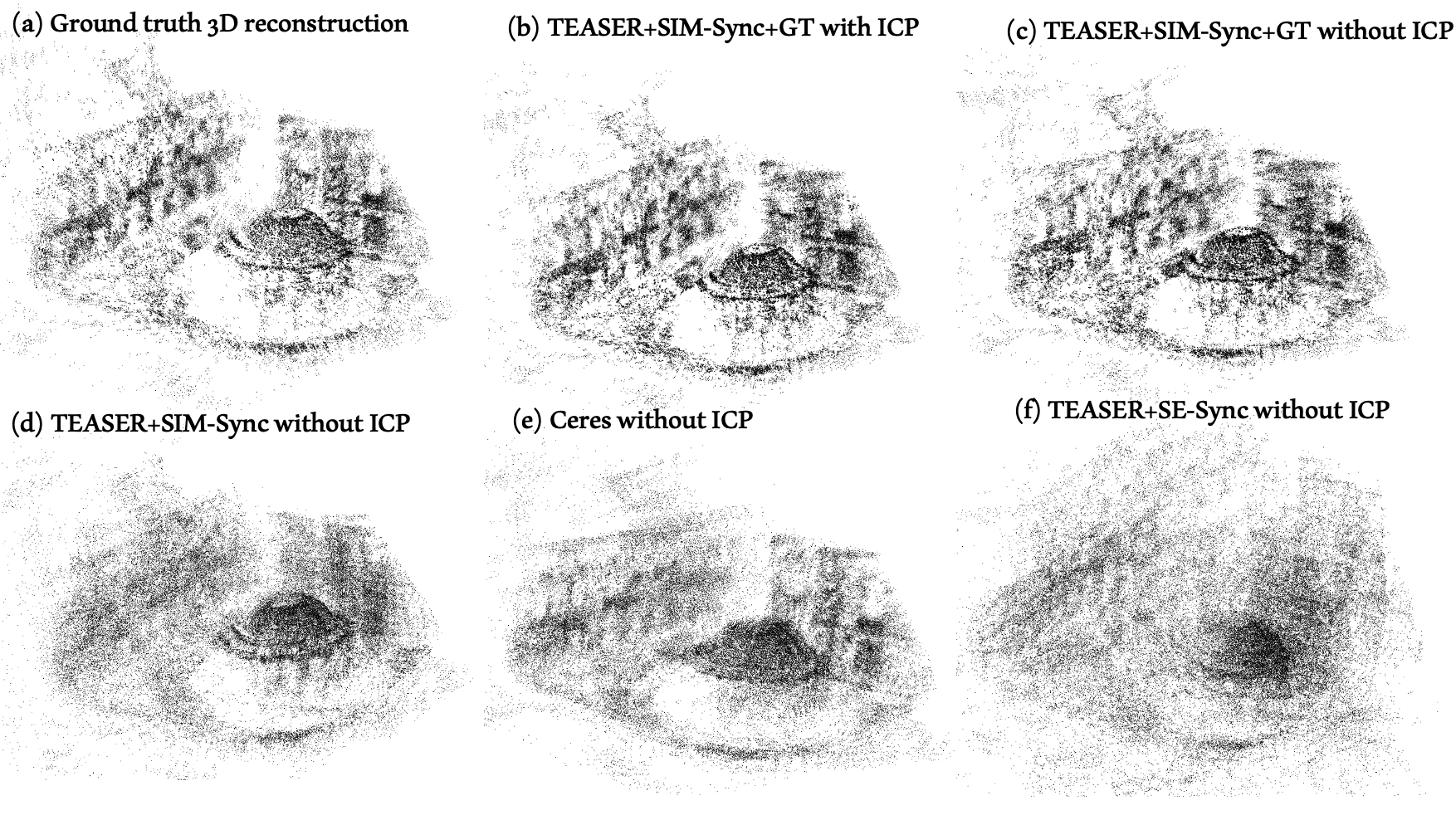}
    \vspace{-6mm}
    \caption{3D reconstruction of a front view street scene from \bal \texttt{dubrovnik_16} sequence.}
    \label{BAL_3dRecon}
    \vspace{-4mm}
\end{figure}

\begin{table}[t]
\caption{Rotation error, translation error and suboptimality of \teasersimsyncgt, \teasersimsync, \teasersesync and \ceres in \bal dataset \texttt{dubrovnik_16} sequence.}
\centering
\label{table: BAL_table_small}
\footnotesize
\begin{tabular}{ccccccc}
\hline
 & \teasersimsyncgt & \multicolumn{2}{c}{\teasersimsync}&\multicolumn{2}{c}{\teasersesync} & \ceres \\
& & w/o \ceres & w \ceres & w/o \ceres & w \ceres &  \\ 
\hline
Rotation Error [deg]& $0.615$ & $8.041$ & $2.031$  & $16.287$  & $2.009$ & $2.837$  \\ 
\hline
Translation Error [m]& $0.271$ & $3.564$ & $1.634$  & $4.393$ & $1.828$ & $2.586$ \\ 
\hline
Suboptimality & $1.578\mathrm{e-}9$ & $2.996\mathrm{e-}9$ & / & $9.651\mathrm{e-}10$ & / & /  \\ 
\hline
\end{tabular}
\vspace{-4mm}
\end{table}
\begin{figure}[t]
    \centering
    \includegraphics[width= 1.0\columnwidth]{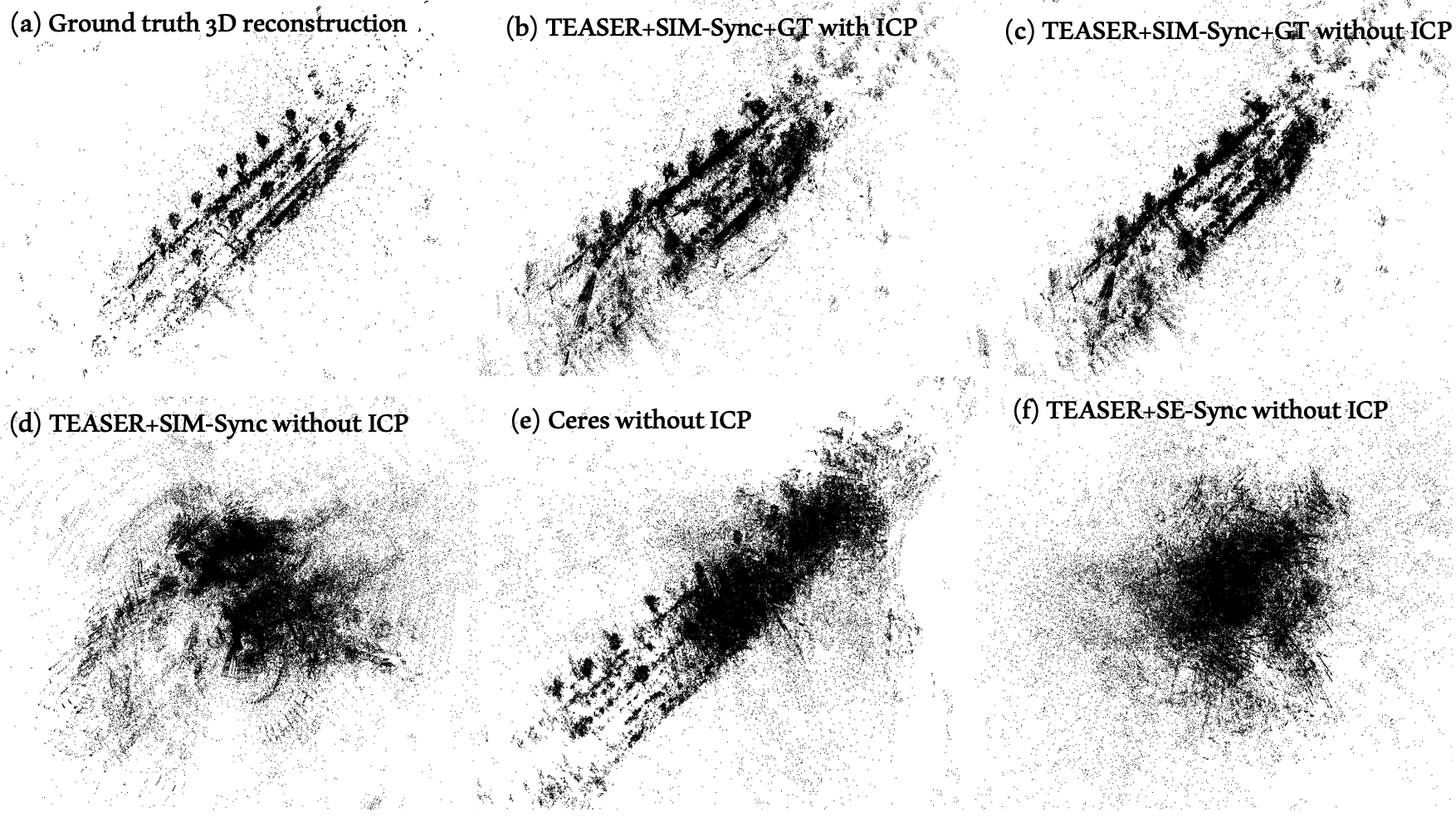}
    \caption{3D reconstruction of a bird view street scene from \bal \texttt{ladyburg_318} dataset.}
    \label{BAL_3dRecon2}
    \vspace{-4mm}
\end{figure}

\begin{table}[t]
\small
\caption{Rotation error, translation error and suboptimality of \teasersimsyncgt, \teasersimsync, \teasersesync and \ceres in \bal dataset \texttt{ladyburg_318} sequence.}
\centering
\label{table: BAL_table}
\footnotesize
\begin{tabular}{ccccccc}
\hline
 & \teasersimsyncgt & \multicolumn{2}{c}{\teasersimsync}&\multicolumn{2}{c}{\teasersesync} & \ceres \\
& & w/o \ceres & w \ceres & w/o \ceres & w \ceres &  \\ 
\hline
Rotation Error [deg]& $9.730$ & $78.457$ & $75.046$  & $86.067$  & $87.995$ & $82.601$  \\ 
\hline
Translation Error [m]& $0.334$ & $2.299$ & $3.307$  & $2.600$ & $5.858$ & $4.776$ \\ 
\hline
Suboptimality & $0.021$ & $0.092$ & / & $0.216$ & / & /  \\ 
\hline
\end{tabular}
\vspace{-3mm}
\end{table}

\subsection{TUM Experiments}
\label{tum-experiment}

\textbf{Setup.} We test two sequences in the \tum dataset, the first 200 frames in the \texttt{freiburg1_xyz} sequence and the first 200 frames in the \texttt{freiburg2_xyz} sequence, respectively.\footnote{We discard the first 60 frames in \texttt{freiburg2_xyz} since the camera shakes and results in blurred images.} For \teasersimsync, we use learned depth obtained from the \midas model \cite{Ranftl2022, birkl2023midas}, with the largest 10\% depth discarded. Note that \midas is not trained on the \tum dataset, and we directly use its default parameter configuration (\ie zero-shot). For \teasersimsyncgtdepth, we use ground truth depth. We conduct two tests for \teasersimsync/\teasersimsyncgtdepth. In the first test, we run \teasersimsync/\teasersimsyncgtdepth on the \emph{essential graph} (EG), which is a set of key frames and edges selected by the \orbslam~\cite{campos2021orb} algorithm. In the second test, we run \teasersimsync/\teasersimsyncgtdepth on a graph $\calG = (\calV,\calE)$ generated as follows
\bea \label{eq:edges_tum}
\calE = \left\{ (i, j) \mid i = 1, 2, \ldots, N-2, \text{ and } j = i+1 \text{ or } j = i+2 \text{ or } (i = N-1 \text{ and } j = N) \right\},
\eea 
\ie we sample the frame pairs of neighboring 3 frames.
We utilize \sift~\cite{lowe1999object} to get initial correspondences and then apply the learned \caps descriptor~\cite{wang2020learning} to sort the correspondences by the feature similarity between two points. We keep a maximum of 400 \sift correspondences. 

\textbf{Baseline.} We use the state-of-the-art visual SLAM algorithm \orbslam \cite{campos2021orb} as a baseline. We use Monocular mode without IMU of \orbslam and run on its default setting in the official script of running \tum dataset.\footnote{\url{https://github.com/UZ-SLAMLab/ORB_SLAM3}} We also compare with the stereo mode of \orbslam.

\textbf{Results.} 
We take a detour first to demonstrate how \teaser works as shown in Fig. \ref{TUM_TEASER_vis}. We randomly pick a pair of images from \teasersimsyncgtdepth and \teasersimsync. The red lines indicate outlier correspondences detected by \teaser while the green lines are inliers. We can see that \teaser performs quite well in classifying outliers from inliers. 

Tables \ref{table: TUM_dataset_quant_comp2} and \ref{table: TUM_dataset_quant_comp3} show the quantitative results of all methods in the \texttt{freiburg1_xyz} and \texttt{freiburg2_xyz} sequences, respectively. We follow the standard evaluation protocol of visual odometry for assessing pose accuracy, \ie Absolute Trajectory Error (ATE) and Relative Pose Error (RPE).\footnote{ATE quantifies the root-mean-square error between predicted camera positions and the groundtruth positions. RPE measures the relative pose disparity between pairs of adjacent frames, including both translation error (RPE-T) and rotational error (RPE-R).} 
Since the scale of the output of \orbslam is unknown, we scale up the predicted translation to the scale of the groundtruth.\footnote{The factor is the median of norms of ground truth translation divided by the median of norms of predicted translation.} 
We can see that \teasersimsyncgtdepth and \teasersimsync with Essential Graph achieve comparable accuracy as \orbslam, while being simpler and more direct algorithms that also offer optimality guarantees. On the other hand, \teasersimsyncgtdepth and \teasersimsync with the naive graph as in \eqref{eq:edges_tum} show worse accuracy. This result implies that the essential graph is a better graph architecture than the naive graph in scene reconstruction.

We show the qualitative 3D reconstruction results of the \texttt{freiburg1_xyz} and \texttt{freiburg2_xyz} sequences in Fig.~\ref{TUM_3dRecon}, using \teasersimsync with learned depth. The reconstruction is formed by stacking the learned depth point clouds of all frames (after transformation to a common coordinate frame).

We can see that even with learned depth, the \teasersimsync reconstruction achieves good accuracy. 

\begin{figure}[t]
    \centering
    \includegraphics[width= 1.0\columnwidth]{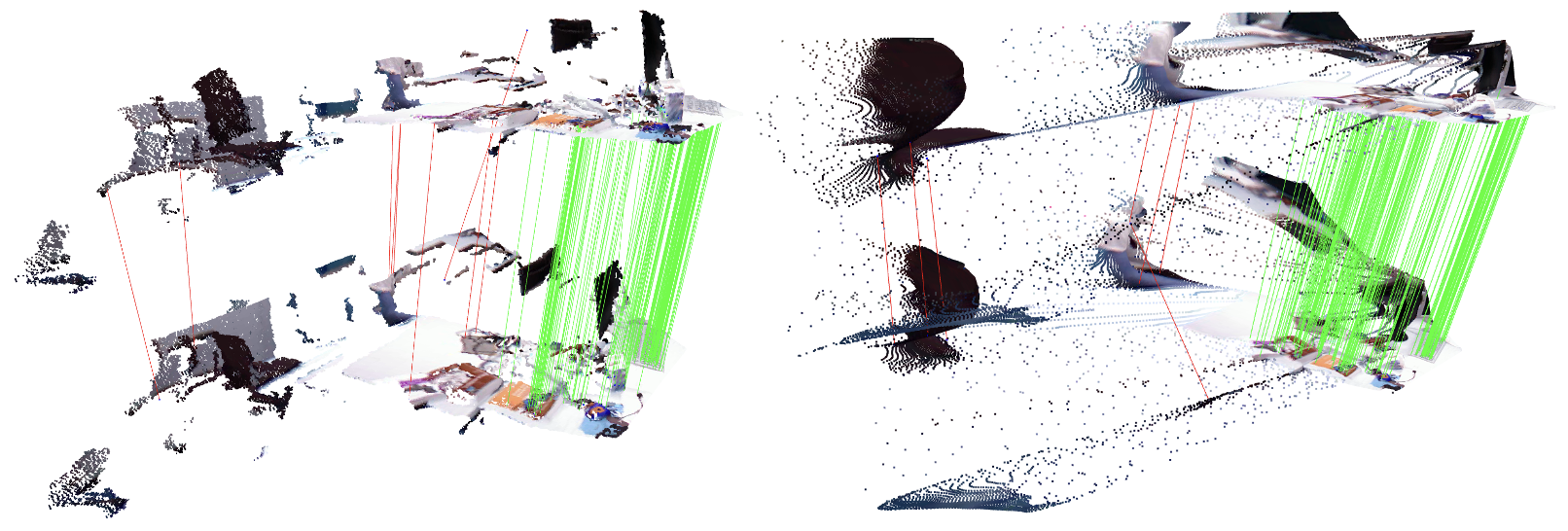}
    \caption{Illustration of outlier detection by \teaser. (Left: \textit{\teasersimsyncgtdepth}, Right: \textit{\teasersimsync with learned depth})}
    \label{TUM_TEASER_vis}
\end{figure}

\begin{figure}[t]
    \begin{center}
    \includegraphics[width= 1.0\columnwidth]{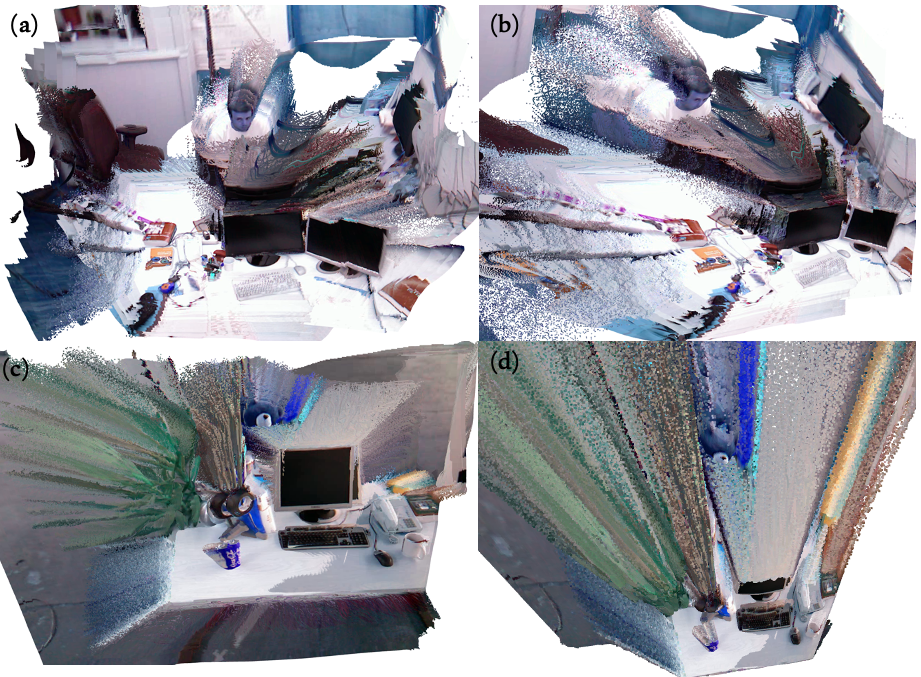}
    \end{center}
    \vspace{-4mm}
    \caption{3D reconstruction of \teasersimsync with learned depth in \tum. (a) Forward view reconstruction and (b) Side view reconstruction in the \texttt{freiburg1_xyz} sequence. (c) Forward view reconstruction and (d) Bird view reconstruction in the \texttt{freiburg2_xyz} sequence.}
    \label{TUM_3dRecon}
\end{figure}

\begin{table}[t]
\small
\caption{Rotation and translation error comparison for \teasersimsyncgtdepth, \teasersimsync and \orbslam in the 200-pose \texttt{freiburg1_xyz} sequence of \tum. We set the regularization factor to $\lambda=3$ for \teasersimsyncgtdepth and \teasersimsync without essential graph, and $\lambda=0$ with essential graph (EG).}
\vspace{1mm}
\centering
\label{table: TUM_dataset_quant_comp2}
\footnotesize
\begin{tabular}{ccccccc}
\hline
 & \multicolumn{2}{c}{\teasersimsyncgtdepth} & \multicolumn{2}{c}{\teasersimsync} &  \orbslam & \orbslam (RGB-D)\\ 
 & w/o EG & w EG & w/o EG & w EG &  &  \\ 
\hline
ATE [m] & $0.2453$ & $0.1335$  & $0.2003$  & $0.0638$ & $0.1081$ & $0.0147$ \\
\hline
RPE Trans [m] & $0.0167$ &  $0.0097$  &  $0.0163$  &  $0.0064$  & $0.0068$ & $0.0043$ \\
\hline
RPE Rot [deg] & $1.2001$ &  $0.6476$ & $1.3320$  &  $0.5659$  & $0.5521$ &  $0.2813$\\
\hline
Suboptimality & $1.1966\mathrm{e-}9$ &  $1.1966\mathrm{e-}9$  & $3.4336\mathrm{e-}9$  &  $9.8312\mathrm{e-}11$  & / &  /\\
\hline
\end{tabular}
\end{table}

\begin{table}[t]
\small
\caption{Rotation and translation error comparison for \teasersimsyncgtdepth, \teasersimsync and \orbslam in the 200-pose \texttt{freiburg2_xyz} sequence of \tum. Regularization factor same as Table~\ref{table: TUM_dataset_quant_comp2}.}
\vspace{1mm}
\centering
\label{table: TUM_dataset_quant_comp3}
\footnotesize
\begin{tabular}{ccccccc}
\hline
 & \multicolumn{2}{c}{\teasersimsyncgtdepth} & \multicolumn{2}{c}{\teasersimsync} &  \orbslam & \orbslam (RGB-D)\\ 
 & w/o EG & w EG & w/o EG & w EG &  &  \\ 
\hline
ATE [m] & $0.2089$ & $0.1237$ & $0.3190$ & $0.0358$ & $0.0246$ & $0.1023$ \\
\hline
RPE Trans [m] & $0.0072$ & $0.0017$  & $0.0030$ & $0.0012$ & $0.0007$ & $0.0017$ \\
\hline
RPE Rot [deg] & $0.3586$ & $0.2349$  & $0.9982$ & $0.2240$ & $0.2155$ & $0.1739$ \\
\hline
Suboptimality & $9.8904\mathrm{e-}10$ & $1.7527\mathrm{e-}9$  & $2.7516\mathrm{e-}9$ & $6.1298\mathrm{e-}10$ & /  &  /\\
\hline
\end{tabular}
\end{table}
\section{Conclusions}
\label{sec:conclusions}

We introduced \simsync, a certifiably optimal algorithm for camera trajectory estimation and scene reconstruction directly from image-level correspondences. With a pretrained depth prediction network, 2D image keypoints are lifted to 3D scaled point clouds, and \simsync seeks to jointly synchronize camera poses and unknown (per-image) scaling factors to minimize the sum of Euclidean distances between matching points. By first developing a QCQP formulation and then applying semidefinite relaxation, \simsync can achieve certifiable global optimality. We demonstrate the tightness, (outlier-)robustness, and practical usefulness of \simsync in both simulated and real datasets. Future research aims to (i) speed up \simsync by exploiting low-rankness of the optimal SDP solutions~\cite{rosen2019se,yang23mp-inexact}, and (ii) leverage the 3D reconstruction from \simsync to improve the imperfect depth prediction from the pretrained model.

\clearpage
\begin{center}
    \Large{\bf Appendix}
\end{center}
\appendix

\renewcommand{\theequation}{A\arabic{equation}}
\renewcommand{\thetheorem}{A\arabic{theorem}}


\section{Proof of Proposition \ref{prop:simple-objective}}
\label{proof_prop1}
\begin{proof}
    Let
    \bea 
    P_{i} =& \bmat{ccc} \sqrt{w_{ij,1}} d_{i,1} \tldp_{i,1} & \cdots & \sqrt{w_{ij,n_{ij}}} d_{i,n_{ij}}  \tldp_{i,n_{ij}} \emat \in \Real{3 \times n_{ij}}, \quad i =1,\dots,N, \nonumber\\
    W_{ij} =& [ \sqrt{w_{ij,1}} ,\dots, \sqrt{w_{ij,n_{ij}}} ]\tran \in \Real{n_{ij}}, \quad (i,j) \in \calE, \nonumber
    \eea
    For any $t_i$ and $\vectorize{s_i R_i}$, define the selection matrices
    \bea 
    t_i = Z^t_i t, \quad \vectorize{s_i R_i} = Z^R_i r, \quad Z^t_i \in \Real{3 \times 3N},\quad Z^R_i \in \Real{9 \times 9N}. \nonumber
    \eea
    We then write the objective in \eqref{eq:sba} as 
    \bea 
    L(t,r) = & \displaystyle \sum_{(i,j) \in \calE} \norm{ (s_iR_i) P_i + t_i W_{ij}\tran - (s_j R_j) P_j - t_j W_{ij}\tran  }_F^2 \nonumber \\
    = & \displaystyle \sum_{(i,j) \in \calE} \norm{ (P_i\tran \kron \eye_3) \vectorize{s_i R_i} + (W_{ij} \kron \eye_3) t_i - (P_j\tran \kron \eye_3) \vectorize{s_j R_j} - (W_{ij} \kron \eye_3) t_j } ^2 \nonumber \\
    = & \displaystyle \sum_{(i,j) \in \calE} \norm{ \underbrace{\bracket{(P_i\tran \kron \eye_3)Z^R_i - (P_j\tran \kron \eye_3)Z^R_j}}_{=:C_{ij}^R} r + \underbrace{\bracket{ (W_{ij} \kron \eye_3)Z^t_i - (W_{ij} \kron \eye_3)Z^t_j }}_{=:C_{ij}^t} t  }^2 \label{eq:app-prove-obj-1}
    \label{eq:expand_z}
    \eea
    Note that $Z^R_i$ and $Z^t_i$ can actually be decomposed as
    \bea
    Z^R_i = & e_i \tran \kron \eye_3 \kron \eye_3 \nonumber \\
    Z^t_i = & e_i \tran \kron \eye_3 \nonumber
    \eea
    where $e_i \in \Real{N}$ is the basis vector in $\Real{N}$ where the $i^{th}$ entry is 1 while other entries are all 0's.
    Plug in $C_{ij}^{R}$ and $C_{ij}^t$, we obtain:
    \bea
    C_{ij}^{R} = & (P_i\tran \kron \eye_3)(e_i \tran \kron \eye_3 \kron \eye_3) - (P_j\tran \kron \eye_3)(e_j \tran \kron \eye_3 \kron \eye_3) \nonumber \\
    = & (P_i \tran(e_i \tran \kron \eye_{3}))\kron \eye_{3} - (P_j \tran(e_j \tran \kron \eye_{3}))\kron \eye_{3} \nonumber\\
    = & ((1 \kron P_i \tran )(e_i \tran \kron \eye_{3}))\kron \eye_{3} - ((1 \kron P_j \tran)(e_j \tran \kron \eye_{3}))\kron \eye_{3} \nonumber\\
    = & (e_i \tran \kron P_i \tran )\kron \eye_{3} - (e_j \tran \kron P_j \tran )\kron \eye_{3} \nonumber\\
    = & (e_i \tran \kron P_i \tran  - e_j \tran \kron P_j \tran )\kron \eye_{3} \nonumber\\
    C_{ij}^t = & (W_{ij} \kron \eye_3)(e_i \tran \kron \eye_3) - (W_{ij} \kron \eye_3)(e_j \tran \kron \eye_3) \nonumber\\
    = & (W_{ij} e_i \tran) \kron \eye_3 - (W_{ij} e_j \tran) \kron \eye_3 \nonumber\\
    = & (W_{ij} e_i \tran - W_{ij} e_j \tran) \kron \eye_3 \nonumber
    \eea
    Plug $C_{ij}^{R}$ and $C_{ij}^t$ back to \eqref{eq:app-prove-obj-1}, we then simplify the objective as:
    \bea
    L(t,r) = & \displaystyle \sum_{(i,j) \in \calE} \norm{ ((e_i \tran \kron P_i \tran  - e_j \tran \kron P_j \tran) \kron \eye_3)  r + ((W_{ij} e_i \tran - W_{ij} e_j \tran) \kron \eye_3) t  }^2 \nonumber \\
    = & \displaystyle \bmat{c} t \\ r \emat \tran \parentheses{  \bmat{cc} Q_1 & V \tran \\ V & Q_2 \emat \kron \eye_3 }  \bmat{c} t \\ r \emat \nonumber \\
    = & \displaystyle t\tran \parentheses{ Q_1 \kron \eye_3} t + 2 r\tran\parentheses{ V \kron \eye_3} t + r\tran \parentheses{ Q_2 \kron \eye_3} r \label{eq:app-lossrt}
    \eea
    where
    \bea
    Q_1 = & \displaystyle \sum_{(i,j) \in \calE} (W_{ij} e_i \tran - W_{ij} e_j \tran)\tran (W_{ij} e_i \tran - W_{ij} e_j \tran) \in \Real{N \times N} \label{eq:app-Q1-restate} \\
    Q_2 = & \displaystyle \sum_{(i,j) \in \calE} (e_i \tran \kron P_i \tran  - e_j \tran \kron P_j \tran )\tran (e_i \tran \kron P_i \tran  - e_j \tran \kron P_j \tran ) \in \Real{3N \times 3N} \nonumber \\
    V = & \displaystyle \sum_{(i,j) \in \calE} (e_i \tran \kron P_i \tran  - e_j \tran \kron P_j \tran )\tran (W_{ij} e_i \tran - W_{ij} e_j \tran) \in \Real{3N \times N} \nonumber
    \eea
    concluding the proof.
\end{proof}

\section{Proof of Proposition \ref{prop:scale-rotation-only}}
\label{proof_prop2}
\begin{proof}
    To represent $t$ as a function of $r$, simply take the gradient of  (\ref{eq:simpleobjective}) \wrt $t$ and setting it to zero, we obtain
    \bea
    (Q_1 \kron \eye_3) {t} = (V \kron \eye_3)\tran r.
    \eea

    \begin{proposition}[Laplacian Representation of $Q_1$]\label{prop:laplacian} $Q_1$ can be represented as:
        \bea
        Q_1 = L(\calG) 
        \eea     
        where $L(\calG)$ is the Laplacian of $\calG$.
    \end{proposition}
    \begin{proof}
        Note that $\calG$ is a weighted undirected graph. Calling $\delta(i)$ for the set of edges incident to a vertex $v$ and $n_e = n_{ij}, w_{e,k} = w_{ij,k}$ for $e = (i,j)$, the Laplacian of $\calG$ is:
            \bea
            \label{lap}
            L(\calG)_{ij} = 
            \begin{cases}
                \sum_{e \in \delta(i)} \sum_{k=1}^{n_e} \sqrt{w_{e,k}} & \text{if } i = j, \\
                -\sum_{k=1}^{n_{ij}} \sqrt{w_{ij,k}} & \text{if } (i,j)\in \calE, \\
                0 & \text{if } (i,j) \notin \calE.
            \end{cases}
            \eea
        On the other hand, by expanding \eqref{eq:app-Q1-restate} and compare it with \eqref{lap}, we finish the proof.
    \end{proof}
    Note that $\rank{L(\calG)}=N-1$, then $\rank{Q_1 \kron \eye_3}=\rank{L(\calG)}\rank{\eye_3}$ by Proposition \ref{prop:laplacian}. Thus $Q_1 \kron \eye_3$ is not invertible. As we fixed $s_1 = 1$, $R_1 = \eye_3$, $t_1 = 0$, $t$ has a unique solution. Calling $\bar{t} = [t_2;\dots;t_N] \in \Real{3N-3}$ and $\bar{Q}_1 = [c_2;\dots;c_N] \in \Real{N \times (N-1)}$ where $c_i$ is the ith column of $Q_{1}$. We have
    \bea
    (\bar{Q}_1 \kron \eye_3) {\bar{t}} = (V \kron \eye_3)\tran r
    \eea
    Since $\rank{L(\calG)}=N-1$ and $\sum_{i=1,..,N} c_i = \mathbf{0}_{N}$, then $span(c_1,...,c_n) = span(c_2,...,c_n)$ and $\rank{\bar{Q}_{1}}=N-1$ which implies that $\bar{Q}_{1}$ has full column rank. Hence, by taking inverse and rearrange, we obtain
    \bea 
    \bar{t} = & \parentheses{\bar{A} \kron \eye_3} r
    \eea 
    where
    \bea
    \bar{A} = & -(\bar{Q}_1 \tran \bar{Q}_1)\inv \bar{Q}_1 \tran V \tran
    \eea
    Together with $t_1 = 0$,
    \bea
    t = & \parentheses{A \kron \eye_3} r
    \eea
    with
    \bea
    A = & \bmat{c} \zero_{1 \times 3N} \\ -(\bar{Q}_1 \tran \bar{Q}_1)\inv \bar{Q}_1 \tran V \tran \emat \in \Real{N \times 3N}
    \eea
    Now we have a closed form solution of $t$. Plug in the solution of $t$ into \eqref{eq:app-lossrt}, we obtain:
    \bea
    L(r) = & \displaystyle r \tran \parentheses{\parentheses{A \tran Q_1 A + VA + A\tran V\tran + Q_2} \kron \eye_3} r
    \eea
    Note that
    \bea
    r = & \vectorize{R}
    \eea
    Then $L(r)$ is equivalent to
    \bea
    \label{loss_R}
    L(R) = & \displaystyle \vectorize{R} \tran \parentheses{\parentheses{A \tran Q_1 A + VA + A\tran V\tran + Q_2} \kron \eye_3} \vectorize{R}
    \eea
    Rewrite (\ref{loss_R}) in a more compact matricized form gives:
    \bea
    \rho^\star = \min_{R} & \trace{QR \tran R} \\
    Q \coloneqq  & A \tran Q_1 A + VA + A\tran V\tran + Q_2 \in \sym{3N},
    \eea
    concluding the proof.
\end{proof}

\section{Proof of Proposition \ref{prop:qcqp}}
\label{proof_prop3}
\begin{proof}
    We first prove \eqref{eq:sOthreequadratic}. 

    ``$\Longrightarrow$'': If $\bar{R} \in \sOthree$, then $\bar{R} = s R$ for some $s \geq 0$ and $R \in \Othree$. Therefore,
    $$
    \bmat{ccc} 
    c_1\tran c_1 & c_1\tran c_2 & c_1\tran c_3 \\
    * & c_2\tran c_2 & c_2\tran c_3 \\
    * & * & c_3\tran c_3 \emat  = \bar{R}\tran \bar{R} = (sR)\tran (sR) = s^2 R\tran R = s^2 \eye_3,
    $$
    which shows the quadratic constraints in \eqref{eq:sOthreequadratic} must hold. 
    
    ``$\Longleftarrow$'':
    Suppose the quadratic constraints in \eqref{eq:sOthreequadratic} hold. (i) If $c_i\tran c_i = 0,i=1,2,3$, then $\bar{R} = \zero \in \sOthree$ trivially holds. (ii) If $c_i\tran c_i = \alpha > 0,i=1,2,3$, then $\bar{R}$ can be written as 
    $$
    \bar{R} = \sqrt{\alpha} \bmat{ccc} u_1 & u_2 & u_3 \emat,
    $$
    where $u_i,i=1,2,3$ are unit vectors. However, 
    $$
    c_i\tran c_j = 0 \Longleftrightarrow \alpha u_i\tran u_j = 0 \Longleftrightarrow u_i\tran u_j = 0, \quad \forall i \neq j,
    $$
    which means $(u_1,u_2,u_3)$ are orthogonal to each other and therefore $[u_1,u_2,u_3] \in \Othree$.

    We then show that if an optimal solution of \eqref{eq:simsync-qcqp} satisfies \eqref{eq:determinant}, then it must be an optimizer of \eqref{eq:scaledRonly}. First note that \eqref{eq:simsync-qcqp} is a relaxation of \eqref{eq:scaledRonly} and $\rho_{\qcqp}^\star \leq \rho^\star$. This is because any feasible solution to \eqref{eq:scaledRonly} must also be feasible to \eqref{eq:simsync-qcqp}, due to the fact that any scaled rotation must also lie in $\sOthree$ by its definition \eqref{eq:sOthreedef}. However, if $R^\star$ satisfies \eqref{eq:determinant}, then we claim that $R^\star$ is also feasible to \eqref{eq:scaledRonly} (hence also optimal to \eqref{eq:scaledRonly}), \ie each $\bar{R}_i^\star$ can be written as a scaled rotation. In fact, $\bar{R}_i^\star \in \sOthree$ implies $\bar{R}_i^\star = s_i^\star R_i^\star$ for some $s_i^\star \geq 0$ and $R_i^\star \in \Othree$. However,
    $$
    \det{\bar{R}_i^\star} = (s_i^\star)^3 \det R_i^\star > 0 
    $$
    implies $s_i^\star \neq 0$ and $R_i^\star \in \SOthree$, because any matrix in $\Othree$ is either a rotation (with $+1$ determinant) or a reflection (with $-1$ determinant).
\end{proof}

\section{Weighted Scaled Point Cloud Registration}
\label{weighted_umeyama}
Consider the optimization problem where $\{(Y_i, X_i), i=1,...,n\}$ are matching scaled point clouds and we seek the best similarity transformation between them
\begin{equation}\label{eq:scaledpcr}
\min_{\substack{s > 0, R \in \SOthree, t \in \Real{3} \\ i=1,\dots,n } }  \sum_{i=1}^{n}  w_{i} \norm{ sRX_i + t - Y_i }^2,
\end{equation} 
where $w_i,i=1,\dots,n$ are known weights. We will show that problem~\eqref{eq:scaledpcr} admits a closed-form solution.

Let the objective function be 
\bea
f = \sum_{i=1}^{n}  w_{i} \norm{ sRX_i + t - Y_i }^2.
\eea
Firstly, take the derivative with respect to $t$:
\bea
\frac{\partial f}{\partial t} = 2\sum_{i} w_i(sRX_i+t - Y_i)
\eea
Set it to zero, we obtain:
\bea
\label{t_sol}
t = \mu_y - sR\mu_x
\eea
where
\bea
\mu_y = \frac{\sum_{i} w_i y_i}{\sum_{i} w_i} \\
\mu_x = \frac{\sum_{i} w_i x_i}{\sum_{i} w_i}
\eea
Substitute $t$ with (\ref{t_sol}) in original optimization, we obtain:
\bea
f = & \sum_{i=1}^{n}  w_{i} \norm{ sR \bracket{X_i - \mu_x} - (Y_i - \mu_y)  }^2 \\
= & \sum_{i=1}^{n} \norm{  \sqrt{w_{i}} sR \bracket{X_i - \mu_x} - \sqrt{w_{i}}(Y_i - \mu_y)  }^2 \\
= & \sum_{i=1}^{n} \norm{  R \bracket{\sqrt{w_{i}} s (X_i - \mu_x)} - \sqrt{w_{i}}(Y_i - \mu_y)  }^2 
\eea
Let
\bea
A = \bmat{ccc} \sqrt{w_{1}} (Y_1 - \mu_y) & \cdots & \sqrt{w_{n}} (Y_n - \mu_y)  \emat \in \Real{3 \times n}, \\
B = \bmat{ccc} \sqrt{w_{1}} (X_1 - \mu_x) & \cdots & \sqrt{w_{n}} (X_n - \mu_x)  \emat \in \Real{3 \times n}
\eea
Then $f$ is simplified as:
\bea
f = \norm{R\cdot(sB) - A}^2
\eea
Now solve $R$ (taken $s$ as known), the optimization is Wabha's problem. When $\rank{AB\tran} \geq 2$, the optimal rotation matrix $R$ can be uniquely determined as a function of $s$
\bea
g = \min_{R} \norm{R\cdot(sB) - A}^2 = \norm{A}^2 + s^2\norm{B}^2 - 2 \trace{sDS}
\eea
where $AB\tran$ can be computed as $UDV\tran$ by singular value decomposition and
\bea
S = 
\begin{cases}
    \eye_3 & \text{if } \det(U) \det(V) = 1, \\
    \diag{1,\ldots,1,-1} & \text{if } \det(U) \det(V) = -1.
\end{cases}
\eea
And the optimizer $R$ is:
\bea\label{eq:app-optimalR}
R = USV\tran
\eea
Now take derivative of $g$ with respect to $s$, we get:
\bea
\frac{\partial g}{\partial s} = 2s\norm{B}^2-2\trace{DS}
\eea
To minimize $g$, we have
\bea\label{eq:app-optimals}
s = \frac{\trace{DS}}{\norm{B}^2}.
\eea
In summary, we first compute $s$ by \eqref{eq:app-optimals}, and then $R$ according to \eqref{eq:app-optimalR}, finally $t$ from \eqref{t_sol}.

\section{Noise Analysis}
There are several places that involve noise analysis in the paper. (i) \sesync needs uncertainty estimation of the solution returned by Arun's method. We provide analysis in Section \ref{scale-free noise}. (ii) In \simsyncgnc, \gnc needs a noise bound $\beta$ (\cf \eqref{eq:simsyncgnc}). We provide analysis in Section \ref{noise-scale-simsync}. (iii) In \gncsimsync and \teasersimsync, \gnc and \teaser need edge-wise noise bounds (\cf \eqref{eq:scaledpairregistration}). We provide analysis in Section \ref{noise with scale}.

\subsection{Covariance Estimation of Arun's Method}
\label{scale-free noise}

Consider $(R_i, t_i), (R_j, t_j) \in \SEthree$ as camera poses in frame $i$ and $j$ respectively. For point clouds $P$ in world frame, we generate point clouds $P_i$ and $P_j$ by corrupting noises $\epsilon_i \sim \mathcal{N}(0, \sigma^2 \eye_3)$ and $\epsilon_j \sim \mathcal{N}(0, \sigma^2 \eye_3)$ (assuming $\epsilon_i$ and $\epsilon_j$ are independent):
\bea
P_i = & R_i P + t_i + \epsilon_i\\
P_j = & R_j P + t_j + \epsilon_j
\eea
Remove variable $P$, we obtain:
\bea
P_i = R_i R_j \tran P_j + t_i - R_i R_j \tran t_j + \epsilon_i -R_i R_j \tran \epsilon_j
\eea
Reparametrize the variables:
\bea
R_{ij} = & R_i R_j \tran \in \SOthree\\
t_{ij} = & t_i - R_j R_j \tran t_j \in \Real{3}\\
\epsilon_{ij} = & \epsilon_i -R_i R_j \tran \epsilon_j \sim \mathcal{N}(0, 2\sigma^2 \eye_3)
\eea
We obtain:
\bea
\label{optimization_rel_pose}
P_i = R_{ij} P_j +t_{ij} + \epsilon_{ij}
\eea
Arun's method estimates $R_{ij}$ and $t_{ij}$. We want to estimate the uncertainty of the solution computed by Arun's method.

Our strategy is to firstly utilize Arun's method to find the optimal solution for noise free (\ref{optimization_rel_pose}), and then form a Maximum Likelihood Estimator by disturbing the rotation and translation around optimizer. Denote the optimizer of $R_{ij}$ and $t_{ij}$ in Arun's method as $R_{ij}^*$ and $t_{ij}^*$.
Then we rewrite (\ref{optimization_rel_pose}) as
\bea
\label{optimization_rel_pose_disturb}
P_i = R_{ij}^* \exp({\hat{\omega}_{ij}}) P_j +t_{ij}^* + \delta_{ij} + \epsilon_{ij}
\eea
where \(\hat{\omega}\) is a skew-symmetric matrix in the Lie algebra \(\mathfrak{so}(3)\) and $\delta_{ij} \in \Real{3}$. In (\ref{optimization_rel_pose_disturb}), we reparameterize the rotation matrix by compositioning on a rotation action and a logarithm map onto the Tangent Space $T_{R_{ij}^*} \SOthree$.
The exponential map is used to map \(\hat{\omega}\) to a 3D rotation matrix \(R\) in the Lie group $\SOthree$:
\bea
R = \exp(\hat{\omega})
\eea
The hat operator maps this vector to a skew-symmetric matrix \(\hat{\omega}\) in \(\mathfrak{so}(3)\), given by:
\bea
\hat{\omega} = 
\begin{bmatrix}
0 & -\omega_z & \omega_y \\
\omega_z & 0 & -\omega_x \\
-\omega_y & \omega_x & 0
\end{bmatrix}
\eea
Specifically, for the $k$-th measurement, (\ref{optimization_rel_pose_disturb}) is:
\bea
P_{ik} = R_{ij}^* \exp({\hat{\omega}_{ij}}) P_{jk} +t_{ij}^* + \delta_{ij} + \epsilon_{ijk}
\eea
where $\epsilon_{ijk} \sim \mathcal{N}(0, \Sigma_k)$ and further assume that $\epsilon_{ijk}$ are i.i.d for $k=1,...n_{ij}$. With $P_{ik}, P_{jk}, R_{ij}^*, t_{ij}^*$ known for all $k = 1,...,n_{ij}$, rewrite (\ref{optimization_rel_pose_disturb}) as optimization problem on $x_{ij} = ({\omega}_{ij}, \delta_{ij})$:
\bea
\label{least-square}
(\omega_{ij}^*, \delta_{ij}^*) = &\argmin \sum_{k=1}^{n_{ij}} \norm{ \underbrace{R_{ij}^* \exp({\hat{\omega}_{ij}}) P_{jk} +t_{ij}^* + \delta_{ij}}_{=: f(x_{ij})} - P_{ik}}_{\Sigma_{k}^{-1}}
\eea
Since $f(x_{ij})$ is nonlinear function of $x_{ij}$ and the distribution of $x_{ij}$ can be non-Gaussian and arbitrarily complex, we can only compute a Cramer-Rao lower bound on the posterior distribution by linearizing $f(x_{ij})$ around optimal $x_{ij}^*$. By the optimality of $R_{ij}^*, t_{ij}^*$, the optimizer is $x^* = (\omega_{ij}, \delta_{ij})=0$. 
\bea
(\omega_{ij}^*, \delta_{ij}^*) = &\argmin \sum_{k=1}^{n_{ij}} \norm{ \underbrace{R_{ij}^* \exp({\hat{\omega}_{ij}}) P_{jk} +t_{ij}^* + \delta_{ij}}_{=: f(x_{ij})} - P_{ik}}_{\Sigma_{k}^{-1}}
\eea
Note that 
\bea
f(x_{ij}) = & f(0) + \left. \frac{df}{dx_{ij}} \right|_{x_{ij}=0} (x_{ij} - 0) \\
= & R_{ij}^* P_{jk} +t_{ij}^* + \left. \frac{df}{dx_{ij}} \right|_{x_{ij}=0} x_{ij}
\eea
where
\bea
\left. \frac{df}{dx_{ij}} \right|_{x_{ij}=0} = \begin{bmatrix}-R_{ij}^*\hat{P}_{jk} & \eye_3 \end{bmatrix} 
\eea
Plug $f(x_{ij})$ into (\ref{least-square}). We obtain:
\bea
(\omega_{ij}^*, \delta_{ij}^*) \approx &\argmin \sum_{k=1}^{n_{ij}} \norm{  \underbrace{\left. \frac{df}{dx_{ij}} \right|_{x_{ij}=0}}_{=:H_k} x_{ij} - \underbrace{(P_{ik}-R_{ij}^* P_{jk} - t_{ij}^*)}_{=:\tilde{y}_k} }_{\Sigma_{k}^{-1}} \\
= &\argmin \sum_{k=1}^{n_{ij}} \norm{ H_k x_{ij} - \tilde{y}_k}_{\Sigma_{k}^{-1}}
\eea
This forms a linear psedo-measurement equation:
\bea
y_k = H_k x_{ij} + \epsilon_{k}
\eea
with $\epsilon_{k} \sim \mathcal{N}(0, \Sigma_k) =\mathcal{N}(0, 2\sigma^2 \eye_3)$.
With the assumption that $\epsilon_{ijk}$ are i.i.d for $k=1,...n_{ij}$, we obtain the posterior covariance of $x_{ij}$:
\bea
\Sigma_{x_{ij}} = & \text{Cov} (\omega_{ij}, t_{ij}) \nonumber \\
= & \left( \sum_{k=1}^{n_{ij}} H_k^T \Sigma_k^{-1} H_k \right)^{-1} \nonumber \\
= & \sigma^2\left( \sum_{k=1}^{n_{ij}} H_k \tran  H_k\right)^{-1} \label{eq:aruncrlb}
\eea
We feed the covariance matrix in \eqref{eq:aruncrlb} to \sesync.

\subsection{Noise Bound for \simsyncgnc}
\label{noise-scale-simsync}

Consider $(R_i, t_i), (R_j, t_j) \in \SEthree$ as camera poses in frame $i$ and $j$ respectively. For point cloud $P \in \Real{3}$ in world frame, we generate point clouds $P_i$ and $P_j$ by corrupting noises $\epsilon_i \sim \mathcal{N}(0, \sigma^2 \eye_3)$ and $\epsilon_j \sim \mathcal{N}(0, \sigma^2 \eye_3)$:
\bea
P_i = & \frac{1}{s_i}(R_i P + t_i + \epsilon_i)\\
P_j = & \frac{1}{s_j}(R_j P + t_j + \epsilon_j)
\eea
Remove variable $P$, we obtain:
\bea
s_iR_i \tran P_i  - R_i \tran t_i - (s_jR_j \tran P_j - R_j \tran t_j) =  R_i \tran \epsilon_i - R_j \tran \epsilon_j
\eea
$(R_i, t_i), (R_j, t_j) \in \SEthree$ are transformations from world frame to camera frame $i$ and $j$. If we represent the above equation using $(R_i', t_i'), (R_j', t_j') \in \SEthree$ that are transformations from camera frame $i$ and $j$ to world frame. We have:
\bea
s_iR_i' P_i  + t_i' - (s_j R_j' P_j + t_j') =  R_i' \epsilon_i - R_j' \epsilon_j \sim \mathcal{N}(0, {2\sigma^2} \eye_3)
\eea
Then we obtain:
\bea
\norm{s_iR_i' P_i  + t_i' - (s_j R_j' P_j + t_j')}^2_{\frac{1}{2\sigma^2} \eye_3} \sim {\chi}^{2}_{n=3}
\eea
For a probability value of 0.9999 with 3 degrees of freedom, a threshold value of 21.11 is computed from the Chi-square distribution table. Statistically, it suggests that we have 99.99\% confidence that any samples that fall outside of this threshold
are outliers. Together with covariance normalization factor $\sqrt{2\sigma^2}$, $\beta = \sqrt{21.11} \cdot \sqrt{2}\sigma$.

\subsection{Noise Bound for \teaser and \gnc}
\label{noise with scale}

Consider $(R_i, t_i), (R_j, t_j) \in \SEthree$ as camera poses in frame $i$ and $j$ respectively. For point clouds $P$ in world frame, we generate point clouds $P_i$ and $P_j$ by corrupting noises $\epsilon_i \sim \mathcal{N}(0, \sigma^2 \eye_3)$ and $\epsilon_j \sim \mathcal{N}(0, \sigma^2 \eye_3)$:
\bea
P_i = & \frac{1}{s_i}(R_i P + t_i + \epsilon_i)\\
P_j = & \frac{1}{s_j}(R_j P + t_j + \epsilon_j)
\eea
Remove variable $P$, we obtain:
\bea
P_i = \frac{s_j}{s_i}R_i R_j \tran P_j + \frac{1}{s_i}t_i - \frac{1}{s_i}R_i R_j \tran t_j + \frac{1}{s_i}(\epsilon_i -R_i R_j \tran \epsilon_j)
\eea
We obtain:
\bea
\label{optimization_rel_pose_umeyama}
P_i = s_{ij}R_{ij} P_j +t_{ij} + \epsilon_{ij}
\eea
by reparametrizing the variables:
\bea
s_{ij} = & \frac{s_j}{s_i} \\
R_{ij} = & R_i R_j \tran \in \SOthree\\
t_{ij} = & \frac{1}{s_i}t_i - \frac{1}{s_i}R_j R_j \tran t_j \in \Real{3}\\
\epsilon_{ij} = & \frac{1}{s_i}(\epsilon_i -R_i R_j \tran \epsilon_j) \sim \mathcal{N}(0, \frac{2\sigma^2}{s_i^2} \eye_3)
\eea
Then we obtain:
\bea
\norm{s_{ij}R_{ij} P_j +t_{ij} - P_i}^2_{\frac{s_i^2}{2\sigma^2} \eye_3} \sim {\chi}^{2}_{n=3}
\eea
For a probability value of 0.9999 with 3 degrees of freedom, a threshold value of 21.11 is computed from the Chi-square distribution table. Statistically, it suggests that we have 99.99\% confidence that any samples that fall outside of this threshold
are outliers. Together with covariance normalization factor $\sqrt{\frac{2\sigma^2}{s_i^2}}$, $\beta = \sqrt{21.11} \cdot \frac{\sqrt{2}\sigma}{s_i}$.

\bibliography{refs}
\bibliographystyle{plainnat}

\end{document}